\documentclass[12pt]{article}

\usepackage[utf8]{inputenc}
\usepackage[T1]{fontenc}
\usepackage[scale=0.75]{geometry} 

\usepackage{amsmath}
\usepackage{amssymb}
\usepackage{mathrsfs} 
\usepackage{bbm} 
\usepackage{color}
\usepackage[svgnames]{xcolor}
\usepackage{graphics}
\usepackage{graphicx}
\usepackage{tikz} 
\usetikzlibrary{shapes.misc,automata,positioning,arrows}
\usepgflibrary{shapes.arrows} 
\usepackage[ruled,linesnumbered,vlined,scleft]{algorithm2e}
\usepackage{url}
\usepackage{comment}
\usepackage{theorem}
\usepackage{thmtools}

\usepackage{stmaryrd}
\usepackage{dsfont}

\newcommand{\ie}{\mbox{i.e.}}
\newcommand{\eg}{\mbox{e.g.}}
\newcommand{\cf}{\mbox{cf.}}

\newcommand{\wrt}{\mbox{w.r.t.}}
\newcommand{\st}{\mbox{s.t.}}

\newcommand{\etal}{\mbox{et al}}
\newcommand{\aka}{\mbox{a.k.a.}}
\newcommand{\myskip}{\medskip} 
\newcommand{\df}[1]{\textbf{#1}} 

\newcommand{\qed}{\rule{2mm}{2mm}}
\newcommand{\tuple}[1]{\langle #1 \rangle}
\newcommand{\defined}{=_{\mathrm{def}}}
\newcommand{\assigned}{\mathrel{\mathop:}=}
\newcommand{\grammar}{\mathrel{\mathop:}\assigned}

\newtheorem{definition}{Definition}
\newtheorem{lemma}{Lemma}
\newtheorem{corollary}{Corollary}
{\theorembodyfont{\upshape}
\newtheorem{example}{Example}
}
\newtheorem{remark}{Remark}
\newenvironment{proof}{\noindent\textbf{Proof:}\null\hfill\null\newline}

\newcommand{\ExpTime}{\textsc{exptime}}

\newcommand{\Lang}{\ensuremath{\mathcal{L}}} 
\newcommand{\KB}{\ensuremath{\mathcal{KB}}} 
\newcommand{\Th}{\ensuremath{\mathcal{KB}_{\text{inf}}}} 

\newcommand{\sat}{\Vdash}
\newcommand{\nsat}{\not\sat}
\newcommand{\entails}{\models}
\newcommand{\nentails}{\not\entails}
\newcommand{\proves}{\vdash}

\newcommand{\Mod}{\ensuremath{\text{\it Mod}}}

\newcommand{\CN}{\ensuremath{\mathsf{C}}} 
\newcommand{\RN}{\ensuremath{\mathsf{R}}} 
\newcommand{\IN}{\ensuremath{\mathsf{I}}} 
\newcommand{\dland}{\sqcap}
\newcommand{\dlor}{\sqcup}
\newcommand{\subs}{\sqsubseteq}

\newcommand{\TB}{\ensuremath{\mathcal{T}}} 
\newcommand{\AB}{\ensuremath{\mathcal{A}}} 
\newcommand{\DB}{\ensuremath{\mathcal{D}}} 

\newcommand{\Dom}{\Delta} 
\newcommand{\I}{\ensuremath{\mathcal{I}}} 

\newcommand{\ALC}{\ensuremath{\mathcal{ALC}}}

\newcommand{\dsubs}{\:\raisebox{0.45ex}{\ensuremath{\sqsubset}}\hskip-1.7ex\raisebox{-0.6ex}{\scalebox{0.9}{\ensuremath{\sim}}}\:}
\newcommand{\ndsubs}{\not\!\!\!\dsubs}

\newcommand{\dforall}{\mathrel{\scalebox{0.73}{\raisebox{0.46ex}{$\bigvee$}}}%
						\hskip-1.79ex\joinrel\scalebox{0.88}{\raisebox{0.4ex}{$\sim$}}}

\newcommand{\Typ}{\mathbf{T}}

\newcommand{\PI}{\ensuremath{\mathcal{P}}} 
\newcommand{\RI}{\ensuremath{\mathcal{R}}} 
\newcommand{\OI}{\ensuremath{\mathcal{O}}} 
\newcommand{\pref}{\prec} 
\newcommand{\npref}{\not\pref}
\newcommand{\incomp}{\ensuremath{\sim}} 

\newcommand{\rk}{\ensuremath{\mathsf{rk}}} 
\newcommand{\E}{\ensuremath{\mathcal{E}}}
\newcommand{\rank}{\ensuremath{\mathsf{rank}}} 
\newcommand{\C}{\ensuremath{\mathcal{C}}}

\newcommand{\Emp}{\ensuremath{\mathsf{Employee}}}
\newcommand{\Comp}{\ensuremath{\mathsf{Company}}}
\newcommand{\Stud}{\ensuremath{\mathsf{Student}}}
\newcommand{\EmpStud}{\ensuremath{\mathsf{EmpStud}}}
\newcommand{\Parent}{\ensuremath{\mathsf{Parent}}}
\newcommand{\Tax}{\ensuremath{\mathsf{Tax}}}
\newcommand{\Bird}{\ensuremath{\mathsf{Bird}}}
\newcommand{\Flies}{\ensuremath{\mathsf{Flies}}}
\newcommand{\Robin}{\ensuremath{\mathsf{Robin}}}
\newcommand{\Penguin}{\ensuremath{\mathsf{Penguin}}}
\newcommand{\Wings}{\ensuremath{\mathsf{Wings}}}
\newcommand{\Worker}{\ensuremath{\mathsf{Worker}}}
\newcommand{\Boss}{\ensuremath{\mathsf{Boss}}}
\newcommand{\Responsible}{\ensuremath{\mathsf{Responsible}}}

\newcommand{\pays}{\ensuremath{\mathsf{pays}}}
\newcommand{\emp}{\ensuremath{\mathsf{empBy}}}
\newcommand{\wk}{\ensuremath{\mathsf{worksFor}}}
\newcommand{\hasSuperior}{\ensuremath{\mathsf{hasSuperior}}}

\newcommand{\john}{\ensuremath{\mathsf{john}}}
\newcommand{\ibm}{\ensuremath{\mathsf{ibm}}}
\newcommand{\mary}{\ensuremath{\mathsf{mary}}}

\newenvironment{exContinued}
{\addtocounter{example}{-1}\begin{example}{\textit{{(continued)}}}}
  {\end{example}}

\begin{document}

\title{Theoretical Foundations of Defeasible Description Logics}

\author{Katarina Britz\footnote{CSIR-SU CAIR, Stellenbosch University, South Africa. email: abritz@sun.ac.za} \and Giovanni Casini\footnote{CSC, Universit\'{e} du Luxembourg, Luxembourg. email: giovanni.casini@gmail.com} \and Thomas Meyer\footnote{CAIR, University of Cape Town, South Africa. email: tmeyer@cs.uct.ac.za} \and Kody Moodley\footnote{Institute of Data Science, Maastricht University, Netherlands. \newline \indent \indent email: kody.moodley@maastrichtuniversity.nl} \and Uli~Sattler\footnote{Information Management Group, University of Manchester, United Kingdom. \newline \indent \indent email: sattler@cs.man.ac.uk} \and Ivan Varzinczak\footnote{CRIL, Univ.\ Artois \& CNRS, France. email: varzinczak@cril.fr}}

\date{}

\maketitle

\begin{abstract}
We extend description logics (DLs) with non-monotonic reasoning features. We start by investigating a notion of \emph{defeasible subsumption} in the spirit of defeasible conditionals as studied by Kraus, Lehmann and Magidor in the propositional case. In particular, we consider a natural and intuitive semantics for defeasible subsumption, and investigate KLM-style syntactic properties for both \emph{preferential} and \emph{rational} subsumption. Our contribution includes two representation results linking our semantic constructions to the set of preferential and rational properties considered. Besides showing that our semantics is appropriate, these results pave the way for more effective decision procedures for defeasible reasoning in~DLs. Indeed, we also analyse the problem of non-monotonic reasoning in~DLs at the level of \emph{entailment} and present an algorithm for the computation of \emph{rational closure} of a defeasible ontology. Importantly, our algorithm relies completely on classical entailment and shows that the computational complexity of reasoning over defeasible ontologies is no worse than that of reasoning in the underlying classical~DL~\ALC.

\end{abstract}

\section{Introduction}\label{Introduction}

Description logics~(DLs)~\cite{BaaderEtAl2007} are central to many modern~AI and database applications since they provide the logical foundation of formal ontologies. Yet, as classical formalisms, DLs do not allow for the proper representation of and reasoning with defeasible information, as shown up in the following example, adapted from Giordano~\etal.'s~\cite{GiordanoEtAl2007}: Students do not get tax invoices; employed students do; employed students who are also parents do not. From a naïve (classical) formalisation of this scenario, one concludes that the notion of employed student is an oxymoron, and, consequently, the concept of employed student is unsatisfiable. But while concept unsatisfiability has been investigated extensively in ontology debugging and repair~\cite{MoodleyEtAl2011,SchlobachCornet2003}, our research problem here goes beyond that, as will become clear in the upcoming sections.

Endowing DLs with defeasible reasoning features is therefore a promising endeavour from the point of view of applications of knowledge representation and reasoning. Indeed, the past~25~years have witnessed many attempts to introduce defeasible-reasoning capabilities in a~DL setting, usually drawing on a well-established body of research on non-monotonic reasoning (NMR). These comprise the so-called preferential approaches~\cite{BritzEtAl2008,BritzEtAl2009,BritzEtAl2011c,CasiniStraccia2010,CasiniStraccia2013,GiordanoEtAl2007,GiordanoEtAl2008,GiordanoEtAl2013,GiordanoEtAl2015,QuantzRoyer1992,QuantzRyan1993}, circumscription-based ones~\cite{BonattiEtAl2011,BonattiEtAl2009,SenguptaEtAl2011}, amongst others~\cite{BaaderHollunder1993,BaaderHollunder1995,BonattiEtAl2015,DoniniEtAl2002,Governatori2004,GrosofEtAl2003,HeymansVermeir2002,PadghamZhang1993,QiEtAl2007,Straccia1993}.

Preferential extensions of~DLs turn out to be particularly promising, mainly because they are based on an elegant, comprehensive and well-studied framework for non-monotonic reasoning in the propositional case proposed by Kraus, Lehmann and Magidor~\cite{KrausEtAl1990,LehmannMagidor1992} and often referred to as the \emph{KLM approach}. Such a framework is valuable for a number of reasons. First, it provides for a thorough analysis of some formal properties that any consequence relation deemed as appropriate in a non-monotonic setting ought to satisfy. Such formal properties, which resemble those of a Gentzen-style proof system (see Section~\ref{DefeasibleSubsumption}), play a central role in assessing how intuitive the obtained results are and enable a more comprehensive characterisation of the introduced non-monotonic conditional from a logical point of view. Second, the KLM approach allows for many decision problems to be reduced to classical entailment checking, sometimes without blowing up the computational complexity compared to the underlying classical case. Finally, it has a well-known connection with the AGM-approach to belief revision~\cite{GardenforsMakinson1994,Rott2001} and with frameworks for reasoning under uncertainty~\cite{BenferhatEtAl1999,DuboisEtAl1994}. It is therefore reasonable to expect that most, if not all, of the aforementioned features of the KLM approach should transfer to KLM-based extensions of~DLs, too.

Following the motivation laid out above, several extensions to the KLM approach to description logics have been proposed recently~\cite{BritzEtAl2008,BritzEtAl2011c,BritzVarzinczak2016b,BritzVarzinczak2017-DL,BritzVarzinczak2018-FoIKS,BritzVarzinczak2019-AMAI,CasiniStraccia2010,CasiniStraccia2013,GiordanoEtAl2007,GiordanoEtAl2008,GiordanoEtAl2013,GiordanoEtAl2015,PenselTurhan2017,Varzinczak2018}, each of them investigating particular constructions and variants of the preferential approach. However, here our aim is to provide a \emph{comprehensive} study of the formal foundations of preferential defeasible reasoning in~DLs. By that we mean (\emph{i})~defining a general and intuitive \emph{semantics}; (\emph{ii})~showing that the corresponding \emph{representation results} (in the~KLM sense of the term) hold,  linking our semantic constructions with the KLM-style set of properties, and (\emph{iii})~presenting an appropriate analysis of \emph{entailment} in the context of ontologies with defeasible information with an associated decision procedure that is implementable.
\myskip

In the remainder of the paper, we shall take the following route: After providing the required background on the~DL we consider in this work as well as fixing the notation (Section~\ref{Preliminaries}), we introduce the notion of defeasible subsumption along with a set of KLM-inspired properties it ought to satisfy (Section~\ref{DefeasibleDLs}). In particular, using an intuitive semantics for the idea that ``usually, an element of the class~$C$ is also an element of the class~$D$'', we provide a characterisation (via representation results) of two important classes of defeasible statements, namely preferential and rational subsumption. In Section~\ref{Entailment}, we start by investigating two obvious candidates for the notion of entailment in the context of defeasible~DLs, namely preferential and modular entailment. These turn out \emph{not} to have all properties seen as important in a non-monotonic DL~setting, mimicking a similar result in the propositional case~\cite{LehmannMagidor1992}. Therefore, we propose a notion of rational entailment and show that it is the definition of consequence we are looking for. We take this definition further by exploring the relationship that rational entailment has with both Lehmann and Magidor's~\cite{LehmannMagidor1992} definition of rational closure and the more recent algorithm by Casini and Straccia~\cite{CasiniStraccia2010} for its computation (Section~\ref{RationalClosure}). After a discussion of, and comparison with, related work (Section~\ref{RelatedWork}), we conclude with a summary of our contributions and some directions for further exploration. Proofs of our results can  be found in the appendix.

\section{Logical preliminaries}\label{Preliminaries}

Description Logics (DLs)~\cite{BaaderEtAl2007} are decidable fragments of first-order logic with interesting properties and a variety of applications. There is a whole family of description logics, an example of which is \ALC\ and on which we shall focus in the present paper.\footnote{For the reader not conversant with Description Logics but familiar with modal logics, there are results in the literature relating some families of description logics to systems of modal logic. For example, a well-known result is the one linking the DL \ALC\ with the normal modal logic K~\cite{Schild1991}.}
\myskip

The (concept) language of $\ALC$ is built upon a finite set of atomic \emph{concept names}~$\CN$, a finite set of \emph{role names}~$\RN$ (\aka\ \emph{attributes}) and a finite set of \emph{individual names}~$\IN$ such that $\CN$, $\RN$ and $\IN$ are pairwise disjoint. In our scenario example, we can have for instance $\CN=\{\Emp,\Comp,\Stud,\EmpStud,\Parent,\Tax\}$, $\RN=\{\pays,\emp,\wk\}$, and $\IN=\{\john,\ibm,\mary\}$, with the respective obvious intuitions. With $A,B,\ldots$ we denote atomic concepts, with $r, s,\ldots$ role names, and with $a,b,\ldots$ individual names. Complex concepts are denoted with $C,D,\ldots$ and are built using the constructors $\lnot$~(complement), $\dland$ (concept conjunction), $\dlor$ (concept disjunction), $\forall$ (value restriction) and $\exists$ (existential restriction) according to the following grammar rules:
\[
C \grammar \top \mid \bot \mid \CN \mid (\lnot C) \mid (C\dland C) \mid (C\dlor C) \mid (\exists r.C) \mid (\forall r.C) 
\]

With $\Lang$ we denote the \emph{language} of all $\ALC$ concepts, which is understood as the smallest set of symbol sequences generated according to the rules above. When writing down concepts of~$\Lang$, we follow the usual convention and omit parentheses whenever they are not essential for disambiguation. Examples of $\ALC$ concepts in our scenario are $\Stud\dland\Emp$ and $\lnot\exists\pays.\Tax$.
\myskip

The semantics of $\ALC$ is the standard set-theoretic Tarskian semantics. An \emph{interpretation} is a structure $\I\defined\tuple{\Dom^{\I},\cdot^{\I}}$, where $\Dom^{\I}$ is a non-empty set called the \emph{domain}, and $\cdot^{\I}$ is an \emph{interpretation function} mapping concept names~$A$ to subsets $A^{\I}$ of $\Dom^{\I}$, role names~$r$ to binary relations $r^{\I}$ over $\Dom^{\I}$, and individual names~$a$ to elements of the domain~$\Dom^{\I}$, \ie, $A^{\I}\subseteq\Dom^{\I}$, $r^{\I}\subseteq\Dom^{\I}\times\Dom^{\I}$, and $a^{\I}\in\Dom^{\I}$.

Figure~\ref{Figure:DLInterpretation} depicts an interpretation for our scenario example with domain $\Dom^{\I}=\{x_{i} \mid 0\leq i\leq 10\}$, and interpreting the elements of the vocabulary as follows: $\Emp^{\I}=\{x_{1},x_{2},x_{5},x_{9}\}$, $\Comp^{\I}=\{x_{6},x_{10}\}$, $\Stud^{\I}=\{x_{1},x_{5},x_{7},x_{8}\}$,  $\EmpStud^{\I}=\{x_{1},x_{5}\}$, $\Parent^{\I}=\{x_{1},x_{2},x_{3}\}$, $\Tax^{\I}=\{x_{4}\}$, $\pays^{\I}=\{(x_{1},x_{0}),(x_{5},x_{4})\}$, $\emp^{\I}=\{(x_{9},x_{10})\}$, $\wk^{\I}=\{(x_{5},x_{6}),(x_{9},x_{10})\}$, $\john^{\I}=x_{5}$, $\ibm^{\I}=x_{6}$, $\mary^{\I}=x_{2}$.

\begin{figure}[h]
\begin{center}
\begin{tikzpicture}[->,>=stealth', auto,node_style/.style={circle,fill=black,minimum size=0.1mm}]
\node at (-0.35,6.5) {$\I:$};

\node at (0.4,6.7) {$\Delta^{\I}$};
\node at (1,5) {$\Tax^{\I}$};
\node at (13.85,6.35) {$\Parent^{\I}$};
\node at (5.5,0.5) {$\Stud^{\I}$};
\node at (9,0.5) {$\Emp^{\I}$};
\node at (13,4.2) {$\Comp^{\I}$};
\node at (6.7,4.45) {\rotatebox{90}{$\EmpStud^{\I}$}};

\draw [rounded corners,thick] (0,0) rectangle (15,7) ;
\draw (0.5,1.8) [rounded corners=5pt] -- (0.5,5.2) -- (3,5.2) -- (3,1.8) -- (0.5,1.8) -- cycle ; 
\draw (3.5,0.25) [rounded corners=5pt] -- (3.5,6.25) -- (3.5,6.75) -- (7.5,6.75) -- (7.5,0.25) -- cycle ; 
\draw (3.65,5.3) [rounded corners=5pt] -- (3.65,6.625) -- (14.5,6.625) -- (14.5,5.3) -- cycle ; 
\draw (3.8,3) [rounded corners=5pt] -- (3.8,6.5) -- (11,6.5) -- (11,0.25) -- (8,0.25) -- (8,3) -- cycle ; 
\draw (11.5,0.25) [rounded corners=5pt] -- (11.5,4.5) -- (14.5,4.5) -- (14.5,0.25) -- cycle ; 
\draw (3.95,3.15) [rounded corners=5pt] -- (3.95,6.35) -- (7,6.35) -- (7,3.15) -- cycle ; 

\node (x0) at (1.75,5.9) {$x_{0}$} ;
\node (x1) at (5.5,5.9) {$x_{1}$} ;
\node (x2) at (9.5,5.9) {$x_{2}(\mary)$} ;
\node (x3) at (13,5.9) {$x_{3}$} ;
\node (x4) at (1.75,3.5) {$x_{4}$} ;
\node (x5) at (5.5,3.5) {$x_{5}(\john)$} ;
\node (x6) at (13,3.5) {$x_{6}(\ibm)$} ;
\node (x7) at (4,1) {$x_{7}$} ;
\node (x8) at (7,1) {$x_{8}$} ;
\node (x9) at (9.5,1) {$x_{9}$} ;
\node (x10) at (13,1) {$x_{10}$} ;
	
\path (x1) edge [thick,swap,near end] node {{\footnotesize$\pays$}} (x0);
\path (x5) edge [thick,swap,near start] node {\ \quad{\footnotesize$\pays$}} (x4);
\path (x5) edge [thick] node {{\footnotesize$\wk$}\quad\quad\quad\ } (x6);
\path (x9) edge [thick,near start] node {{\footnotesize$\wk$}\quad\ } (x10);
\path (x9) edge [thick,near start,swap] node {{\footnotesize$\emp$}\quad} (x10);
\end{tikzpicture}
\end{center}
\caption{A DL interpretation.}
\label{Figure:DLInterpretation}
\end{figure}
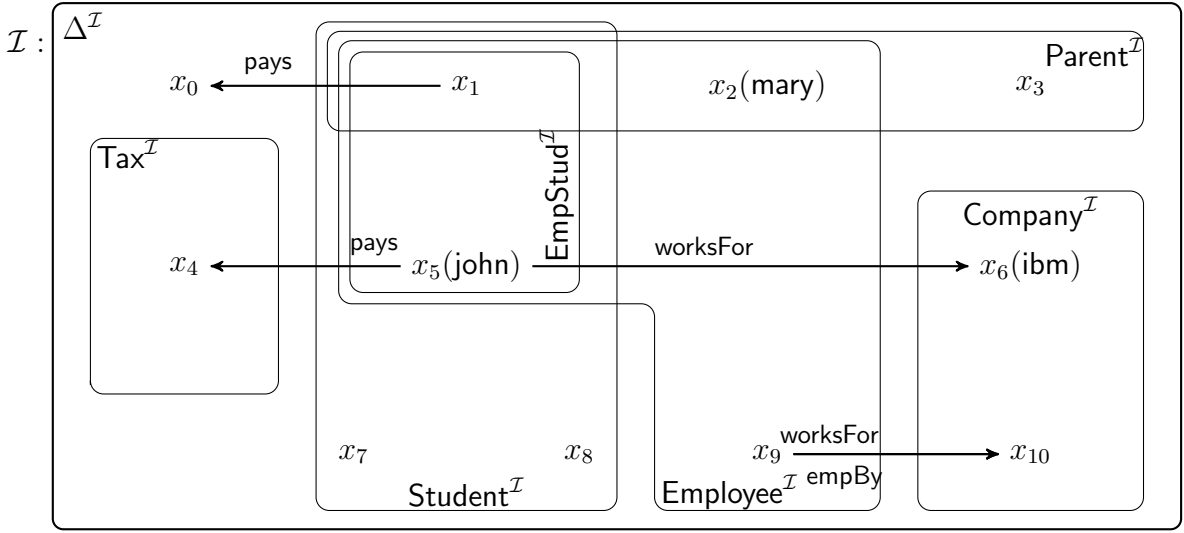

Let $\I=\tuple{\Dom^{\I},\cdot^{\I}}$ be an interpretation and define $r^{\I}(x) \defined \{y\in\Dom^{\I} \mid (x,y) \in r^{\I}\}$, for $r\in\RN$.  We extend the interpretation function~$\cdot^{\I}$ to interpret complex concepts of $\Lang$ as follows:
\[
\begin{array}{l}
\top^{\I}\defined\Dom^{\I};\\[0.2cm]
\bot^{\I}\defined\emptyset;\\[0.2cm]
(\lnot C)^{\I}\defined\Dom^{\I}\setminus C^{\I};\\[0.2cm]
(C\dland D)^{\I}\defined C^{\I}\cap D^{\I};\\[0.2cm]
(C\dlor D)^{\I}\defined C^{\I}\cup D^{\I};\\[0.2cm]
(\exists r.C)^{\I}\defined \{ x\in\Dom^{\I} \mid r^{\I}(x)\cap C^{\I}\neq\emptyset\};\\[0.2cm]
(\forall r.C)^{\I}\defined \{ x\in\Dom^{\I} \mid r^{\I}(x)\subseteq C^{\I}\}.
\end{array}
\]

For the interpretation~$\I$ in Figure~\ref{Figure:DLInterpretation}, we have $(\Parent\dland\Emp)^{\I}=\{x_{1},x_{2}\}$ and $(\exists\pays.\Tax)^{\I}=\{x_{5}\}$.
\myskip

Given $C, D\in\Lang$, a statement of the form $C\subs D$ is called a \emph{subsumption statement}, or \emph{general concept inclusion} (GCI), read ``$C$ is subsumed by $D$''. A concrete example of GCI is $\EmpStud\subs\Stud\dland\Emp$. $C\equiv D$ is an abbreviation for both $C\subs D$ and $D\subs C$. An $\ALC$ \emph{TBox}~\TB\ is a finite set of GCIs. Given $C\in\Lang$, $r\in\RN$ and $a,b\in\IN$, an \emph{assertional statement} (\emph{assertion}, for short) is an expression of the form $a:C$ or $(a,b):r$, read, respectively, ``$a$ is an instance of~$C$'' and ``$a$ is related to~$b$ via~$r$''. Examples of assertions are $\john:\EmpStud$ and $(\john,\ibm):\wk$. An $\ALC$ \emph{ABox}~\AB\ is a finite set of assertional statements. We denote statements with $\alpha, \beta, \ldots$. Given~$\TB$ and~$\AB$, with $\KB\defined\TB\cup\AB$ we denote an~$\ALC$ \emph{knowledge base}, \aka\ an \emph{ontology}, an example of which is given below:
\[
\TB=\left\{\begin{array}{c}
				\EmpStud\subs\Stud\dland\Emp,\\[0.15cm]
				\Stud\subs\lnot\exists\pays.\Tax,\\[0.15cm]
				\EmpStud\subs\exists\pays.\Tax,\\[0.15cm]
				\EmpStud\dland\Parent\subs\lnot\exists\pays.\Tax,\\[0.15cm]
				\Emp\subs\exists\wk.\Comp
		   \end{array}
	\right\}
\]\[
\AB=\{ \john:\EmpStud, \john:\Parent, (\john,\ibm):\wk \}
\]

An interpretation $\I$ \emph{satisfies} a GCI $C\subs D$ (denoted $\I\sat C\subs D$) if $C^{\I}\subseteq D^{\I}$. (And then $\I\sat C\equiv D$ if $C^{\I}=D^{\I}$.) $\I$ \emph{satisfies} an assertion $a:C$ (respectively, $(a,b):r$), denoted $\I\sat a:C$ (respectively, $\I\sat (a,b):r$), if $a^{\I}\in C^{\I}$ (respectively, $(a^{\I},b^{\I})\in r^{\I}$).

In the interpretation~$\I$ in Figure~\ref{Figure:DLInterpretation}, we have $\I\sat\EmpStud\subs\Stud\dland\Emp$, $\I\sat\john:\exists\pays.\Tax$ and $\I\nsat(\john,\ibm):\emp$.
\myskip

We say that an interpretation $\I$ is a \emph{model} of a TBox~\TB\ (respectively, of an ABox~\AB), denoted $\I\sat\TB$ (respectively, $\I\sat\AB$) if $\I\sat\alpha$ for every $\alpha$ in~\TB\ (respectively, in~\AB). We say that~$\I$ is a model of a knowledge base~$\KB=\TB\cup\AB$ if $\I\sat\TB$ and $\I\sat\AB$. It can be verified that the interpretation in Figure~\ref{Figure:DLInterpretation} is not a model of the example knowledge base above. (Actually, it is not hard to see that the knowledge base above admits no model.)

A statement~$\alpha$ is (classically) \emph{entailed} by a knowledge base~$\KB$, denoted $\KB\entails\alpha$, if every model of~$\KB$ satisfies~$\alpha$. If $\I\sat\alpha$ for all interpretations~$\I$, we say $\alpha$ is a \emph{validity} and denote this fact with $\entails\alpha$.
\myskip

The focus of the present paper being on defeasibility for description logic TBoxes only, we henceforth assume the ABox is empty. (We are currently in the process of extending our approach to description logic knowledge bases, with ABoxes included into the mix.) It is easy to see that, for~$\TB$ as above, we have $\TB\entails\EmpStud\subs\bot$.
\myskip

For more details on Description Logics in general and on~\ALC\ in particular, the reader is invited to consult the Description Logic Handbook~\cite{BaaderEtAl2007} and the introductory textbook on Description Logic~\cite{BaaderEtAl2017}.

\section{Foundations for defeasibility in DLs}\label{DefeasibleDLs}

In this section, we lay the formal foundations of our approach to defeasible reasoning in DL ontologies. For the most part, we build on the so-called preferential approach to non-monotonic reasoning~\cite{KrausEtAl1990,LehmannMagidor1992,Shoham1988}.

\subsection{Defeasible subsumption relations and their KLM-style properties}\label{DefeasibleSubsumption}

In a sense, class subsumption (alias concept inclusion) of the form $C\subs D$ is the main notion in DL ontologies. Given its implication-like intuition, subsumption lends itself naturally to defeasibility: ``provisionally, if an object falls under~$C$, then it also falls under~$D$'', as in ``usually, students are tax exempted''. In that respect, a defeasible version of concept inclusion is the starting point for an investigation of defeasible reasoning in DL ontologies. (We shall also address defeasibility of the entailment relation in later sections.)

\begin{definition}[Defeasible Concept Inclusion]\label{Def:DCI}
Let $C,D\in\Lang$. A \df{defeasible concept inclusion axiom} (DCI, for short) is a statement of the form $C\dsubs D$.
\end{definition}

A defeasible concept inclusion of the form $C\dsubs D$ is to be read as ``\emph{usually},  an instance of the class~$C$ is also an instance of the class~$D$''. For instance, the DCI $\Stud\dsubs\lnot\exists\pays.\Tax$ formalises the example above. Paraphrasing Lehmann~\cite{Lehmann1989}, the intuition of~$C\dsubs D$ is that ``if [the fact it belongs to]~$C$ were all the information about an object available to an agent, then [that it also belongs to]~$D$ would be a sensible conclusion to draw about such an object''. It is worth noting that~$\dsubs$, just as~$\subs$, is a `connective' sitting between the concept language (object level) and the meta-language (that of entailment) and it is meant to be the defeasible counterpart of the classical subsumption~$\subs$.

Being (defeasible) statements, DCIs will also be denoted by $\alpha,\beta,\ldots$ Whenever a distinction between GCIs and DCIs is in order, we shall make it explicitly.

\begin{definition}[Defeasible TBox]\label{Def:DefeasibleTbox}
A \df{defeasible TBox} (\df{DTBox}, for short) is a finite set of DCIs.
\end{definition}

Given a TBox~$\TB$ and a DTBox~$\DB$, we let $\KB\defined\TB\cup\DB$ and refer to it as a {\em defeasible knowledge base} (alias \emph{defeasible ontology}).

\begin{example}\label{Example:Student}
The following defeasible knowledge base gives a formal specification for our student scenario:
\[
\TB=\{\EmpStud\subs\Stud\}
\]\[
\DB=\left\{\begin{array}{c}
	\Stud\dsubs\lnot\exists\pays.\Tax,\\[0.1cm]
	\EmpStud\dsubs\exists\pays.\Tax,\\[0.1cm]
	\EmpStud\dland\Parent\dsubs\lnot\exists\pays.\Tax\\
	\end{array}
  \right\}
\] \hfill\ \qed
\end{example}

In our semantic construction later on, it will also be useful to be able to refer to infinite sets of concept inclusions. Let $\Th$ therefore denote a {\em defeasible theory}, defined as a defeasible knowledge base but without the restriction on $\TB$ and $\DB$ to finite sets.
\myskip

In order to assess the behaviour of the new connective and check it against both the intuition and the set of properties usually considered in a non-monotonic setting, it is convenient to look at a set of ${\dsubs}$-statements as a binary relation of the `antecedent-consequent' kind.

\begin{definition}[Defeasible Subsumption Relation]\label{Def:DefSubsumption}
A \df{defeasible subsumption relation} is a binary relation $\dsubs\subseteq\Lang\times\Lang$.
\end{definition}

The idea is to mimic the analysis of defeasible entailment relations carried out by Kraus~\etal.~\cite{KrausEtAl1990} in the propositional case, where entailment is seen as a binary relation on the set of propositional sentences. Here we shall adopt the view of subsumption as a binary relation on concepts of our description language.

Sometimes (\eg\ in the structural properties below) we shall write $(C,D)\in\dsubs$ in the infix notation, \ie, as $C\dsubs D$. The context will make clear when we will be talking about elements of a relation or statements (DCIs) in a defeasible knowledge base. Whenever disambiguation is in order, we shall flag it to the reader.

\begin{definition}[Preferential Subsumption Relation]\label{Def:PrefSubsumption}
A defeasible subsumption relation~$\dsubs$ is a \df{preferential subsumption relation} if it satisfies the following set of properties, which we refer to as the (DL versions of the) preferential~KLM properties:
\[
\begin{array}{lllll}
\\
(\text{Cons}) \ \top\ndsubs\bot &
(\text{Ref}) \ C \dsubs C &  
(\text{\small LLE}) \ {\displaystyle \frac{C\equiv D,\ C \dsubs E}{D \dsubs E}}\\[0.45cm]  (\text{And}) \ {\displaystyle \frac{C \dsubs D,\ C\dsubs E}{C \dsubs D \dland E}} & 
(\text{Or}) \ {\displaystyle \frac{C \dsubs E,\ D \dsubs E}{C \dlor D \dsubs E}} &
(\text{\small RW}) \ {\displaystyle \frac{C \dsubs D,\ D \subs E}{C \dsubs E}}\\[0.45cm] (\text{\small CM}) \ {\displaystyle \frac{C \dsubs D,\ C \dsubs E}{C\dland D \dsubs E}}
\end{array}
\]
\end{definition}
\myskip

The (Cons) property is a consequence of the adoption of a DL-based semantics, which enforces the non-emptiness of the domain, as will become clear in the next section.
The rest of the properties in Definition~\ref{Def:PrefSubsumption} result from a translation of the properties for preferential consequence relations proposed by Kraus~\etal.~\cite{KrausEtAl1990} in the propositional setting. They have been discussed at length in the literature for both the propositional and the~DL cases~\cite{BritzEtAl2008,BritzEtAl2011c,GiordanoEtAl2009a,GiordanoEtAl2009b,KrausEtAl1990,LehmannMagidor1992} and we shall not repeat so here.

If, in addition to the preferential properties above, the relation~$\dsubs$ also satisfies rational monotonicity~(RM) below, then it is said to be a {\em rational} subsumption relation:
\[
(\text{\small RM})\ \frac{C\dsubs D, \ C\ndsubs\lnot E}{C\dland E\dsubs D}
\]

Rational monotonicity is often considered a desirable property to have, one of the reasons stemming from the fact it is a necessary condition for the satisfaction of the principle of \emph{presumption of typicality}~\cite[Section 3.1]{Lehmann1995}. Such a principle is a simple yet intuitive formalisation of a form of reasoning we carry out when facing lack of information: we reason assuming that we are in the most typical possible situation, compatible with the information at our disposal. (More details will be provided in Section~\ref{Entailment}).

\subsection{Preferential semantics and representation results}\label{Semantics}

In this section, we present our semantics for preferential and rational subsumption by enriching standard~DL interpretations~$\I$ with an ordering on the elements of the domain~$\Dom^{\I}$. The intuition underlying this is simple and natural, and extends similar work done for the propositional case by Shoham~\cite{Shoham1988}, Kraus~\etal.~\cite{KrausEtAl1990}, Lehmann and Magidor~\cite{LehmannMagidor1992} and Booth~\etal.~\cite{BoothEtAl2015,BoothEtAl2012,BoothEtAl2013} to the case for description logics. This is not the first extension of this kind, as evidenced by the work of Boutilier~\cite{Boutilier1994}, Baltag and Smets~\cite{BaltagSmets2006,BaltagSmets2008}, Giordano~\etal.~\cite{GiordanoEtAl2007,GiordanoEtAl2009a,GiordanoEtAl2009b,GiordanoEtAl2012,GiordanoEtAl2013,GiordanoEtAl2015}, Britz~\etal.~\cite{BritzEtAl2013b,BritzEtAl2009b,BritzEtAl2008,BritzEtAl2009,BritzEtAl2011c} and Britz and Varzinczak~\cite{BritzVarzinczak2013,BritzVarzinczak2016b,BritzVarzinczak2018-JANCL,BritzVarzinczak2017-DL,BritzVarzinczak2018-JoLLI,BritzVarzinczak2018-FoIKS}. However, this is the first comprehensive semantic account of both preferential and rational subsumption relations, with accompanying representation results, based on the standard semantics for~description logics.

Informally, our semantic constructions are based on the idea that objects of the domain can be ordered according to their degree of {\em normality}~\cite{Boutilier1994} or {\em typicality}~\cite{BoothEtAl2012,BoothEtAl2013,BritzEtAl2009,GiordanoEtAl2007}. Paraphrasing Boutilier~\cite[pp.\ 110--116]{Boutilier1994},

\begin{quote}
Surely there is no {\em inherent} property of objects that allows them to be judged to be more or less normal in absolute terms. These orderings are purely `subjective' (in the sense that they can be thought of as part of an agent's belief state) and the space of orderings deemed plausible by the agent may (among other things) be determined by \eg\ empirical data. By using orderings in this way, we encode our (or the agent's) {\em expectations} about the objects corresponding to their perceived regularity or typicality. Those objects not violating our expectations are considered to be more normal than the objects that violate some.
\end{quote}

Hence we do not require that there exists something intrinsic about objects that makes one object more normal than another. Rather, the intention is to provide a framework in which to express all conceivable ways in which objects, with their associated properties and relationships with other objects, can be ordered in terms of typicality, in the same way that the class of all standard~DL interpretations constitute a framework representing all conceivable ways of representing the properties of objects and their relationships with other objects. Just as the latter are constrained by stating subsumption statements in a knowledge base, the possible orderings that are considered plausible are encoded by writing down DCIs.

That said, we are ready for the definition of the first semantic construction the present work relies on.

\begin{definition}[Preferential Interpretation]\label{Def:PrefInterpretation}
A \df{preferential interpretation} is a tuple $\PI\defined\tuple{\Dom^{\PI},\cdot^{\PI},\pref^{\PI}}$, where $\tuple{\Dom^{\PI},\cdot^{\PI}}$ is a (standard)~DL interpretation (which we denote by~$\I_{\PI}$ and refer to as the classical interpretation associated with~$\PI$), and~$\pref^{\PI}$ is a strict partial order on~$\Dom^{\PI}$ (\ie, $\pref^{\PI}$ is irreflexive and transitive) satisfying the smoothness condition (for every~$C\in\Lang$, if $C^{\PI}\neq\emptyset$, then $\min_{\pref^{\PI}}C^{\PI}\neq\emptyset$).\footnote{Given $X\subseteq\Dom^{\PI}$, with $\min_{\pref^{\PI}}X$ we denote the set $\{x\in X\mid$ for every $y\in X, y\npref^{\PI}x\}$.}
\end{definition}

Figure~\ref{Figure:PrefInterpretation} depicts a preferential interpretation in our scenario example where~$\Dom^{\PI}$ and~$\cdot^{\PI}$ are as in the interpretation~$\I$ shown in Figure~\ref{Figure:DLInterpretation}, and $\pref^{\PI}=\{(x_{7},x_{5}),(x_{8},x_{5}),(x_{9},x_{5}),(x_{5},x_{1})$, $(x_{7},x_{1}),(x_{8},x_{1}),(x_{9},x_{1}),(x_{9},x_{2})$, $(x_{10},x_{6})\}$, represented by the dashed arrows in the picture. (For the sake of presentation, in the picture we omit the transitive $\pref^{\PI}$-arrows.)

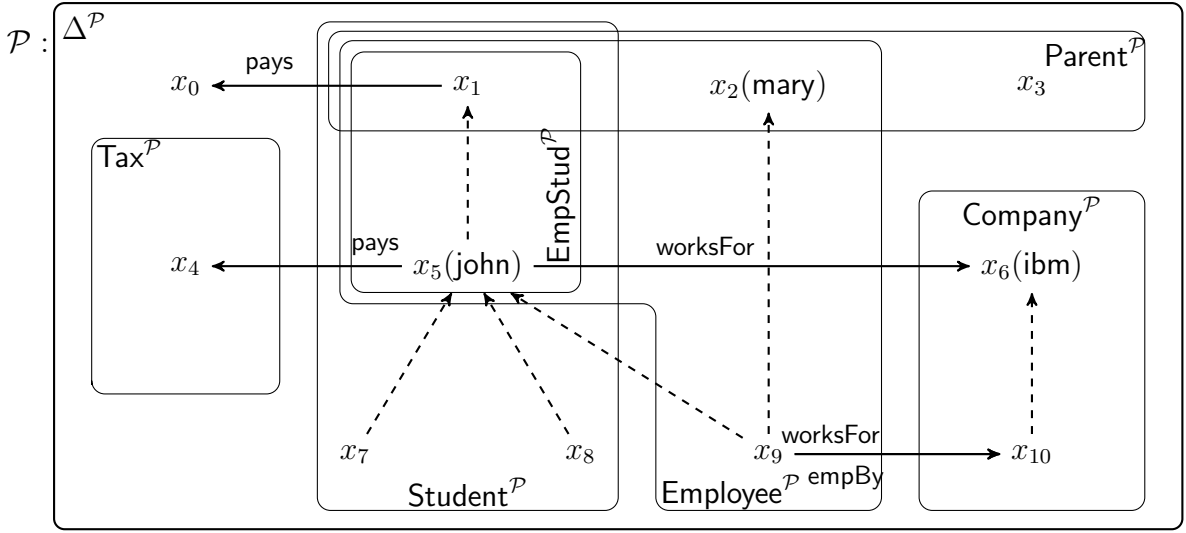
\begin{figure}[h]
\begin{center}
\begin{tikzpicture}[->,>=stealth', auto,node_style/.style={circle,fill=black,minimum size=0.1mm}]
\node at (-0.35,6.5) {$\PI:$};

\node at (0.4,6.7) {$\Delta^{\PI}$};
\node at (1,5) {$\Tax^{\PI}$};
\node at (13.85,6.35) {$\Parent^{\PI}$};
\node at (5.5,0.5) {$\Stud^{\PI}$};
\node at (9,0.5) {$\Emp^{\PI}$};
\node at (13,4.2) {$\Comp^{\PI}$};
\node at (6.7,4.45) {\rotatebox{90}{$\EmpStud^{\PI}$}};

\draw [rounded corners,thick] (0,0) rectangle (15,7) ;
\draw (0.5,1.8) [rounded corners=5pt] -- (0.5,5.2) -- (3,5.2) -- (3,1.8) -- (0.5,1.8) -- cycle ; 
\draw (3.5,0.25) [rounded corners=5pt] -- (3.5,6.25) -- (3.5,6.75) -- (7.5,6.75) -- (7.5,0.25) -- cycle ; 
\draw (3.65,5.3) [rounded corners=5pt] -- (3.65,6.625) -- (14.5,6.625) -- (14.5,5.3) -- cycle ; 
\draw (3.8,3) [rounded corners=5pt] -- (3.8,6.5) -- (11,6.5) -- (11,0.25) -- (8,0.25) -- (8,3) -- cycle ; 
\draw (11.5,0.25) [rounded corners=5pt] -- (11.5,4.5) -- (14.5,4.5) -- (14.5,0.25) -- cycle ; 
\draw (3.95,3.15) [rounded corners=5pt] -- (3.95,6.35) -- (7,6.35) -- (7,3.15) -- cycle ; 

\node (x0) at (1.75,5.9) {$x_{0}$} ;
\node (x1) at (5.5,5.9) {$x_{1}$} ;
\node (x2) at (9.5,5.9) {$x_{2}(\mary)$} ;
\node (x3) at (13,5.9) {$x_{3}$} ;
\node (x4) at (1.75,3.5) {$x_{4}$} ;
\node (x5) at (5.5,3.5) {$x_{5}(\john)$} ;
\node (x6) at (13,3.5) {$x_{6}(\ibm)$} ;
\node (x7) at (4,1) {$x_{7}$} ;
\node (x8) at (7,1) {$x_{8}$} ;
\node (x9) at (9.5,1) {$x_{9}$} ;
\node (x10) at (13,1) {$x_{10}$} ;
	
\path (x1) edge [thick,swap,near end] node {{\footnotesize$\pays$}} (x0);
\path (x5) edge [thick,swap,near start] node {\ \quad{\footnotesize$\pays$}} (x4);
\path (x5) edge [thick] node {{\footnotesize$\wk$}\quad\quad\quad\ } (x6);
\path (x9) edge [thick,near start] node {{\footnotesize$\wk$}\quad\ } (x10);
\path (x9) edge [thick,near start,swap] node {{\footnotesize$\emp$}\quad} (x10);


\path (x7) edge [thick,dashed] node {} (x5) ;
\path (x8) edge [thick,dashed] node {} (x5) ;
\path (x5) edge [thick,dashed] node {} (x1) ;
\path (x9) edge [thick,dashed] node {} (x5) ;
\path (x9) edge [thick,dashed] node {} (x2) ;
\path (x10) edge [thick,dashed] node {} (x6) ;

\end{tikzpicture}
\end{center}
\caption{A preferential interpretation.}
\label{Figure:PrefInterpretation}
\end{figure}

Preferential interpretations provide us with a simple and intuitive way to give a semantics to DCIs.

\begin{definition}[Satisfaction]\label{Def:Satisfaction}
Let $\PI$ be a preferential interpretation, $C,D\in\Lang$, $r\in\RN$ and $a,b\in\IN$. The \df{satisfaction relation}~$\sat$ is defined as follows:
\begin{itemize}\setlength{\itemsep}{1ex}
\item $\PI\sat C\subs D$ if $C^{\PI}\subseteq D^{\PI}$;
\item $\PI\sat C\dsubs D$ if $\min_{\pref^{\PI}}C^{\PI}\subseteq D^{\PI}$.
\end{itemize}
If $\PI\sat\alpha$, then we say $\PI$ \df{satisfies} $\alpha$. $\PI$ satisfies a defeasible knowledge base~$\KB$, written $\PI\sat\KB$, if $\PI\sat\alpha$ for every $\alpha\in\KB$, in which case we say~$\PI$ is a \df{preferential model} of~$\KB$. We say $C\in\Lang$ is \df{satisfiable} \wrt~$\KB$ if there is a model~$\PI$ of~$\KB$ \st\ $C^{\PI}\neq\emptyset$.
\end{definition}

It is easy to see that the addition of the $\pref^{\PI}$-component preserves the truth of all classical subsumption statements holding in the remaining structure:

\begin{lemma}\label{Lemma:TruthInInterpretation}
Let $\PI$ be a preferential interpretation. For every $C,D\in\Lang$, $\PI\sat C\subs D$ if and only if $\I_{\PI}\sat C\subs D$.
\end{lemma}

It is worth noting that, due to smoothness of $\pref^{\PI}$, every (classical) subsumption statement is equivalent, with respect to preferential interpretations, to some DCI.

\begin{restatable}{lemma}{restatableClassicalStatements}
\label{Lemma:ClassicalStatements}
For every preferential interpretation~$\PI$, and every $C,D\in\Lang$, $\PI\sat C\subs D$ if and only if $\PI\sat C\dland\lnot D\dsubs\bot$.
\end{restatable}

The following result, of which the proof can be found in Appendix~\ref{Sec:Sec3.2}, will come in handy later on.

\begin{restatable}{lemma}{restatableClosureDisjointUnionPref}
\label{Lemma:ClosureDisjointUnionPref}
Preferential interpretations are closed under disjoint union.
\end{restatable}
\myskip

An obvious question that can now be raised is: ``How do we know our preferential semantics provides an appropriate meaning to the notion of defeasible concept inclusion?'' The following definition will help us in answering this question:

\begin{definition}[$\PI$-Induced Defeasible Subsumption]\label{Def:PInducedSubsumption}
Let $\PI$ be a preferential interpretation. Then $\dsubs_{\PI}\defined\{(C,D) \mid \PI\sat C\dsubs D\}$ is the \df{defeasible subsumption relation induced by}~$\PI$.
\end{definition}

The first important result we present here, which also answers the above raised question, shows that there is a full correspondence between the class of preferential subsumption relations and the class of defeasible subsumption relations induced by preferential interpretations. It is the DL analogue of a representation result proved by Kraus~\etal.~for the propositional case~\cite[Theorem~3]{KrausEtAl1990} and its proof can be found in Appendix~\ref{ProofRepResultPreferential}.

\begin{restatable}{theorem}{restatableRepResultPreferential}[Representation Result for Preferential Subsumption]\label{Theorem:RepResultPreferential}
A defeasible subsumption relation $\dsubs\subseteq\Lang\times\Lang$ is preferential if and only if there is a preferential interpretation~$\PI$ such that $\dsubs_{\PI}=\dsubs$.
\end{restatable}

What is perhaps surprising about this result is that no additional properties based on the syntactic structure of the underlying~DL are necessary to characterise the defeasible subsumption relations induced by preferential interpretations. We provide below a few properties involving the use of quantifiers that are satisfied by all preferential subsumption relations. (See Section~\ref{RationalClosure} for more on properties explicitly mentioning DL-specific constructs.)

The first two are `existential' and `universal' versions of cautious monotonicity~(CM):
\[
(\text{{\small CM}}_{\exists})\ \frac{\exists r.C\dsubs E,\ \exists r. C\dsubs\forall r.D}{\exists r.(C\dland D)\dsubs E}
\]\[
(\text{{\small CM}}_{\forall})\ \frac{\forall r. C\dsubs E,\ \forall r.C\dsubs\forall r.D}{\forall r.(C\dland D)\dsubs E}
\]

The third one is a rephrasing of the Rule of Necessitation in modal logic~\cite{Chellas1980}. It guarantees the absence of so-called \emph{spurious objects}~\cite{BritzEtAl2012} in the original preferential semantics for DLs by Britz~\etal.~\cite{BritzEtAl2011b,BritzEtAl2011c}. That is, if $C$ is unsatisfiable, then so is $\exists r.C$ (\cf\ Lemma~\ref{Lemma:ClassicalStatements}).
\[
\text{(Norm)}\ \frac{C\dsubs\bot}{\exists r.C\dsubs\bot}
\]

In addition to preferential interpretations, we are also interested in the study of \emph{modular} interpretations, which are preferential interpretations in which the $\pref$-component is a \emph{modular} ordering:

\begin{definition}[Modular Order]\label{Def:Modular}
Given a set $X$, $\pref\ \subseteq X\times X$ is \df{modular} if it is a strict partial order, and its associated incomparability relation $\incomp$, defined by $x\incomp y$ if neither $x\pref y$ nor $y\pref x$, is transitive.
\end{definition}

If $\pref$ is modular, then $\incomp$ is an equivalence relation.

\begin{definition}[Modular Interpretation]\label{Def:ModInterpretation}
A \df{modular interpretation} is a preferential interpretation $\RI=\tuple{\Dom^{\RI},\cdot^{\RI},\pref^{\RI}}$ such that~$\pref^{\RI}$ is modular.
\end{definition}

Intuitively, modular interpretations allow us to compare any two objects \wrt\ their plausibility. Those that are  incomparable are viewed as being equally plausible. As such, modular interpretations are special cases of preferential interpretations, where plausibility can be represented by any smooth strict partial order.


The main reason to consider modular interpretations is that they provide the semantic foundation of rational subsumption relations. This is made precise by our second important result below, which shows that the defeasible subsumption relations induced by modular interpretations are precisely the rational subsumption relations. Again, this is the~DL analogue of a representation result proved by Lehmann and Magidor for the propositional case~\cite[Theorem~5]{LehmannMagidor1992} and its proof can be found in Appendix~\ref{ProofRepResultRational}.

\begin{restatable}{theorem}{restatableRepResultRational}[Representation Result for Rational Subsumption]\label{Theorem:RepResultRational}
A defeasible subsumption relation $\dsubs\subseteq\Lang\times\Lang$ is rational if and only if there is a modular interpretation~$\RI$ such that $\dsubs_{\RI}=\dsubs$.
\end{restatable}

Analogous to the case for cautious monotonicity above, the following `existential' and `universal' versions of rational monotonicity are satisfied by all rational subsumption relations:
\[
(\text{{\small RM}}_{\exists})\ \frac{\exists r.C\dsubs E,\ \exists r. C\ndsubs\forall r.\lnot D}{\exists r.(C\dland D)\dsubs E}
\]\[
(\text{{\small RM}}_{\forall})\ \frac{\forall r. C\dsubs E,\ \forall r.C\ndsubs\forall r.\lnot D}{\forall r.(C\dland D)\dsubs E}
\]

It is worth pausing for a moment to emphasise the significance of these two results (Theorems~\ref{Theorem:RepResultPreferential} and~\ref{Theorem:RepResultRational}). They provide exact semantic characterisations of two important classes of defeasible subsumption relations, namely preferential and rational subsumption, in terms of the classes of preferential and modular interpretations, respectively. As we shall see in Section~\ref{Entailment}, these results form the core of the investigation into an appropriate notion of entailment for defeasible~DL ontologies.

\section{Rationality in entailment}\label{Entailment}

\newcommand{\pre}{\ensuremath{\mathsf{pref}}}
\newcommand{\modular}{\ensuremath{\mathsf{mod}}}
\newcommand{\rat}{\ensuremath{\mathsf{rat}}}

From the standpoint of knowledge representation and reasoning, a pivotal question is that of deciding which statements are \emph{entailed} by a knowledge base. We shall devote the remainder of the paper to this matter, and in this section we lay out the formal foundations for that.

\subsection{Preferential entailment}\label{PreferentialEntailment}

In the exploration of a notion of entailment for defeasible ontologies, an obvious starting point is to consider a Tarskian definition of consequence:

\begin{definition}[Preferential Entailment]\label{Def:PrefEntailment}
A statement~$\alpha$ is \df{preferentially entailed} by a defeasible knowledge base~$\KB$, written $\KB\entails_{\pre}\alpha$, if every preferential model of~$\KB$ satisfies~$\alpha$.
\end{definition}

As usual, this form of entailment is accompanied by a corresponding notion of closure.

\begin{definition}[Preferential Closure]\label{Def:PrefClosure}
Let~$\KB$ be a defeasible knowledge base. With $\KB^{*}_{\pre}\defined\{\alpha\ \mid\ \KB\entails_{\pre}\alpha\}$ we denote the \df{preferential closure} of~$\KB$.
\end{definition}

Intuitively, the preferential closure of a defeasible knowledge base~$\KB$ corresponds to the `core' set of statements, classical and defeasible, that should hold given those in~$\KB$. Hence, preferential entailment and preferential closure are two sides of the same coin, mimicking an analogous result for preferential reasoning in the propositional~\cite{KrausEtAl1990} case.

Recall (\cf~the discussion following Definition~\ref{Def:DefeasibleTbox}) that a defeasible theory $\Th$ is a defeasible knowledge base without the restriction to finite sets. When assessing how appropriate a notion of entailment for defeasible ontologies is, the following definitions turn out to be useful, as will become clear in the sequel:

\begin{definition}[$\Th$-Induced Defeasible Subsumption]\label{Def:KBInducedSubsumption}
Let $\Th$ be a defeasible theory. Then (1)~$\DB_{\Th}\defined\{C\dsubs D \mid C\dsubs D\in\Th\}\cup\{C\dland\neg D\dsubs\bot \mid C\subs D\in\Th\}$ is the \df{DTBox induced by~$\Th$} and (2)~$\dsubs_{\Th}\defined\{(C,D) \mid C\dsubs D\in\DB_{\Th}\}$ is the \df{defeasible subsumption relation induced by}~$\Th$.
\end{definition}

So, the DTBox induced by $\Th$ is the set of defeasible subsumption statements contained in $\Th$, together with the defeasible versions of the classical subsumption statements in $\Th$. The defeasible subsumption relation induced by~$\Th$ is simply the defeasible subsumption relation corresponding to $\DB_{\Th}$.

\begin{definition}\label{Def:PreferentialKB}
A defeasible theory~$\Th$ is called \df{preferential} if the subsumption relation induced by it satisfies the preferential properties in Definition~\ref{Def:PrefSubsumption}.
\end{definition}

It turns out that the defeasible subsumption relation induced by the preferential closure of a defeasible knowledge base $\KB$ is exactly the intersection of the defeasible subsumption relations induced by the preferential defeasible theories containing $\KB$.

\begin{restatable}{lemma}{restatableLemmaPrefClosure}\label{Lemma:PrefClosure}
Let~$\KB$ be a defeasible knowledge base. Then
\[
\dsubs_{\KB^{*}_{\pre}} = \bigcap\{\dsubs_{\Th} \mid \KB\subseteq\Th \text{ and }\Th\text{ is preferential}\}.
\]
\end{restatable}


It follows immediately that the preferential closure of a defeasible knowledge base $\KB$ is preferential, and induces the smallest defeasible subsumption relation induced by a preferential defeasible theory containing $\KB$.






Preferential entailment is not always desirable, one of the reasons being that it is monotonic, courtesy of the Tarskian notion of consequence it relies on (see Definition~\ref{Def:PrefEntailment}). In most cases, as witnessed by the great deal of work in the non-monotonic reasoning community, a move towards rationality is in order. Thanks to the definitions above and the result in Theorem~\ref{Theorem:RepResultRational}, we already know where to start looking for it.

\begin{definition}[Modular Entailment]\label{Def:ModularEntailment}
A statement~$\alpha$ is \df{modularly entailed} by a defeasible knowledge base~$\KB$, written $\KB\entails_{\modular}\alpha$, if every modular model of~$\KB$ satisfies~$\alpha$.
\end{definition}

As is the case for preferential entailment, modular entailment is accompanied by a corresponding notion of closure.

\begin{definition}[Modular Closure]\label{Def:ModularClosure}
Let~$\KB$ be a defeasible knowledge base. With $\KB^{*}_{\modular}\defined\{\alpha\ \mid\ \KB\entails_{\modular}\alpha\}$ we denote the \df{modular closure} of~$\KB$.
\end{definition}

\begin{definition}\label{Def:RationalalKB}
A defeasible theory~$\Th$ is called \df{rational} if it is preferential and $\dsubs_{\Th}$ is also closed under the rational monotonicity rule~(RM).
\end{definition}

For modular closure we get a result similar to Lemma \ref{Lemma:PrefClosure}.
\begin{restatable}{lemma}{restatableLemmaModClosure}\label{Lemma:ModClosure}
Let~$\KB$ be a defeasible knowledge base. Then
\[
\dsubs_{\KB^{*}_{\modular}} = \bigcap\{\dsubs_{\Th} \mid \KB\subseteq\Th \text{ and }\Th\text{ is rational}\}.
\]
\end{restatable}




That is, the modular closure of a defeasible knowledge base $\KB$ induces the smallest defeasible subsumption relation induced by a rational defeasible theory containing $\KB$. However, the modular closure of a defeasible knowledge base $\KB$ is not necessarily rational. That is, if one looks at the set of statements (in particular the $\dsubs$-ones) modularly entailed by a knowledge base as a defeasible subsumption relation, then it need not satisfy the rational monotonicity property. This is so because modular entailment coincides with preferential entailment, as the following result, adapted from a well-known similar result in the propositional case~\cite[Theorem~4.2]{LehmannMagidor1992}, shows.



\begin{lemma}\label{Lemma:ModClosurePreferential}
$\KB^{*}_{\modular}=\KB^{*}_{\pre}$.
\end{lemma}

As a result, modular entailment unfortunately falls short of providing us with an appropriate notion of non-monotonic entailment. In what follows, we overcome precisely this issue.

\subsection{Semantic rational entailment}\label{RationalEntailment}

In this section, we introduce a definition of semantic entailment which, as we shall see, is appropriate in the light of the discussion above. The constructions we are going to present are inspired by the semantic characterisation of rational closure by Booth and Paris~\cite{BoothParis1998} in the propositional case.
We shall give a corresponding proof-theoretic characterisation of our version of semantic entailment in Section~\ref{DefinitionRC}.
\myskip

We focus our attention on  a subclass of modular orders, referred to as \emph{ranked orders}:

\begin{definition}[Ranked Order]\label{Def:Ranked_order}
Given a set $X$, the binary relation $\pref\ \subseteq X\times X$ is a \df{ranked order} if there is a mapping $h_{\RI}:X\longrightarrow\mathbb{N}$ satisfying the following convexity property:
\begin{itemize}
\item for every $i\in \mathbb{N}$, if for some $x\in X$ $h_{\RI}(x)=i$, then, for every $j$ such that $0\leq j<i$, there is a $y\in X$ for which $h_{\RI}(y)=j$,
\end{itemize}
and \st\ for every $x,y\in X$, $x\pref y$ iff $h_{\RI}(x)<h_{\RI}(y)$.
\end{definition}

It is easy to see that a ranked order $\pref$ is also modular: $\pref$ is  a strict partial order, and, since two objects $x,y$ are incomparable (\ie, $x\sim y$) if and only if $h_{\RI}(x)=h_{\RI}(y)$, $\sim$ is a transitive relation. By constraining our preference relations to the ranked orders, we can identify a subset of the modular interpretations we refer to as the \emph{ranked interpretations}.

\begin{definition}[Ranked Interpretation]\label{Def:Ranked_interpretation}
A \df{ranked interpretation} is a modular interpretation $\RI=\tuple{\Dom^{\RI},\cdot^{\RI},\pref^{\RI}}$ s.t. $\pref^{\RI}$ is a ranked order.
\end{definition}

We now provide two basic results about ranked interpretations. First, all finite modular interpretations are ranked interpretations.


\begin{restatable}{lemma}{restatablefinitemodular}\label{finite_modular}
A modular interpretation $\RI=\tuple{\Dom^{\RI},\cdot^{\RI},\pref^{\RI}}$ s.t. $\Dom^{\RI}$ is finite is a ranked interpretation.
\end{restatable}




Next, for every ranked interpretation~$\RI$, the function $h_{\RI}(\cdot)$ is unique.

\begin{restatable}{proposition}{restatableuniquerank}\label{unique_rank}
Given a ranked interpretation $\RI=\tuple{\Dom^{\RI},\cdot^{\RI},\pref^{\RI}}$, there is only one function $h_{\RI}:X\longrightarrow\mathbb{N}$ satisfying the convexity property and s.t. for every $x,y\in X$, $x\pref y$ iff $h_{\RI}(x)<h_{\RI}(y)$.
\end{restatable}



Proposition~\ref{unique_rank} allows us to use the function $h_\RI(\cdot)$ to define the notions of \emph{height} and \emph{layers}.

\begin{definition}[Height \& Layers]\label{Def:height}
Given a ranked interpretation $\RI=\tuple{\Dom^{\RI},\cdot^{\RI},\pref^{\RI}}$, its characteristic ranking function $h_{\RI}(\cdot)$, and an object $x\in\Dom^{\RI}$, $h_{\RI}(x)$ is called the \df{height} of~$x$ in~$\RI$.

For every ranked interpretation $\RI=\tuple{\Dom^{\RI},\cdot^{\RI},\pref^{\RI}}$, we can partition the domain~$\Dom^{\RI}$ into a sequence of \emph{layers} $(L_{0},\ldots,L_{n},\ldots)$, where, for every object $x\in\Dom^{\RI}$, we have $x\in L_{i}$ iff $h_{\RI}(x)= i$.
\end{definition}

Intuitively, the lower the height of an object in an interpretation~$\RI$, the more typical (or normal) the object is in~$\RI$. We can also think of a level of typicality for concepts: the height of a concept $C\in\Lang$ in~$\RI$ is the index of the layer to which the restriction of the concept's extension to its $\pref^{\RI}$-minimal elements belong, \ie, $h_{\RI}(C)=i$ if $\emptyset\subset\min_{\pref^{\RI}}C^{\RI}\subseteq L_{i}$. As a convention, if $\min_{\pref^{\RI}}C^{\RI}=\emptyset$, that is, if $C^{\RI}=\emptyset$, then $h_{\RI}(C)=\infty$.
\myskip

The following result (proved in Appendix~\ref{ProofsEntailment}) will be useful for some of the proofs in later sections of the paper: 

\begin{restatable}[Finite-Model Property]{theorem}{restatableFiniteModelProperty}\label{Theorem:FMP}
Defeasible \ALC\ has the finite-model property. In particular, every  defeasible \ALC\ knowledge base that has a modular model, has also a finite ranked  model.
\end{restatable}
\myskip

Given a set of ranked interpretations, we can introduce a new form of model merging, \emph{ranked union}.

\begin{definition}[Ranked Union]
Given a countable set of ranked interpretations $\mathfrak{R}=\{\RI_1,\RI_2,\ldots\}$, a ranked interpretation $\RI^{\mathfrak{R}}\defined\tuple{\Dom^{\mathfrak{R}},\cdot^{\mathfrak{R}},\pref^{\mathfrak{R}}}$ is the \df{ranked union} of $\mathfrak{R}$ if the following holds:
\begin{itemize}
\item $\Dom^{\mathfrak{R}}\defined\coprod_{\RI\in\mathfrak{R}}\Delta^{\RI}$, \ie, the disjoint union of the domains from $\mathfrak{R}$, where each $\RI\in\mathfrak{R}$ has the elements $x,y,\ldots$ of its domain renamed as $x_{\RI}$, $y_{\RI}$, \ldots\ so that they are all distinct in $\Dom^{\mathfrak{R}}$;
\item $x_{\RI}\in A^{\mathfrak{R}}$ iff $x\in A^{\RI}$;
\item $(x_{\RI},y_{\RI'})\in r^{\mathfrak{R}}$ iff $\RI=\RI'$ and $(x,y)\in r^{\RI}$;
\item for every $x_{\RI}\in \Dom^{\mathfrak{R}}$, $h_{\mathfrak{R}}(x_{\RI})=h_{\RI}(x)$.
\end{itemize}
The latter condition corresponds to imposing that $x_{\RI}\pref^{\mathfrak{R}}y_{\RI'}$ iff $h_{\RI}(x)<h_{\RI'}(y)$.
\end{definition}

Informally, the ranked union of a set of ranked interpretations is the result of merging all their layers of height $i$ into a single layer of height $i$, for all $i$. 


\begin{restatable}{lemma}{restatableClosureDisjointUnionMod}\label{Lemma:ClosureDisjointUnionMod}
Given a set of ranked models of a defeasible knowledge base $\KB$, their ranked union is itself a ranked model of $\KB$.
\end{restatable}

Let $\KB$ be a defeasible knowledge base and let $\Delta$ be a fixed countably infinite set. Define
\[
\Mod_{\Dom}(\KB) \defined \{\RI=\tuple{\Dom^{\RI},\cdot^{\RI},\pref^{\RI}} \mid \RI\sat\KB, \RI\text{ is ranked and } \Dom^{\RI}=\Delta\}.
\] 
The following result shows that the set $\Mod_{\Dom}(\KB)$ suffices to characterise modular entailment (the proof is in Appendix~\ref{ProofsEntailment}):

\begin{restatable}{lemma}{restatableCountablyInfiniteDomain}
\label{Lemma:CountablyInfiniteDomain}
For every $\KB$ and every $C,D\in\Lang$, $\KB\entails_{\modular}C\dsubs D$ iff $\RI\sat C\dsubs D$, for every $\RI\in\Mod_{\Dom}(\KB)$.
\end{restatable}

Therefore, we can use just the set of interpretations in $\Mod_{\Dom}(\KB)$ to decide the consequences of~$\KB$ \wrt\ modular entailment.

We can now use the set $\Mod_{\Dom}(\KB)$ as a springboard to introduce what will turn out to be a canonical modular interpretation for $\KB$. Using $\Mod_{\Dom}(\KB)$ and ranked union we can define the following relevant model.

\begin{definition}[Big Ranked Model]\label{Def:Bigrankedmodel}
Let $\KB$ be a defeasible knowledge base.  The \df{big ranked model} of \KB~is the ranked model $\OI\defined\tuple{\Dom^{\OI},\cdot^{\OI},\pref^{\OI}}$ that is the ranked union of the models in $\Mod_{\Dom}(\KB)$.
\end{definition}





Given Lemma~\ref{Lemma:ClosureDisjointUnionMod}, we can state the following:

\begin{corollary}\label{Corollary:IsModularModel}
$\OI$ is a ranked model of~$\KB$.
\end{corollary}

Armed with the definitions and results above, we are now ready to provide an alternative definition of entailment in the context of defeasible ontologies:

\begin{definition}[Rational Entailment]\label{Def:RationalEntailment}
A statement~$\alpha$ is \df{rationally entailed} by a knowledge base~$\KB$, written $\KB\entails_{\rat}\alpha$, if $\OI\sat\alpha$.
\end{definition}

That such a notion of entailment indeed deserves its name is witnessed by the following result, a consequence of Corollary~\ref{Corollary:IsModularModel} and Theorem~\ref{Theorem:RepResultRational}:

\begin{corollary}
Let $\KB$ be a defeasible knowledge base. $\{C\dsubs D\mid \OI\sat C\dsubs D\}$ is rational.
\end{corollary}

In conclusion, rational entailment is a good candidate for the appropriate notion of defeasible consequence we have been looking for. Of course, a question that arises is whether a notion of closure, in the spirit of preferential and modular closures, that is equivalent to it can be defined. In the next section, we address precisely this matter.

\section{Rational closure for defeasible knowledge bases}\label{RationalClosure}

We now turn our attention to the exploration, in a DL setting, of the well-known notion of \emph{rational closure} of a defeasible knowledge base as studied by Lehmann and Magidor~\cite{LehmannMagidor1992} for propositional logic. For the most part, we base our constructions on the work by Casini and Straccia~\cite{CasiniStraccia2010,CasiniStraccia2013}, amending it wherever necessary. (An alternative semantic characterisation of rational closure in DLs has also been proposed by Giordano~\etal.~\cite{GiordanoEtAl2013,GiordanoEtAl2015}.) As we shall see, rational closure provides a proof-theoretic characterisation of rational entailment and the complexity of its computation is no higher than that of computing entailment in the underlying classical~DL.

\subsection{Rational closure and a correspondence result}\label{DefinitionRC}

Rational closure is a form of inferential closure based on modular entailment $\entails_{\modular}$, but it extends its inferential power. Such an extension of modular entailment is obtained by formalising the already mentioned principle of \emph{presumption of typicality}~\cite[Section 3.1]{Lehmann1995}. That is, under possibly incomplete information, we always assume that we are dealing with the most typical possible situation that is compatible with the information at our disposal. We first define what it means for a concept to be {\em exceptional}, a notion that is central to the definition of rational closure:

\begin{definition}[Exceptionality]\label{Def:Exceptionality}
Let $\KB$ be a defeasible knowledge base and $C\in\Lang$. We say~$C$ is \df{exceptional} in $\KB$ if $\KB\entails_{\modular}\top\dsubs\lnot C$. A DCI $C\dsubs D$ is exceptional in \KB\ if $C$ is exceptional in~$\KB$.
\end{definition}

A concept $C$ is considered exceptional in a knowledge base~$\KB$ if it is not possible to have a modular model of~$\KB$ in which there is a typical object (\ie, an object at least as typical as all the others) that is in the interpretation of~$C$. Intuitively, a DCI is exceptional if it does not concern the most typical objects, \ie, it is about less normal (or exceptional) ones. This is an intuitive translation of the notion of exceptionality used by Lehmann and Magidor~\cite{LehmannMagidor1992} in the propositional framework, and has already been used by Casini and Straccia~\cite{CasiniStraccia2010} and Giordano~\etal.~\cite{GiordanoEtAl2015} in their investigations into defeasible reasoning for description logics.

Applying the notion of exceptionality iteratively, we associate with every concept~$C$ a \emph{rank} in~$\KB$, which we denote by $\rank_{\KB}(C)$. We extend this to DCIs and associate with every statement $C\dsubs D$ a rank, denoted $\rank_{\KB}(C\dsubs D)$:

\begin{enumerate}
\item Let $\rank_{\KB}(C)=0$, if $C$ is not exceptional in~$\KB$, and let $\rank_{\KB}(C\dsubs D)=0$ for every DCI having~$C$ in the LHS, with $\rank_{\KB}(C)=0$. The set of DCIs in~$\DB$ with rank~$0$ is denoted as $\DB^\rank_{0}$.
\item Let $\rank_{\KB}(C)=1$, if $C$ does not have a rank of~$0$ and it is not exceptional in the knowledge base $\KB^1$ composed of $\TB$~and the exceptional part of~$\DB$, that is, $\KB^1=\tuple{\TB,\DB\setminus\DB^\rank_{0}}$. If $\rank_{\KB}(C)=1$, then let $\rank_{\KB}(C\dsubs D)=1$ for every DCI $C\dsubs D$. The set of DCIs in~$\DB$ with rank $1$ is denoted $\DB^\rank_{1}$.
\item In general, for $i>0$, a concept~$C$ is assigned a rank of~$i$ if it does not have a rank of $i-1$ and it is not exceptional in~$\KB^{i}=\tuple{\TB,\DB\setminus\bigcup_{j=0}^{i-1}\DB^{\rank}_{j}}$. If $\rank_{\KB}(C)=i$, then $\rank_{\KB}(C\dsubs D)=i$, for every DCI $C\dsubs D$ having~$C$ in the LHS. The set of DCIs in~$\DB$ with rank~$i$ is denoted $\DB^{\rank}_{i}$.
\item By iterating the previous steps, we eventually reach a subset $\E\subseteq\DB$ such that all the DCIs in~$\E$ are exceptional (since~$\DB$ is finite, we must reach such a point). If $\E\neq\emptyset$, we define the rank of the DCIs in $\E$ as~$\infty$, and the set~$\E$ is denoted $\DB^\rank_\infty$. Moreover, we set $\rank_{\KB}(C)=\infty$ for every~$C$ in the LHS of some DCI in~$\DB^\rank_\infty$.
\end{enumerate}
The notion of rank can also be extended to GCIs as follows: 
$\rank_{\KB}(C\subs D) = \rank_{\KB}(C\dland\lnot D)$.

Following on the procedure above, the DTBox~$\DB$ is partitioned into a finite sequence $\tuple{\DB^{\rank}_{0},\ldots,\DB^{\rank}_{n},\DB^{\rank}_{\infty}}$ ($n\geq 0$), where $\DB^{\rank}_{\infty}$ may possibly be empty. So, through this procedure we can assign a rank to every~DCI.

We can check that for a concept $C$ has a rank of $\infty$ iff it is not satisfiable in any model of~$\KB$, that is, $\KB\entails_{\modular} C\subs\bot$.

\begin{restatable}{lemma}{restatableInfiniteRank}
\label{Lemma:InfiniteRank}
For every knowledge base $\KB$ and every concept $C$, $\rank_{\KB}(C)=\infty$ iff $\KB\entails_{\modular}C\subs\bot$.
\end{restatable}

\begin{example}\label{Example:Exceptionality}
Let $\KB=\TB\cup\DB$, where~$\TB$ and~$\DB$ are as in Example~\ref{Example:Student}, \ie, $\TB=\{\EmpStud\subs\Stud\}$ and 
\[
\DB=\left\{\begin{array}{c}
						\Stud\dsubs\lnot\exists\pays.\Tax,\\[0.1cm]
						\EmpStud\dsubs\exists\pays.\Tax,\\[0.1cm]
						\EmpStud\dland\Parent\dsubs\lnot\exists\pays.\Tax\\
					 \end{array}
		\right\}
\]
Examining the concepts on the LHS of each DCI in~$\KB$, one can verify that $\Stud$ is not exceptional \wrt~$\KB$. Therefore, $\rank_{\KB}(\Stud)=0$. We also find that $\rank_{\KB}(\EmpStud)\neq{0}$ and $\rank_{\KB}(\EmpStud\dland\Parent)\neq{0}$ because both concepts are exceptional \wrt~$\KB$. Hence $\DB^{\rank}_{0}=\{\Stud\dsubs\lnot\exists\pays.\Tax\}$ and $\KB^{0}=\TB\cup\DB^{\rank}_{0}$.

$\KB^{1}$ is composed of~$\TB$ and~$\DB\setminus\DB^{\rank}_{0}$. We find that $\EmpStud$ is \emph{not} exceptional \wrt~$\KB^{1}$ and therefore $\rank_{\KB}(\EmpStud)=1$. Since $\EmpStud\dland\Parent$ is exceptional \wrt~$\KB^{1}$, $\rank_{\KB}(\EmpStud\dland\Parent)\neq{1}$. Thus $\DB^{\rank}_{1}=\{\EmpStud\dsubs\exists\pays.\Tax\}$. Similarly, $\KB^{2}$ is composed of~$\TB$ and $\{\EmpStud\dland\Parent\dsubs\lnot\exists\pays.\Tax\}$. We have that $\EmpStud\dland\Parent$ is not exceptional \wrt~$\KB^{2}$ and therefore $\rank_{\KB}(\EmpStud\dland\Parent)=2$. Finally, for this example, $\DB^{\rank}_{\infty}=\emptyset$. \hfill\ \qed
\end{example}

Adapting Lehmann and Magidor's construction for propositional logic~\cite{LehmannMagidor1992}, the  rational closure of a defeasible knowledge base~$\KB$ is defined as follows:

\begin{definition}[Rational Closure]\label{Def:RationalClosure}
Let $\KB$ be a defeasible knowledge base and $C,D\in\Lang$.
\begin{enumerate}
\item $C\dsubs D$ is in the rational closure of~$\KB$~if 
\[
\rank_{\KB}(C\dland D)<\rank_{\KB}(C\dland\lnot D) \text{ or } \rank_{\KB}(C)=\infty.
\]
\item $C\subs D$ is in the rational closure of~$\KB$~if $\rank_{\KB}(C\dland\lnot D)=\infty$.
\end{enumerate}
\end{definition}

Informally, the definition above says that $C\dsubs D$ is in the rational closure of~$\KB$ if the modular models of~$\KB$ tell us that some instances of $C\dland D$ are more plausible than all instances of $C\dland\lnot D$, while $C\subs D$ is in the rational closure of~$\KB$ if the instances of $C\dland\lnot D$ are impossible.
\myskip

\begin{exContinued}
Applying the definition above to the knowledge base in Example~\ref{Example:Exceptionality}, we can verify that $\Stud\dsubs\lnot\exists\pays.\Tax$ is in the rational closure of~$\KB$ because $\rank_{\KB}(\Stud\dland\lnot\exists\pays.\Tax)=0$ and $\rank_{\KB}(\Stud\dland\exists\pays.\Tax)>0$. The latter can be derived from the fact that $\Stud\dland\exists\pays.\Tax$ is exceptional \wrt~$\KB$. Similarly, one can derive that both DCIs $\EmpStud\dsubs\exists\pays.\Tax$ and $\EmpStud\dland\Parent\dsubs\lnot\exists\pays.\Tax$ are in the rational closure of~$\KB$ as well. \hfill\ \qed
\end{exContinued}

We now state the main result of the present section, which provides an answer to the question raised at the end of Section~\ref{RationalEntailment}. (The proof can be found in Appendix~\ref{ProofsComputingRC}.)

\begin{restatable}{theorem}{restatableCharacterisation}
\label{Theorem:Characterisation}
Let $\KB$ be a defeasible knowledge base having a modular model. A statement $\alpha$ is in the rational closure of~$\KB$ iff $\KB\entails_{\rat}\alpha$.
\end{restatable}

An easy corollary of this result is that rational closure preserves the equivalence between GCIs of the form $C\subs D$ and their defeasible counterparts ($C\dland\lnot D\dsubs\bot$).

\begin{corollary}
$C\subs D$ is in the rational closure of a defeasible knowledge base $\KB$ iff $C\dland\lnot D\dsubs\bot$ is in the restriction of the closure of $\KB$ under rational entailment to defeasible concept inclusions.
\end{corollary}

Rational entailment from a knowledge base can therefore be formulated as membership checking of the rational closure of the knowledge base. Of course, from an application-oriented point of view, this raises the question of how to compute membership of the rational closure of a knowledge base, and what is the complexity thereof. This is precisely the topic of the next section.

\subsection{Rational entailment checking}\label{ComputingRC}

\newcommand{\mat}[1]{\overline{#1}}
\newcommand{\excep}{\ensuremath{\mathsf{Exceptional}}}
\newcommand{\compRank}{\ensuremath{\mathsf{ComputeRanking}}}
\newcommand{\compClosure}{\ensuremath{\mathsf{RationalClosure}}}
\newcommand{\fix}{\ensuremath{\mathsf{fix}}}
\newcommand{\Cn}[1]{{\ensuremath{\text{\it Cn}_{#1}}}}

We now present an algorithm to effectively check the rational entailment of a DCI from a defeasible knowledge base. Our algorithm is a modification of the one given by Casini and Straccia~\cite{CasiniStraccia2010} for defeasible \ALC. Their algorithm had to be modified slightly since it does not always give back the correct result in case $\DB^\rank_\infty\neq\emptyset$ --- \cf\ Item~4 in the description in Section~\ref{DefinitionRC}.
\myskip

Let $\KB=\TB\cup\DB$ be a defeasible knowledge base. The first step of the algorithm is to assign a rank to each DCI in~$\DB$. Central to this step is the exceptionality function $\excep(\cdot)$, which computes the semantic notion of exceptionality of Definition~\ref{Def:Exceptionality}. The function makes use of the notion of \emph{materialisation} to reduce concept exceptionality checking to entailment checking:

\begin{definition}[Materialisation]\label{Def:Materialisation}
Let $\DB$ be a set of DCIs. With $\mat{\DB}\defined\{\lnot C\dlor D \mid C\dsubs D\in\DB\}$ we denote the \df{materialisation} of~$\DB$.
\end{definition}

We can show that, given $\KB=\TB\cup\DB$ and $\DB'\subseteq\DB$, if  $\TB\entails\bigsqcap\mat{\DB'}\subs\lnot C$, where~$\entails$ denotes \emph{classical}~\ALC\ entailment, a DCI $C\dsubs D$ is exceptional \wrt\ $\TB\cup\DB'$, thereby justifying the use of Line~3 of function~$\excep(\cdot)$.  The proof of the following lemma can be found in Appendix~\ref{ProofsComputingRC}.

\begin{restatable}{lemma}{restatableExceptionality}
\label{Lem:Exceptionality}
For $\KB=\TB\cup\DB$, if $\TB\entails\bigsqcap\mat{\DB}\subs\lnot C$, then $C\dsubs D$ is exceptional \wrt\ $\TB\cup\DB$.
\end{restatable}

Given a set of DCIs $\DB'\subseteq\DB$, $\excep(\TB,\DB')$ returns a subset~$\E$ of~$\DB'$ such that $\E$ is exceptional \wrt\ $\TB\cup\DB'$.

\begin{function}[ht]
\KwIn{$\TB$ and $\DB'\subseteq\DB$}
\KwOut{$\E\subseteq\DB'\text{ such that }\E\text{ is exceptional \wrt\ }\TB\cup\DB'$}
$\E\assigned\emptyset$\;
\ForEach{$C\dsubs D\in\DB'$}{	
   	\If{$\TB\entails\bigsqcap\mat{\DB'}\subs\lnot C$}{$\E\assigned\E\cup\{C\dsubs D\}$}
 }
\textbf{return} $\E$
\caption{Exceptional($\TB,\DB'$)}\label{Func:Exceptional}
\end{function}

While the converse of Lemma~\ref{Lem:Exceptionality} does not hold, it follows from Lemma~\ref{Lemma:SameRankings} below that this reduction to classical entailment checking, when applied iteratively (lines 4--14 in Algorithm $\compRank(\cdot)$), fully captures the semantic notion of exceptionality of Definition~\ref{Def:Exceptionality}.

\begin{exContinued}
If we feed the knowledge base in Example~\ref{Example:Exceptionality} to the function $\excep(\cdot)$, we obtain the output
\[
\E=\{\EmpStud\dsubs\exists\pays.\Tax, \EmpStud\dland\Parent\dsubs\lnot\exists\pays.\Tax\}.
\]
This is because both concepts on the LHS of the DCIs in~$\DB'$ are exceptional \wrt~$\KB$ in Example~\ref{Example:Exceptionality}. \hfill\ \qed
\end{exContinued}

We now describe the overall ranking algorithm, presented in the function $\compRank(\cdot)$ below. The algorithm makes a finite sequence of calls to the function $\excep(\cdot)$, starting from the knowledge base $\KB=\TB\cup\DB$. The algorithm terminates with a partitioning of the axioms in the DTBox, from which a ranking of axioms can easily be obtained.


\begin{function}[ht]
\caption{ComputeRanking($\KB$)\label{Func:Ranking}}
\KwIn{$\KB=\TB\cup\DB$}
\KwOut{$\KB^*=\TB^*\cup\DB^*$, a partitioning $R=\{\DB_{0},\ldots,\DB_{n}\}$ for $\DB^*$,  and an exceptionality ranking $\E$}
$\TB^*\assigned\TB$\;
$\DB^*\assigned\DB$\;
$R\assigned\emptyset$\;
\Repeat{$\DB^*_{\infty}=\emptyset$}
	{
	$i\assigned 0$\;
	$\E_{0}\assigned\DB^*$\;
	$\E_{1}\assigned \text{Exceptional}(\TB^*,\E_{0}$)\;
	\While{$\E_{i+1}\neq\E_{i}$}{
		$i\assigned i + 1$\;
		$\E_{i+1}\assigned \text{Exceptional}(\TB^*,\E_{i}$)\;
	}
	$\DB^*_{\infty}\assigned\E_{i}$\;
	$\TB^*\assigned\TB^*\cup\{C\subs D\mid C\dsubs D\in\DB^{*}_{\infty}\}$\;
	$\DB^*\assigned\DB^*\setminus\DB^*_{\infty}$\;
	}
	$\E\assigned (\E_0,\ldots,\E_{i-1})$\;
\textbf{return} ($\KB^*=\TB^*\cup\DB^*$, $\E$)\;
\end{function}

We initialise $\TB^{*}$ to $\TB$ and $\DB^{*}$ to $\DB$ (Lines~1 and~2 of $\compRank(\cdot)$). We then repeatedly invoke the function~$\excep(\cdot)$ to obtain a sequence of sets of DCIs $\E_{0},\E_{1},\ldots$, where $\E_{0}=\DB^{*}$ and each $\E_{i+1}$ is the set of exceptional axioms in $\E_{i}$ (Lines~4--14 of $\compRank(\cdot)$).

Now, let $\C_{\DB^{*}}\defined\{C \mid C\dsubs D\in\DB^{*}\}$, \ie, $\C_{\DB^{*}}$ is the set of all \emph{antecedents} of DCIs in~$\DB^{*}$. The exceptionality ranking of the DCIs in~$\DB^{*}$ computed by $\excep(\cdot)$ makes use of $\TB^{*}$, $\mat{\DB^{*}}$, and $\C_{\DB^{*}}$. That is, it checks, for each concept $C\in\C_{\DB^{*}}$, whether $\TB^{*}\entails\bigsqcap\mat{\DB^{*}}\subs\lnot C$. In case $C$ is exceptional, every DCI $C\dsubs D\in\DB^{*}$ is exceptional \wrt\ $\KB^{*}=\TB^{*}\cup\DB^{*}$ and is added to the set $\E_{1}$.

If $\E_{1}\neq\E_{0}$, then we call $\excep(\cdot)$ for $\TB^{*}\cup\E_{1}$, defining the set $\E_{2}$, and so on. Hence, given $\KB^{*}=\TB^{*}\cup\DB^{*}$, we construct a sequence $\E_{0},\E_{1},\ldots$ in the following way:
\begin{itemize}
\item $\E_{0}\assigned\DB^{*}$;
\item $\E_{i+1}\assigned\excep(\TB^{*},\E_{i})$, for $i\geq0$.
\end{itemize}

\begin{exContinued}
Using the knowledge base of Example~\ref{Example:Exceptionality}, we initialise $\TB^{*}=\{\EmpStud\subs\Stud\}$ and 
\[
\DB^{*}=
  \left\{
    \begin{array}{c}
      \Stud\dsubs\lnot\exists\pays.\Tax,\\[0.1cm] 
      \EmpStud\dsubs\exists\pays.\Tax,\\[0.1cm]
      \EmpStud\dland\Parent\dsubs\lnot\exists\pays.\Tax 
    \end{array}
  \right\}.
\]

We then obtain the following exceptionality sequence:
\[
\E_{0}=\left\{\begin{array}{c}
				\Stud\dsubs\lnot\exists\pays.\Tax,\\[0.1cm]
				\EmpStud\dsubs\exists\pays.\Tax,\\[0.1cm]
				\EmpStud\dland\Parent\dsubs\lnot\exists\pays.\Tax
			  \end{array}
		\right\}
\]\[
\E_{1}=\left\{\begin{array}{c}
				\EmpStud\dsubs\exists\pays.\Tax,\\[0.1cm]
				\EmpStud\dland\Parent\dsubs\lnot\exists\pays.\Tax
			  \end{array}
		\right\}
\]\[
\E_{2}=\{\EmpStud\dland\Parent\dsubs\lnot\exists\pays.\Tax\}
\]
\hspace*{\fill}\qed
\end{exContinued}

Since $\DB^{*}$ is finite, the construction will eventually terminate with a fixed point $\E_{\fix} = \excep(\TB^{*},\E_{\fix})$. If this fixed point is non-empty, then the axioms in there are said to have infinite rank. We therefore set $\DB^{*}_\infty\defined\E_{\fix}$ (Line~11 of $\compRank(\cdot)$), and the classical translations of these axioms are moved to the TBox. Hence we redefine the knowledge base in the following way (Lines~12 and 13 of $\compRank(\cdot)$):
\begin{itemize}
\item $\TB^{*}\assigned\TB^{*}\cup\{C\subs D\mid C\dsubs D\in\DB^{*}_\infty\}$;
\item $\DB^{*}\assigned\DB^{*}\setminus\DB^{*}_\infty$.
\end{itemize}

Function~$\compRank(\cdot)$ must terminate since~$\DB$ is finite, and at every iteration, $\DB^{*}$ becomes smaller (hence, we have at most $|\DB|$ iterations). In the end, we obtain a knowledge base $\KB^{*}=\TB^{*}\cup\DB^{*}$ which is modularly equivalent to the original knowledge base $\KB=\TB\cup\DB$ (see Lemma~\ref{Prop:PrefEquivalence} below), in which $\DB^{*}$ has no DCIs of infinite rank (all the classical knowledge implicit in the DTBox has been moved to the TBox). We say that such a knowledge base is in \emph{rank normal form}.


We also obtain a final exceptionality sequence $\E=(\E_{0},\E_{1},\ldots,\E_{i-1})$ (see Line~15 of $\compRank(\cdot)$). Given $\E$, it is possible to  partition the set $\DB^{*}$ into the sets $\DB_{0},\ldots,\DB_{n}$,  for some $n=i-1\geq 0$:
\begin{itemize}
    \item For every $j$, $0\leq j\leq n$, $\DB_{j}\assigned\E_{j}\setminus\E_{j+1}$;
    \item $R\assigned (\DB_{0},\ldots,\DB_{n})$.
\end{itemize}

The sequence $R$ is a partition of the DTBox according to the level of exceptionality of each defeasible inclusion in it.

\begin{exContinued}
For~$\KB$ as in Example~\ref{Example:Exceptionality}, we obtain the partition $R=\{\DB_{0},\DB_{1},\DB_{2}\}$, where $\DB_{0}=\{\Stud\dsubs\lnot\exists\pays.\Tax\}$, $\DB_{1}=\{\EmpStud\dsubs\exists\pays.\Tax\}$ and $\DB_{2}=\{\EmpStud\dland\Parent\dsubs\lnot\exists\pays.\Tax\}$. \hfill\ \qed
\end{exContinued}

At this stage, we have moved all the classical information implicit the DTBox to the TBox, and ranked all the remaining~DCIs, where the rank of a DCI is the index of the unique partition to which it belongs, defined as follows:

\begin{definition}[Ranking]\label{Def:Ranking}
For every $C,D\in\Lang$:
\begin{itemize}
\item $\rk(C)\defined i$, $0\leq i\leq n$, if $\E_i$ is the first element in the sequence $(\E_{0},\ldots,\E_{n})$ \st\ $\TB^{*}\nentails\bigsqcap\mat{\E_{i}}\dland C\subs\bot$;
\item $\rk(C)\defined\infty$, if there is no such $\E_{i}$;
\item $\rk(C\dsubs D)\defined\rk(C)$.
\end{itemize}
\end{definition}

\begin{remark}
For every $i\leq j\leq n$, $\entails\bigsqcap\mat{\E_{j}}\subs\bigsqcap\mat{\E_{i}}$.
\end{remark}

\begin{remark}
For every $i < j\leq n$, $\DB_{i}\cap\DB_{j}=\emptyset$.
\end{remark}

To summarise, we transform our initial knowledge base $\KB=\TB\cup\DB$, obtaining a modularly equivalent knowledge base $\KB^{*}=\TB^{*}\cup\DB^{*}$ (see Lemma~\ref{Prop:PrefEquivalence} below) and a ranking of DCIs in the form of a partitioning of~$\DB^{*}$. The main difference between $\compRank(\cdot)$ and the analogous procedure by Casini and Straccia~\cite{CasiniStraccia2010} is the reiteration of the ranking procedure until $\DB^*_{\infty}=\emptyset$ (lines 4-14 in $\compRank(\cdot)$). While the two procedures behave identically in the case where there are no DCIs $C\dsubs D$ \st\ $\rank_{\KB}(C\dsubs D)=\infty$ in $\DB$, the procedure by Casini and Straccia~\cite{CasiniStraccia2010} did not handle all the cases correctly in which there is classical information implicit in the DTBox. Lemma~\ref{Lemma:ClassicalInformation} in Appendix~\ref{ProofsComputingRC} and Lemma~\ref{Lemma:SameRankings} below prove that the procedure here is correct \wrt\ the semantics.

Given the knowledge base $\KB^{*}=\TB^{*}\cup\DB^{*}$, we can now define the main algorithm for deciding whether a DCI $C\dsubs D$ is in the rational closure of~$\KB$. To do that, we use the same approach as in the function $\excep(\cdot)$, that is, given $\KB^{*}=\TB^{*}\cup\DB^{*}$ and our sequence of sets $\E_{0},\ldots,\E_{n}$, we use the TBox~$\TB^{*}$ and the sets of conjunctions of materialisations $\bigsqcap\mat{\E_{0}},\ldots,\bigsqcap\mat{\E_{n}}$.

\begin{definition}[Rational Deduction]\label{Def:RationalDeduction}
Let $\KB=\TB\cup\DB$ and let $C,D\in\Lang$. We say that $C\dsubs D$ is \df{rationally deducible} from~$\KB$, denoted $\KB\proves_{\rat}C\dsubs D$, if $\TB^{*}\entails \bigsqcap\mat{\E_{i}}\dland C\subs D$, where $\bigsqcap\mat{\E_{i}}$ is the first element of the sequence $\bigsqcap\mat{\E_{0}},\ldots,\bigsqcap\mat{\E_{n}}$ \st\ $\TB^{*}\nentails\bigsqcap\mat{\E_{i}}\subs\lnot C$. If there is no such element, $\KB\proves_{\rat}C\dsubs D$ if $\TB^{*}\entails C\subs D$.
\end{definition}

Observe that $\KB\proves_{\rat}C\subs D$ if and only if $\KB\proves_{\rat}C\dland\lnot D\dsubs\bot$, \ie, if and only if $\KB\proves_{\rat}C\dland\lnot D\subs\bot$ (that is to say, $\TB^{*}\entails C\subs D$).
\myskip

The algorithm corresponding to the steps above is presented in the function $\compClosure(\cdot)$ below.

\begin{function}[ht]
\caption{RationalClosure($\KB$, $\alpha$)\label{Func:RationalClosure}}
\KwIn{$\KB=\TB\cup\DB$ and a query $\alpha=C\dsubs D$.}
\KwOut{$\mathtt{true}$ if $\KB\proves_{\rat}C\dsubs D$, $\mathtt{false}$ otherwise}
$(\KB^{*}=\TB^{*}\cup\DB^{*},\E=(\E_0,\ldots,\E_n))\assigned \compRank(\KB)$\;
$i\assigned 0$\;
\While {$\TB^{*}\entails\bigsqcap\mat{\E_{i}}\dland C\subs\bot$ and $i\leq n$}
{$i\assigned i + 1$\;
}
\If{$i\leq n$}
	{\textbf{return} $\TB^{*}\entails\bigsqcap\mat{\E_{i}}\dland C\subs D$\;}
\Else{\textbf{return} $\TB^{*}\entails C\subs D$\;}
\end{function}

\begin{exContinued}
Let~$\KB$ be as in Example~\ref{Example:Exceptionality} and assume we want to check whether $\EmpStud\dsubs\exists\pays.\Tax$ is in the rational closure of~$\KB$. Then, the while-loop on Line~$2$ of function~$\compClosure(\cdot)$ terminates when $i=1$. At this stage, $\bigsqcap\mat{\E_{i}}=(\lnot\EmpStud\dlor\exists\pays.\Tax)\dland(\lnot\EmpStud\dlor\lnot\Parent\dlor\lnot\exists\pays.\Tax)$. Given this, one can check that $\TB^{*}\nentails\bigsqcap\mat{\E_{i}}\dland C\subs\bot$, \ie, $\{\EmpStud\subs\Stud\}\nentails(\lnot\EmpStud\dlor\exists\pays.\Tax)\dland(\lnot\EmpStud\dlor\lnot\Parent\dlor\lnot\exists\pays.\Tax)\dland\EmpStud\subs\bot$. Finally, it is easy to confirm that $\TB^{*}\nentails\bigsqcap\mat{\E_{i}}\dland C\subs D$, \ie, $\{\EmpStud\subs\Stud\}\nentails(\lnot\EmpStud\dlor\exists\pays.\Tax)\dland(\lnot\EmpStud\dlor\lnot\Parent\dlor\lnot\exists\pays.\Tax)\dland\EmpStud\subs\exists\pays.\Tax$. \hfill\ \qed
\end{exContinued}


Before we state the main theorem of this section, we need to establish the correspondence between the ranking function $\rank_\KB(\cdot)$ presented in Section~\ref{DefinitionRC} in the construction of the rational closure of $\KB$ and linked by Theorem~\ref{Theorem:Characterisation} to the definition of rational entailment, and the ranking function $\rk(\cdot)$ of Definition~\ref{Def:Ranking} used in the above algorithm. We also need to establish that the normalisation of a knowledge base by our algorithm maintains modular equivalence. The proofs of the following lemmas, as well as a number of prerequisite results, are in Appendix~\ref{ProofsComputingRC}.

\begin{restatable}{lemma}{restatablePrefEquivalence}
\label{Prop:PrefEquivalence}
Let $\KB=\TB\cup\DB$ and let $\KB^{*}=\TB^{*}\cup\DB^{*}$ be obtained from~$\KB$ through function~$\compRank(\cdot)$. Then $\KB$ and~$\KB^{*}$ are modularly equivalent.
\end{restatable}

\begin{restatable}{lemma}{restatableSameRankings}
\label{Lemma:SameRankings}
For every defeasible knowledge base $\KB=\TB\cup\DB$ and every $C\in\Lang$, $\rank_{\KB}(C)=\rk(C)$.
\end{restatable}

Now we can state the main theorem, which links rational entailment to rational deduction via Theorem~\ref{Theorem:Characterisation}. (The proof can be found in Appendix~\ref{ProofsComputingRC}.)

\begin{restatable}{theorem}{restatableSoundnessCompletenessRC}
\label{Theorem:SoundnessCompletenessRC}
Let $\KB=\TB\cup\DB$ and let $C,D\in\Lang$. Then $\KB\proves_{\rat}C\dsubs D$ iff $\KB\entails_{\rat}C\dsubs D$.
\end{restatable}

As an immediate consequence, we have that the function~$\compClosure(\cdot)$ is correct \wrt\ the definition of rational closure in Definition~\ref{Def:RationalClosure}.

\begin{restatable}{corollary}{restatableComplexity}
Checking rational entailment is \ExpTime-complete.
\end{restatable}

Hence entailment checking for defeasible ontologies is just as hard as classical subsumption checking.

We conclude this section by noting that although rational closure is viewed as an appropriate form of defeasible reasoning, it does have its limitations, the first of which is that it does not satisfy the \emph{presumption of independence}~\cite[Section 3.1]{Lehmann1995}. To consider a well-worn example, suppose we know that birds usually fly and usually have wings, that both penguins and robins are birds, and that penguins usually do not fly. That is, we have the following knowledge base: $\KB = \{\Bird\dsubs\Flies, \Bird\dsubs\Wings, \Penguin\subs\Bird, \Robin\subs\Bird, \Penguin\dsubs\lnot\Flies\}$. Rational closure allows us to conclude that robins usually have wings, since they are viewed as typical birds, thereby satisfying the presumption of typicality. But with penguins being atypical birds, rational closure does not allow us to conclude that penguins usually have wings, thus violating the presumption of independence which, in this context, would require the atypicality of penguins \wrt~flying to be independent of the typicality of penguins \wrt~having wings.

This deficiency is well-known, and there are other forms of defeasible reasoning that can overcome this, most notably lexicographic closure \cite{CasiniStraccia2012}, relevance closure \cite{casini2014relevant}, and inheritance-based closure \cite{CasiniStraccia2011}. But note that the presumption of independence is  \emph{propositional} in nature. In fact, the DL version of lexicographic closure is essentially a lifting to the DL case of a propositional solution to the problem~\cite{Lehmann1995}.

What is perhaps of more interest is the inability of rational closure to deal with defeasibility relating to the \emph{non-propositional} aspects of descriptions logics. For example, Pensel and Turhan~\cite{PenselTurhan2017,PenselTurhan2018} have shown that rational closure across role expressions does not always support defeasible inheritance appropriately. Suppose we know that bosses are workers, do not have workers as their superiors, and are usually responsible. Furthermore, suppose we know that workers usually have bosses as their superiors. We thus have the knowledge base $\KB=\{\Boss\subs\Worker, \Boss\subs\lnot\exists \hasSuperior .\Worker$, $\Boss\dsubs\Responsible, \Worker\dsubs \exists\hasSuperior. \Boss\}$. Since workers usually have bosses as their superiors, and bosses are usually responsible, one would expect to be able to conclude that workers usually have responsible superiors. But rational closure is unable to do so. From the perspective of the algorithm for rational closure, this can be traced back to the use of materialisation (Definition~\ref{Def:Materialisation}) when computing exceptionality, as Pensel and Turhan show. A more detailed semantic explanation for this inability is still forthcoming, though.

\section{Related work}\label{RelatedWork}

In a sense, the first investigations on non-monotonic reasoning in DL-based systems date back to the work by Brewka~\cite{Brewka1987} and Cadoli~\etal.~\cite{CadoliEtAl1990}. Other early proposals to introduce default-style rules into description logics include the work by Baader and Hollunder~\cite{BaaderHollunder1993,BaaderHollunder1995}, Padgham and Zhang~\cite{PadghamZhang1993} and~Straccia~\cite{Straccia1993}, which are essentially based on Reiter's default logic~\cite{Reiter1980}.
\myskip

Quantz and Royer~\cite{QuantzRoyer1992} were probably the first to consider lifting the preferential approach to a DL setting.  They propose a general framework for Preferential Default Description Logics (PDDL) based on an \ALC-like language by introducing a version of default subsumption and  proposing a preferential semantics for it. Their semantics is based on a simplified version of standard~DL interpretations. They assume all domains to be finite, which means their framework is much more restrictive than ours in this aspect. They also allow for the use of object names (something we don't do), and assume that the unique-name assumption holds for object names. 

 They focus on a version of entailment which they refer to as preferential entailment, but which is to be distinguished from the version of preferential entailment that we have presented in this paper. In what follows, we shall refer to their version as \emph{QR-preferential entailment}.

QR-preferential entailment is concerned with what ought to follow from a set of  DL statements, together with a set of default subsumption statements, and is parameterised by a fixed partial order on (simplified) DL interpretations. (I.e., the ordering is on the set of~DL interpretations, not on the elements of their respective domains.) They prove that any QR-preferential entailment satisfies the properties of a preferential consequence relation and, with some restrictions on the partial order, satisfies rational monotonicity as well. QR-preferential entailment can therefore be viewed as something in between the notions of preferential entailment (or modular entailment) and rational entailment. It is also worth noting that although QR-preferential entailment satisfies the properties of a preferential consequence relation, Quantz and Royer do not prove that QR-preferential entailment provides a characterisation of preferential consequence in the spirit of the representation results we have shown here.
\myskip

Closely related to our work is that of Giordano~\etal.~\cite{GiordanoEtAl2009b,GiordanoEtAl2013} who use preferential orderings on $\Dom^{\I}$ to define a typicality operator $\Typ(\cdot)$ on \ALC\ concepts such that the expression $\Typ(C)\subs D$ corresponds to our $C\dsubs D$. They provide a version of a representation result for preferential orderings in terms of properties on selection functions (functions on the power set of the domain of interpretations), but not a representation result along the lines of those we have shown here. In the same work, the authors define a tableaux calculus for computing preferential entailment that relies on KLM-style rules.

Recently~\cite{GiordanoEtAl2015}, Giordano and colleagues extended the aforementioned work by considering modular orderings on $\Dom^{\I}$ (\ie, modular interpretations) and then augment the inferential power of their system with a version of a minimal-model semantics, in which some modular interpretations are preferred over others. This is similar in intuition to rational entailment, but their approach also has a circumscriptive flavour to it (see below) since it relies on the specification of a set of concepts for which atypical instances must be minimised.
\myskip

Outside the family of preferential systems, there are mature proposals based on circumscription for DLs~\cite{BonattiEtAl2011,BonattiEtAl2009,SenguptaEtAl2011}. The main drawback of these approaches is the burden on the ontology engineer to make appropriate decisions related to the (circumscriptive) fixing and varying of concepts and the priority of defeasible subsumption statements. Such choices can have a major effect on the conclusions drawn from the system, and can easily lead to counter-intuitive inferences. Moreover, the use of circumscription usually implies a considerable increase in  computational complexity \wrt~the underlying monotonic entailment relation. The comparison between the present work and proposals outside the preferential family is more an issue about the pros and cons of the different kinds of non-monotonic reasoning, rather than about their DL re-formulation. As stated in the introduction,  the preferential approach has a series of desirable qualities that, to our knowledge, no other approach to non-monotonic reasoning shares.

A more recent proposal is the approach by Bonatti~\etal.~\cite{BonattiEtAl2015,BonattiSauro2017}, which introduces a \emph{normality} operator $\textbf{N}(\cdot)$ on concepts. The resulting system, DL$^{N}$, is not based on the preferential approach, though, and as a consequence their closure operation does not allow  defeasible subsumption to satisfy the preferential properties, but it satisfies some interesting properties on the meta-level. It also has the advantage of being computationally tractable for any tractable classical~DL.
\myskip

Lukasiewicz~\cite{Lukasiewicz2008} proposes probabilistic versions of the description logics $\mathcal{SHIF}(\textbf{D})$ and $\mathcal{SHOIN}(\textbf{D})$. As a special case of these logics, he obtains a version of a logic with defeasible subsumption with a semantics based on that of the propositional version of lexicographic closure~\cite{Lehmann1995}.
\myskip

Casini and Straccia~\cite{CasiniStraccia2010} define KLM-based decision procedures for~\ALC. Their proposal has a syntactic characterisation, but lacks an appropriate semantics, a deficiency that the present paper comes to remedy. The semantic approach presented here can be extended also to other forms of entailment proposed by them~\cite{CasiniStraccia2011,CasiniStraccia2012,CasiniStraccia2013}, and recently Casini, Straccia and Meyer have used it also to characterise a decision procedure for defeasible $\mathcal{EL}_\bot$ \cite{CasiniEtAl2018}.
\myskip

Britz and Varzinczak~\cite{BritzVarzinczak2013,BritzVarzinczak2018-JANCL,BritzVarzinczak2018-JoLLI} explore the notion of \emph{defeasible modalities}, with which defeasible effects of actions, defeasible knowledge, obligations and others can be formalised and given an intuitive preferential semantics. Their approach differs from ours in that they consider preferential entailment only, but the semantic constructions are similar. This was recently extended~\cite{BritzEtAl2013b} to a notion of \emph{defeasible role restrictions} in a~DL setting. The idea comprises extending the language of~\ALC\ with an additional construct~${\dforall}$. The semantics of a concept~${\dforall}r.C \defined\{x\in\Dom^{\PI}\mid \min_{\pref^{\PI}}r^{\PI}(x)\subseteq C^{\PI}\}$ is then given by all objects of $\Dom^{\PI}$ such that all of their \emph{minimal} $r$-related objects are $C$-instances. This is useful in situations where certain classical concept descriptions may be too strong.

Recently, Britz and Varzinczak have lifted the preferential semantics to also allow for orderings on role-interpretations~\cite{BritzVarzinczak2016b,BritzVarzinczak2017-DL} that, in turn, induce multi-orderings on objects of the domain~\cite{BritzVarzinczak2017-Commonsense,BritzVarzinczak2018-FoIKS}. The latter give us the handle needed to introduce a notion of \emph{context} in defeasible subsumption relations making typicality a relativised construct. The former provides a semantics for defeasible role inclusions of the form $r\dsubs s$ and for defeasible role assertions such as ``$r$ is usually transitive'', ``$r$ and~$s$ are usually disjoint'', as well as others.

Finally, there is the recent work of Pensel and Turhan~\cite{PenselTurhan2017,PenselTurhan2018} mentioned in Section~\ref{ComputingRC}, the aim of which is to extend both rational closure and relevance closure with defeasible inheritance across role expressions in the description logic~$\mathcal{EL}_{\bot}$. With their work being restricted to~$\mathcal{EL}_{\bot}$, the semantics they propose is based on a form of canonical model similar to those frequently used for the~$\mathcal{EL}$ family of~DLs, and is therefore quite different from ours. A detailed comparison of their semantics with the one we provide in this paper is left as future work.

\section{Concluding remarks}\label{Conclusion}

The main contributions of the work reported in the present paper can be summarised as follows:
\begin{enumerate}
\item The analysis of a simple and intuitive semantics for defeasible subsumption in description logics that is general enough to constitute the core framework within which to investigate non-monotonic extensions of~DLs;
\item A characterisation of preferential and rational subsumption relations, with the respective representation results, evidencing the fact that our semantic constructions are appropriate;
\item An investigation of what an appropriate notion of entailment in a defeasible~DL context means and the analysis of a suitable candidate, namely rational entailment, and
\item The formal connection between rational entailment, the notion of rational closure and an algorithm for its computation.
\end{enumerate}

With regard to point~4 above, the main advantages of our approach are as follows: (\emph{i})~it relies completely on classical entailment, \ie, entailment checking over defeasible ontologies can be reduced to a number of classical entailment checks over a rewritten ontology, (\emph{ii})~it has computational complexity that is no worse than that of entailment checking in the classical underlying~DL, and (\emph{iii})~it is easily implementable, \eg\ as a Protégé plugin\footnote{\url{https://github.com/kodymoodley/defeasibleinferenceplatform}}, of which the performance has been shown to scale well in practice~\cite{CasiniEtAl2015}. In a companion paper~\cite{BritzEtAl2015a} the framework described here is extended to include ABox reasoning, with more extensive experimental results confirming the initial promising results on scalability.

Further topics for future research include the integration of notions such as \emph{typicality} for both concepts and roles~\cite{BoothEtAl2015,BoothEtAl2012,BoothEtAl2013,GiordanoEtAl2009b,GiordanoEtAl2013,Varzinczak2018} and role-based defeasible constructors~\cite{BritzVarzinczak2016b,BritzVarzinczak2017-DL,BritzVarzinczak2018-FoIKS,BritzVarzinczak2019-AMAI} into the framework here presented. Another avenue for future exploration is the study of belief revision for~DLs via our results for rationality, somewhat mimicking the well-known connection between belief revision and rational consequence in the propositional case~\cite{GardenforsMakinson1994}, thereby pushing the frontiers of theory change in logics that are more expressive than the propositional one.


\bibliographystyle{plain}
\bibliography{References}

\appendix

\section{Proofs of lemmas in Section~\ref{Semantics}}\label{Sec:Sec3.2}

\newcommand{\SetPI}{\ensuremath{\mathscr{P}}}
\newcommand{\SetRI}{\ensuremath{\mathscr{R}}}
\newcommand{\DU}{\ensuremath{\mathcal{U}}}
\newcommand{\Bis}{{\sf E}}

\textbf{NB}: The results marked $(*)$ are introduced here in the Appendix, while they are omitted in the main text.

\restatableClassicalStatements*
\begin{proof}
From left to right, $\PI\sat C\subs D$ implies, by Lemma~\ref{Lemma:TruthInInterpretation}, that $(C\dland\lnot D)^\PI=\emptyset$. The latter implies that, for every concept $E$, $\PI\sat C\dland\lnot D\dsubs E$, and, as a particular case, $\PI\sat C\dland\lnot D\dsubs\bot$. From right to left, if $\PI\nsat C\subs D$, then $(C\dland\lnot D)^\PI\neq\emptyset$. Let $x$ be an object in $\min_{\prec^{\PI}}(C\dland\lnot D)^\PI$: for $\PI\sat C\dland\lnot D\dsubs\bot$, we should have also $x\in\bot^\PI$, which is a contradiction.{\hfill} \qed
\end{proof}

\begin{definition}[Disjoint-Union Preferential Interpretation]\label{Def:DUPI}
Let $S$ be a countable set and let~$\SetPI=\{\PI_{s}=\tuple{\Dom^{\PI_{s}},\cdot^{\PI_{s}},\pref^{\PI_{s}}} \mid s\in S\}$ be a collection of preferential interpretations. The \df{disjoint union of}~$\SetPI$ is a tuple~$\DU\defined\tuple{\Dom^{\DU},\cdot^{\DU},\pref^{\DU}}$ where:
\begin{itemize}\setlength{\itemsep}{1.3ex}
\item $\Dom^{\DU}\defined\{(x,s) \mid x\in\Dom^{\PI_{s}}$ and $s\in S\}$;
\item $A^{\DU}\defined\{(x,s) \mid x\in A^{\PI_{s}}$ and $s\in S\}$, for every $A\in\CN$;
\item $r^{\DU}\defined\{((x,s),(y,s)) \mid (x,y)\in r^{\PI_{s}}$ and $s\in S\}$, for every $r\in\RN$;
\item $\pref^{\DU}\defined\{((x,s),(y,s)) \mid (x,y)\in\pref^{\PI_{s}}$ and $s\in S\}$.
\end{itemize}
\end{definition}

\begin{lemma}[$*$]\label{Lemma:ConceptEquivalence}
Let $S$ and~$\SetPI$ be as in Definition~\ref{Def:DUPI} and let~$\DU$ be the latter's disjoint union. For every~$C\in\Lang$, every~$s\in S$, and every $x\in\Dom^{\PI_{s}}$, $x\in C^{\PI_{s}}$ if and only if $(x,s)\in C^{\DU}$.
\end{lemma}
\begin{proof}
For every $s\in S$, define $\Bis_{s}\defined\{(x,(x,s)) \mid x\in\Dom^{\PI_{s}}\}$. We can easily show that~$\Bis_{s}$ is a preferential bisimulation~\cite{BritzVarzinczak2018-JoLLI} between~$\PI_{s}$ and~$\DU$. The lemma is then proved by induction on the structure of concepts in the usual way~\cite{BaaderEtAl2017}. \hfill\ \qed
\end{proof}

It is easy to see that the following result also holds:

\begin{lemma}[$*$]\label{Lemma:MinConceptEquivalence}
Let $S$ and~$\SetPI$ be as in Definition~\ref{Def:DUPI} and let~$\DU$ be the latter's disjoint union. For every~$C\in\Lang$, every~$s\in S$, and every $x\in\Dom^{\PI_{s}}$, $x\in\min_{\pref^{\PI_{s}}}C^{\PI_{s}}$ if and only if $(x,s)\in\min_{\pref^{\DU}}C^{\DU}$.
\end{lemma}

\restatableClosureDisjointUnionPref*


\begin{proof}
Let~$\KB$ be a defeasible knowledge base, let $S$ and $\SetPI$ be as in Definition~\ref{Def:DUPI} and such that $\PI_{s}\sat\KB$, for every $\PI_{s}\in\SetPI$, and let~$\DU$ be the disjoint union of the models in~$\SetPI$. We have to show that~$\DU\sat\KB$. Assume that $\DU\nsat\KB$. Then there must be a DCI $C\dsubs D\in\KB$ (recall Lemma~\ref{Lemma:ClassicalStatements}) and an object $(x,s)\in\Dom^{\DU}$ such that $(x,s)\in\min_{\pref^{\DU}}C^{\DU}$ but $(x,s)\notin D^{\DU}$. From Lemmas~\ref{Lemma:ConceptEquivalence} and~\ref{Lemma:MinConceptEquivalence} above, it follows that $x\in\min_{\pref^{\PI_{s}}}C^{\PI_{s}}$ and $x\notin D^{\PI_{s}}$, and therefore $\PI_{s}\nsat C\dsubs D$. Hence, $\PI_{s}\nsat\KB$, which contradicts our assumption. \hfill\ \qed
 \end{proof}

\section{Proof of Theorem~\ref{Theorem:RepResultPreferential}}\label{ProofRepResultPreferential}

\restatableRepResultPreferential*

\subsection{If part}\label{Proof:PreferentialSoundness}

We show that $\dsubs_{\PI}$ is preferential for every preferential interpretation~$\PI=\tuple{\Dom^{\PI},\cdot^{\PI},\pref^{\PI}}$.
\myskip


\noindent(Ref): Let $x\in\Dom^{\PI}$ be such that $x\in\min_{\pref^{\PI}}C^{\PI}$. Then clearly $x\in C^{\PI}$ and therefore $\PI\sat C\dsubs C$. Hence $C\dsubs_{\PI} C$.
\myskip

\noindent(LLE): Assume that $C\dsubs_{\PI}E$ and $\PI\sat C\equiv D$. Then $\PI\sat C\dsubs E$, which means $\min_{\pref^{\PI}}C^{\PI}\subseteq E^{\PI}$. Since $\PI\sat C\equiv D$, \ie, $C^{\PI}=D^{\PI}$, we have $\min_{\pref^{\PI}}C^{\PI}=\min_{\pref^{\PI}}D^{\PI}$. Hence $\min_{\pref^{\PI}}D^{\PI}\subseteq E^{\PI}$, and therefore $\PI\sat D\dsubs E$, from which follows $D\dsubs_{\PI}E$.
\myskip

\noindent(And): Assume we have both $C\dsubs_{\PI}D$ and $C\dsubs_{\PI}E$. Then $\PI\sat C\dsubs D$ and $\PI\sat C\dsubs E$, \ie, $\min_{\pref^{\PI}}C^{\PI}\subseteq D^{\PI}$ and $\min_{\pref^{\PI}}C^{\PI}\subseteq E^{\PI}$, and then $\min_{\pref^{\PI}}C^{\PI}\subseteq D^{\PI}\cap E^{\PI}$, from which follows $\min_{\pref^{\PI}}C^{\PI}\subseteq(D\dland E)^{\PI}$. Hence $\PI\sat C\dsubs D\dland E$, and therefore $C\dsubs_{\PI}D\dland E$.
\myskip

\noindent(Or): Assume we have both $C\dsubs_{\PI}E$ and $D\dsubs_{\PI}E$. Let $x\in\min_{\pref^{\PI}}(C\dlor D)^{\PI}$. Then $x$ is minimal in $C^{\PI}\cup D^{\PI}$, and therefore either $x\in\min_{\pref^{\PI}}C^{\PI}$ or $x\in\min_{\pref^{\PI}}D^{\PI}$. In either case $x\in E^{\PI}$. Hence $\PI\sat C\dlor D\dsubs E$ and therefore $C\dlor D\dsubs_{\PI}E$.
\myskip

\noindent(RW): Assume we have both $C\dsubs_{\PI}D$ and $\PI\sat D\subs E$. Then $\PI\sat C\dsubs D$, \ie, $\min_{\pref^{\PI}}C^{\PI}\subseteq D^{\PI}$, and $D^{\PI}\subseteq E^{\PI}$. Hence $\min_{\pref^{\PI}}C^{\PI}\subseteq E^{\PI}$ and then $\PI\sat C\dsubs E$.  Therefore $C\dsubs_{\PI} E$.
\myskip

\noindent(CM): Assume we have both $C\dsubs_{\PI}D$ and $C\dsubs_{\PI}E$. Then $\PI\sat C\dsubs D$ and $\PI\sat C\dsubs E$, and therefore $\min_{\pref^{\PI}}C^{\PI}\subseteq D^{\PI}$ and $\min_{\pref^{\PI}}C^{\PI}\subseteq E^{\PI}$. Let $x\in\min_{\pref^{\PI}}(C\dland D)^{\PI}$. We show that $x\in\min_{\pref^{\PI}}C^{\PI}$. Suppose this is not the case. Since $\pref^{\PI}$ is smooth, there must be $x'\in\min_{\pref^{\PI}}C^{\PI}$ such that $x'\pref^{\PI}x$. Because $\PI\sat C\dsubs D$, $x'\in D^{\PI}$, and then $x'\in C^{\PI}\cap D^{\PI}$, \ie, $x'\in(C\dland D)^{\PI}$. From this and $x'\pref^{\PI}x$ it follows that $x$ is not minimal in $(C\dland D)^{\PI}$, which is a contradiction. Hence $x\in\min_{\pref^{\PI}}C^{\PI}$. From this and $\min_{\pref^{\PI}}C^{\PI}\subseteq E^{\PI}$, it follows that $x\in E^{\PI}$. Hence $\PI\sat C\dland D\dsubs E$, and therefore $C\dland D\dsubs_{\PI}E$. \hfill\ \qed

\subsection{Only-if part}\label{Proof:PreferentialCompleteness}

\newcommand{\U}{\ensuremath{\mathscr{U}}}

\textbf{NB}: The results marked $(*)$ are introduced here in the Appendix, while they are omitted in the main text.
\myskip

Let $\dsubs\subseteq\Lang\times\Lang$ be a preferential subsumption relation. We shall construct a preferential interpretation~$\PI$ such that $\dsubs_{\PI}\defined\{(C,D) \mid \PI\sat C\dsubs D\}=\dsubs$.

\begin{definition}\label{Def:Universe}
Let $\U\defined\{(\I,x) \mid \I=\tuple{\Dom^{\I},\cdot^{\I}}$ and $x\in\Dom^{\I}\}$.
\end{definition}

Intuitively, $\U$ denotes the \emph{universe} of objects in the context of their respective~DL interpretations, \ie, $\U$ is a set of first-order interpretations.

\begin{definition}\label{Def:Normality}
A pair $(\I,x)\in\U$ is \df{normal} for $C\in\Lang$ if for every $D\in\Lang$ such that $C\dsubs D$, we have $x\in D^{\I}$.
\end{definition}

\begin{lemma}[$*$]\label{Lemma:NormalityTwiddle}
Let $\dsubs\subseteq\Lang\times\Lang$ satisfy (Ref), (RW) and~(And), and let $C,D\in\Lang$. Then all normal~$(\I,x)$ for~$C$ satisfy~$D$ if and only if~$C\dsubs D$.
\end{lemma}
\begin{proof}
The if part follows from the definition of normality above. For the only-if part, assume~$C{\ndsubs}D$. We build a pair~$(\I,x)$ that is normal for~$C$ but that does not satisfy~$D$. Let~$\Gamma\defined\{\lnot D\}\cup\{E \mid C\dsubs E\}$. All we need to do is show that there is~$(\I,x)$ such that $x\in F^{\I}$ for every~$F\in\Gamma$. Suppose this is not the case. Then by compactness there exists a finite~$\Gamma'\subseteq\Gamma$ such that $\entails\bigsqcap_{F\in\Gamma'}F\subs D$. From this follows $\entails\top\subs\lnot\bigsqcap_{F\in\Gamma'}F\dlor D$, and, in particular, we have $\entails C\subs\lnot\bigsqcap_{F\in\Gamma'}F\dlor D$. Now from~(Ref) we have~$C\dsubs C$. From this, $\entails C\subs\lnot\bigsqcap_{F\in\Gamma'}F\dlor D$ and~(RW) we get $C\dsubs(\lnot\bigsqcap_{F\in\Gamma'}F\dlor D)$. But we also have $C\dsubs\bigsqcap_{F\in\Gamma'}F$ by the~(And) rule, and then by applying~(And) once more we derive $C\dsubs\bigsqcap_{F\in\Gamma'}F\dland(\lnot\bigsqcap_{F\in\Gamma'}F\dlor D)$. From this and~(RW) we conclude~$C\dsubs D$, from which we derive a contradiction. \hfill\ \qed
\end{proof}

\begin{lemma}[$*$]\label{Lemma:Property21-KLM90}
If $\dsubs$ is preferential, the following rule holds:
\[
\frac{C\dlor D\dsubs C,\ D\dlor E\dsubs D}{C\dlor E\dsubs C}
\]
\end{lemma}
\begin{proof}
The proof is analogous to that by Kraus~\etal.~\cite[Lemma~5.5]{KrausEtAl1990}. \hfill\ \qed
\end{proof}

\begin{definition}\label{Def:NotLessOrdinary}
Let $C,D\in\Lang$. $C\leq D$ if $C\dlor D\dsubs C$.
\end{definition}

\begin{lemma}[$*$]\label{Lemma:NotLessOrdinaryIsReflexiveTransitive}
If $\dsubs$ is preferential, then $\leq$ is reflexive and transitive.
\end{lemma}
\begin{proof}
From (Ref) we have $C\dsubs C$. This and~(LLE) gives us~$C\dlor C\dsubs C$, therefore we have $C\leq C$ and~$\leq$ is reflexive. Transitivity follows from Lemma~\ref{Lemma:Property21-KLM90}. \hfill\ \qed
\end{proof}

\begin{lemma}[$*$]\label{Lemma:Property20-KLM90}
If $\dsubs$ is preferential, the following rule holds:
\[
\frac{C\dlor D\dsubs C,\ D\dsubs E}{C\dsubs\lnot D\dlor E}
\]
\end{lemma}
\begin{proof}
The proof is analogous to that by Kraus~\etal.~\cite[Lemma~5.5]{KrausEtAl1990}. \hfill\ \qed
\end{proof}

\begin{lemma}[$*$]\label{Lemma:NormalityNotLessOrdinary}
If $C\leq D$ and $(\I,x)$ is normal for~$C$, and $x\in D^{\I}$, then $(\I,x)$ is normal for~$D$.
\end{lemma}
\begin{proof}
From $C\leq D$ we get $C\dlor D\dsubs C$. Assume that $D\dsubs E$ is the case. Then by Lemma~\ref{Lemma:Property20-KLM90} we have~$C\dsubs\lnot D\dlor E$. Since~$(\I,x)$ is normal for~$C$, we have $x\in(\lnot D\dlor E)^{\I}$. Given that $x\in D^{\I}$, we must have $x\in E^{\I}$. \hfill\ \qed
\end{proof}

\begin{lemma}[$*$]\label{Lemma:Property22-KLM90}
If $\dsubs$ is preferential, the following rule holds:
\[
\frac{C\dlor D\dsubs C,\ D\dlor E\dsubs D}{C\dsubs\lnot E\dlor D}
\]
\end{lemma}
\begin{proof}
The proof is analogous to that by Kraus~\etal.~\cite[Lemma~5.5]{KrausEtAl1990}. \hfill\ \qed
\end{proof}

\begin{lemma}[$*$]\label{Lemma:NormalityInterpolated}
If $C\leq D\leq E$ and $(\I,x)$ is normal for~$C$, and $x\in E^{\I}$, then $(\I,x)$ is normal for~$D$.
\end{lemma}
\begin{proof}
By Lemma~\ref{Lemma:NormalityNotLessOrdinary}, it is enough to show that~$x\in D^{\I}$. By Lemma~\ref{Lemma:Property22-KLM90} we have $C\dsubs\lnot E\dlor D$. Since~$(\I,x)$ is normal for~$C$ and~$x\in E^{\I}$, then we must have~$x\in D^{\I}$. \hfill\ \qed
\end{proof}
\myskip

We now construct a preferential interpretation as in Definition~\ref{Def:PrefInterpretation}.
\myskip

\newcommand{\Is}{\mathscr{I}}
\newcommand{\J}{\ensuremath{\mathcal{J}}}

Let $\C_{\bot}\defined\{C \mid C\dsubs\bot\}$ and let~$\Is\defined\{\I=\tuple{\Dom^{\I},\cdot^{\I}} \mid C^{\I}=\emptyset$ for all $C\in\C_{\bot}\}$. Intuitively, $\Is$ contains all interpretations that are `compatible' with~$\dsubs$ in the sense of not satisfying concepts that are defeasibly subsumed by the contradiction.

For each~$\I\in\Is$, let~$\I^{+}\defined\tuple{\Dom^{\I^{+}},\cdot^{\I^{+}}}$ be such that:
\begin{itemize}
\item $\Dom^{\I^{+}}\defined X^{C}\cup X^{\bot}$, where $X^{C}\defined\{\tuple{\I,x,C} \mid (\I,x)$ is normal for $C\in\Lang\}$, and $X^{\bot}\defined\{\tuple{\I,x,\bot} \mid (\I,x)$ is not normal for any $C\in\Lang\}$;
\item $\cdot^{\I^{+}}$ is such that for every~$D\in\Lang$, $\tuple{\I,x,C}\in D^{\I^{+}}$ if and only if $x\in D^{\I}$, and for every $r\in\RN$, $(\tuple{\I,x,C},\tuple{\I,y,D})\in r^{\I^{+}}$ if and only if $(x,y)\in r^{\I}$.
\end{itemize}
\myskip

Let $\PI\defined\tuple{\Dom^{\PI},\cdot^{\PI},\pref^{\PI}}$ be such that:
\begin{itemize}
\item $\Dom^{\PI}\defined\bigcup_{\I\in\Is}\Dom^{\I^{+}}$;
\item $\cdot^{\PI}\defined\bigcup_{\I\in\Is}\cdot^{\I^{+}}$;
\item $\pref^{\PI}$ is the smallest relation such that:
\begin{itemize}
\item For every $\tuple{\I,x,C}\in\Dom^{\PI}$ such that $C\neq\bot$, $\tuple{\I,x,C}\pref^{\PI}\tuple{\J,y,\bot}$ for every $\tuple{\J,y,\bot}\in\Dom^{\PI}$;
\item For every $\tuple{\I,x,C},\tuple{\J,y,D}\in\Dom^{\PI}$ such that $C,D\neq\bot$, $\tuple{\I,x,C}\pref^{\PI}\tuple{\J,y,D}$ if and only if $C\leq D$ and $x\notin D^{\I}$.
\end{itemize}
\end{itemize}
(In the construction of~$\PI$, note that all pairs~$(\I,x)$ that are not normal for any concept~$C$ are moved higher up in the ordering so that they correspond to the least preferred objects of the domain.)

In Lemmas~\ref{Lemma:NonEmptyDomain} to \ref{Lemma:Smoothness} below we show that~$\PI$ as constructed above is indeed a preferential interpretation.

\begin{lemma}[$*$]\label{Lemma:NonEmptyDomain}
$\Dom^{\PI}\neq\emptyset$.
\end{lemma}
\begin{proof}
From $\top\ndsubs\bot$ and Lemma~\ref{Lemma:NormalityTwiddle}, it follows that there is some normal $(\I,x)$ for~$\top$ that does not satisfy~$\bot$. Hence $\tuple{\I,x,\top}\in\Dom^{\PI}$ and therefore~$\Dom^{\PI}\neq\emptyset$. \hfill\ \qed
\end{proof}

\begin{lemma}[$*$]\label{Lemma:LessThanBot}
$C\leq\bot$ for every~$C\in\Lang$.
\end{lemma}
\begin{proof}
By~(Ref) we have $C\dsubs C$. Since $\entails C\equiv C\dlor\bot$, by~(LLE) we get $C\dlor\bot\dsubs C$, and then from the definition of~$\leq$ follows~$C\leq\bot$. \hfill\ \qed
\end{proof}

\begin{lemma}[$*$]\label{Lemma:PrefIsPartialOrder}
$\pref^{\PI}$ is a strict partial order on~$\Dom^{\PI}$, \ie, $\pref^{\PI}$ is irreflexive and transitive.
\end{lemma}
\begin{proof}
First we show irreflexivity. From the construction of~$\pref^{\PI}$, it clearly follows that for every $\tuple{\I,x,\bot}\in\Dom^{\PI}$, $\tuple{\I,x,\bot}\npref\tuple{\I,x,\bot}$. Assume that~$\tuple{\I,x,C}\pref^{\PI}\tuple{\I,x,C}$ for some~$C\neq\bot$. Then $C\leq C$ and $x\notin C^{\I}$, \ie, $C\dlor C\dsubs C$, and then $C\dsubs C$, by~(LLE). This and $x\notin C^{\I}$ contradicts the fact that $(\I,x)$ is normal for~$C$. Hence $\tuple{\I,x,C}\npref^{\PI}\tuple{\I,x,C}$ for every~$\tuple{\I,x,C}\in\Dom^{\I}$.
\myskip

\noindent We now show transitivity. Suppose $\tuple{\I,x,C}\pref^{\PI}\tuple{\I',x',D}$ and $\tuple{\I',x',D}\pref^{\PI}\tuple{\I'',x'',E}$. From the definition of~$\pref^{\PI}$ we know that~$C,D\neq\bot$, since all non-normal objects are at the highest level in the ordering and are all incomparable. We then have~$C\leq D$ and $D\leq E$. (If $E=\bot$, we also have~$D\leq E$ by Lemma~\ref{Lemma:LessThanBot}.) From transitivity of~$\leq$ (Lemma~\ref{Lemma:NotLessOrdinaryIsReflexiveTransitive}), we conclude~$C\leq E$. Since~$\tuple{\I,x,C}\in\Dom^{\PI}$ and $\tuple{\I,x,C}\pref^{\PI}\tuple{\I',x',D}$, we conclude that~$(\I,x)$ is normal for~$C$ and $x\notin D^{\PI}$. This and Lemma~\ref{Lemma:NormalityInterpolated} imply that~$x\notin E^{\PI}$. \hfill\ \qed
\end{proof}

\begin{lemma}[$*$]\label{Lemma:MinimalObject}
Given $\tuple{\I,x,D}\in\Dom^{\PI}$, $\tuple{\I,x,D}\in\min_{\pref^{\PI}}C^{\PI}$ iff $x\in C^{\I}$ and~$D\leq C$.
\end{lemma}
\begin{proof}
For the if part, suppose that~$x\in C^{\I}$ and~$D\leq C$. Then it clearly follows that $\tuple{\I,x,D}\in C^{\PI}$ (Lemma~\ref{Lemma:NormalityNotLessOrdinary}). Now suppose that $\tuple{\I,x,D}$ is not $\pref^{\PI}$-minimal in~$C^{\PI}$, \ie, there is $\tuple{\I',x',E}$ for some~$\I'$ such that $x'\in\Dom^{\I'}$ and some $E\in\Lang$ such that $\tuple{\I',x',E}\pref^{\PI}\tuple{\I,x,D}$ and $x'\in C^{\I'}$. From this and the definition of~$\pref^{\PI}$, it follows that~$E\leq D$ and $x'\notin D^{\I'}$. Hence $E\leq D\leq C$ and $(\I',x')$ is normal for~$E$, and since $x'\in C^{\I'}$, by Lemma~\ref{Lemma:NormalityInterpolated} we get that $(\I',x')$ is normal for~$D$, from which we conclude $x'\in D^{\I'}$, a contradiction.
\myskip

\noindent For the only-if part, suppose that $\tuple{\I,x,D}$ is $\pref^{\PI}$-minimal in~$C^{\PI}$. Then clearly~$x\in C^{\I}$. Now assume there is some $(\I',x')$ which is normal for $C\dlor D$ and $x'\notin D^{\I'}$. Since $C\dlor D\leq D$, we must have $\tuple{\I',x',C\dlor D}\pref^{\PI}\tuple{\I,x,D}$. Since~$(\I',x')$ is normal for $C\dlor D$ and $x'\notin D^{\I'}$, it follows that $x'\in C^{\I'}$. This contradicts the minimality of $\tuple{\I,x,D}$ in~$C^{\PI}$. Hence all normal $(\I',x')$ for $C\dlor D$ satisfy~$D$. From this and Lemma~\ref{Lemma:NormalityTwiddle} follows $C\dlor D\dsubs D$, \ie, $D\leq C$. \hfill\ \qed
\end{proof}

\begin{lemma}[$*$]
There is no $C\in\Lang$ such that~$C^{\PI}\neq\emptyset$ and $\bot\leq C$.
\end{lemma}
\begin{proof}
Let $C\in\Lang$ be such that~$C^{\PI}\neq\emptyset$. Assume that $\bot\leq C$. Then $\bot\dlor C\dsubs\bot$, \ie, $C\dsubs\bot$. Then~$C\in\C_{\bot}$, and then $C^{\PI}=\emptyset$ by the construction of~$\PI$. \hfill\ \qed
\end{proof}

\begin{corollary}[$*$]
It follows from the two last lemmas that there is no $C\in\Lang$ for which any $\tuple{\I,x,\bot}\in\Dom^{\PI}$ is minimal.
\end{corollary}

\begin{lemma}[$*$]\label{Lemma:Smoothness}
For any $C\in\Lang$, $C^{\PI}$ is smooth.
\end{lemma}
\begin{proof}
Suppose that $\tuple{\I,x,D}\in C^{\PI}$, \ie, $x\in C^{\I}$. If $D\leq C$, then by Lemma~\ref{Lemma:MinimalObject} $\tuple{\I,x,D}$ is $\pref^{\PI}$-minimal in~$C^{\PI}$. On the other hand, \ie, if $D\not\leq C$, $C\dlor D\ndsubs D$, then by Lemma~\ref{Lemma:NormalityTwiddle} there is a normal~$(\I',x')$ for $C\dlor D$ such that $x\notin D^{\I'}$. But $C\dlor D\dsubs C\dlor D$, and then $(C\dlor D)\dlor D\dsubs C\dlor D$, and then $C\dlor D\leq D$. Hence $\tuple{\I',x',C\dlor D}\pref^{\PI}\tuple{\I,x,D}$. But $x'\in(C\dlor D)^{\I'}$ and $x'\notin D^{\I'}$, therefore $x'\in C^{\I'}$. Since $C\dlor D\leq C$, from Lemma~\ref{Lemma:MinimalObject} we conclude that $\tuple{\I',x',C\dlor D}$ is $\pref^{\PI}$-minimal in~$C^{\PI}$. \hfill\ \qed
\end{proof}
\myskip

Next we show in Lemma~\ref{Lemma:Completeness} that the abstract relation~$\dsubs$ we started off with coincides with the relation~$\dsubs_{\PI}$ obtained from our constructed preferential interpretation~$\PI$.

\begin{lemma}[$*$]\label{Lemma:Completeness}
$C\dsubs D$ if and only if~$C\dsubs_{\PI}D$.
\end{lemma}
\begin{proof}
For the only-if part, we show that $\min_{\pref^{\PI}}C^{\PI}\subseteq D^{\PI}$. Let $\tuple{\I,x,E}$ be $\pref^{\PI}$-minimal in~$C^{\PI}$. Then $(\I,x)$ is normal for~$E$ and~$x\in C^{\PI}$, and from Lemma~\ref{Lemma:MinimalObject} we also have~$E\leq C$. From these results and Lemma~\ref{Lemma:NormalityNotLessOrdinary} it follows that~$(\I,x)$ is normal for~$C$. Since~$C\dsubs D$, we have $x\in D^{\I}$, and therefore $\tuple{\I,x,E}\in\Dom^{\PI}$.
\smallskip

\noindent For the if part, let $C\dsubs_{\PI}D$. From the definition of~$\pref^{\PI}$, it follows that for every~$(\I,x)$ normal for~$C$, $\tuple{\I,x,C}\in\min_{\pref^{\PI}}C^{\PI}$. Since $C\dsubs_{\PI}D$, then $y\in D^{\I'}$ for every $(\I',y)$ that is normal for~$C$. This and Lemma~\ref{Lemma:NormalityTwiddle} give us~$C\dsubs D$. \hfill\ \qed
\end{proof}
\myskip

\noindent\textbf{Proof of Theorem~\ref{Theorem:RepResultPreferential}}:

\noindent Soundness, the if part, is given in Section~\ref{Proof:PreferentialSoundness}. For the only-if part, let~$\dsubs$ be a preferential subsumption relation and let~$\PI$ be defined as above. Lemmas~\ref{Lemma:NonEmptyDomain}, \ref{Lemma:PrefIsPartialOrder} and~\ref{Lemma:Smoothness} show that~$\PI$ is a preferential~DL interpretation. Lemma~\ref{Lemma:Completeness} shows that~$\PI$ defines a subsumption relation that is exactly~$\dsubs$. \hfill\ \qed

\section{Proof of Theorem~\ref{Theorem:RepResultRational}}\label{ProofRepResultRational}

\restatableRepResultRational*

\subsection{If part}\label{Proof:ModularSoundness}

Satisfaction of the basic~KLM properties for preferential subsumption follows from the proof in Section~\ref{Proof:PreferentialSoundness}, given the fact that modular interpretations are a special case of preferential interpretations. Below we show that rational monotonicity is satisfied.
\myskip

Assume that $C\dsubs_{\RI}E$ but $C\ndsubs_{\RI}\lnot D$. From the latter it follows that there is $x\in\min_{\pref^{\RI}}C^{\RI}$ such that $x\in D^{\RI}$, \ie, $x\in(C\dland D)^{\RI}$. Let now $x'\in\min_{\pref^{\RI}}(C\dland D)^{\RI}$. Since $x\in(C\dland D)^{\RI}$, $x\npref^{\RI} x'$. This means that $x'\in\min_{\pref^{\RI}}C^{\RI}$, for if there is $x''\in C^{\RI}$ such that $x''\pref^{\RI}x'$, then $x''\pref^{\RI}x$, which is impossible since $x$ is minimal in $C^{\RI}$. From $x'\in\min_{\pref^{\RI}}C^{\RI}$ and $\RI\sat C\dsubs E$ follows $x'\in E^{\RI}$. Hence $\RI\sat C\dland D\dsubs E$ and therefore $C\dland D\dsubs_{\RI}E$. \hfill\ \qed

\subsection{Only-if part}\label{Proof:ModularCompleteness}

\textbf{NB}: The results marked $(*)$ are introduced here in the Appendix, while they are omitted in the main text.
\myskip

The proof of the only-if part relies on the results for the preferential case (Section~\ref{Proof:PreferentialSoundness}), with the main difference being the definition of the preference relation, which is shown to be a smooth modular order. This ensures that the canonical model constructed in the proof is a modular interpretation.
\myskip

Let $\dsubs\subseteq\Lang\times\Lang$ satisfy all the basic properties of preferential subsumption relations together with rational monotonicity.

The proof of the following lemma is analogous to that of Lemma~3.11 by Lehmann and Magidor~\cite{LehmannMagidor1992}:

\begin{lemma}[$*$]\label{Lemma:MoreProperties}
If $\dsubs$ is rational, then the properties below hold:
\[
\frac{C\dlor D\dsubs\lnot D}{C\dsubs\lnot D}\quad\quad\quad\frac{C\dlor E\dsubs\lnot C,\ D\dlor E\ndsubs\lnot D}{C\dlor D\dsubs\lnot C}
\]
\end{lemma}

\begin{definition}
Let $C\in\Lang$. We say that $C$ is \df{consistent} \wrt\ $\dsubs$ if $C{\ndsubs}\bot$. Given $\RI=\tuple{\Dom^{\RI},\cdot^{\RI},\pref^{\RI}}$, we say that~$C$ is \df{consistent} \wrt\ $\dsubs_{\RI}$ if $C\ndsubs_{\RI}\bot$, \ie, if there is $x\in\Dom^{\RI}$ \st\ $x\in C^{\RI}$.
\end{definition}

Let $\C=\{ C\in\Lang \mid C \text{ is consistent \wrt\ }\dsubs\}$.

\begin{lemma}[$*$]\label{Lemma:Consistent}
Let $C\in\Lang$ and let $\dsubs$ be a rational relation. Then~$C\in\C$ iff there is $(\I,x)\in\U$ \st\ $(\I,x)$ is normal for~$C$. (Cf.\ Definitions~\ref{Def:Universe} and~\ref{Def:Normality} in Appendix~\ref{Proof:PreferentialCompleteness}.)
\end{lemma}

\newcommand{\nme}{\ensuremath{\mathscr{R}}} 

\begin{definition}
Given $C,D\in\C$, $C$ is \df{not more exceptional than} $D$, written $C\nme D$, if $C\dlor D{\ndsubs}\lnot C$. We say that $C$ is \df{as exceptional as} $D$, written $C\sim D$, if $C\nme D$ and $D\nme C$.
\end{definition}

The proof of the lemma below follows those of Lemmas~A.4 and~A.5 by Lehmann and Magidor~\cite{LehmannMagidor1992}:

\begin{lemma}[$*$]\label{Lemma:RefTrans}
$\nme$ is reflexive and transitive.
\end{lemma}



That $\sim$ is an equivalence relation follows from the fact that $\nme$ is reflexive and transitive~(Lemma~\ref{Lemma:RefTrans}). With $[C]$ we denote the equivalence class of $C$. The set of equivalence classes of concepts of $\C$ under $\sim$ is denoted by $[\C]$. We write $[C]\leq[D]$ if $C\nme D$, and $[C]<[D]$ if $[C]\leq[D]$ and $C\not\sim D$.

Thanks to Lemma~\ref{Lemma:RefTrans} we can state:

\begin{lemma}[$*$]
The relation $<$ is a strict order on $[\C]$.
\end{lemma}

\begin{lemma}[$*$]\label{Lemma:CdsubsnotD}
Let $C,D\in\Lang$ be consistent \wrt\ $\dsubs$. If $[C]<[D]$, then $C\dsubs\lnot D$.
\end{lemma}
\begin{proof}
The assumption implies that $C\nme D$ is not the case, \ie, $C\dlor D\dsubs\lnot C$. This and Lemma~\ref{Lemma:MoreProperties} imply the conclusion.\hfill\ \qed
\end{proof}
\myskip

Lemma~\ref{Lemma:CdsubsnotD} warrants the following result:

\begin{lemma}[$*$]
Let $C,D\in\Lang$ be consistent \wrt\ $\dsubs$. If there is $(\I,x)\in\U$ \st\ $(\I,x)$ is normal for $C$ and $x\in D^{\I}$, then $[D]\leq[C]$.
\end{lemma}

Armed with these results, we can then construct an interpretation~$\RI$ analogous to the preferential interpretation~$\PI$ in Appendix~\ref{Proof:PreferentialCompleteness}, with the preference relation defined as follows:
\begin{itemize}
\item For every $\tuple{\I,x,C}\in\Dom^{\RI}$ such that $C\neq\bot$, $\tuple{\I,x,C}\pref^{\RI}\tuple{\J,y,\bot}$ for every $\tuple{\J,y,\bot}\in\Dom^{\RI}$;
\item For every $\tuple{\I,x,C},\tuple{\J,y,D}\in\Dom^{\RI}$ such that $C,D\neq\bot$, $\tuple{\I,x,C}\pref^{\RI}\tuple{\J,y,D}$ if $[C]<[D]$.
\end{itemize}

It is not hard to see that this definition implies the following result:

\begin{lemma}[$*$]\label{Lemma:Modular}
$\pref^{\RI}$ is a modular partial order.
\end{lemma}



The proof of the following lemma follows that of Lehmann and Magidor's Lemma~A.12~\cite{LehmannMagidor1992}:

\begin{lemma}[$*$]\label{Lemma:SmoothnessRanked}
For every $C\in\Lang$, if $C$ is consistent, then $C^{\RI}$ is smooth.
\end{lemma}

From this point on, a result analogous to Lemma~\ref{Lemma:Completeness} in~\ref{Proof:PreferentialCompleteness} can be shown to hold for the defeasible subsumption~$\dsubs_{\RI}$ induced by~$\RI$. From that the result follows.\hfill\ \qed

\section{Proofs of results in Section~\ref{Entailment}}\label{ProofsEntailment}

\restatableLemmaPrefClosure*
\begin{proof}
By Definitions~\ref{Def:PrefEntailment} and~\ref{Def:PrefClosure}, $\alpha\in\KB^{*}_{\pre}$ iff for every preferential model $\PI$ of $\KB$, $\PI\sat\alpha$. Combined with Lemma \ref{Lemma:ClassicalStatements}, this implies that, for any defeasible subsumption $C\dsubs D$, $C\dsubs D\in\KB^{*}_{\pre}$ iff $(C, D)\in\dsubs_{\PI}$ for every preferential model $\PI$ of $\KB$. Due to Theorem~\ref{Theorem:RepResultPreferential}, the latter condition, that is, $(C, D)\in\dsubs_{\PI}$ for every preferential model $\PI$ of $\KB$, holds  iff $(C, D)\in\dsubs_{\Th}$ for every preferential theory~$\Th$ containing~$\KB$. This concludes the proof.\hfill\ \qed
\end{proof}

\restatableLemmaModClosure*
\begin{proof}
The proof follows the one of Lemma \ref{Lemma:PrefClosure}. It is sufficient to refer to Definitions \ref{Def:ModularEntailment} and \ref{Def:ModularClosure} instead of Definitions \ref{Def:PrefEntailment} and \ref{Def:PrefClosure}, and to Theorem \ref{Theorem:RepResultRational} instead of Theorem \ref{Theorem:RepResultPreferential}. \hfill\ \qed

\end{proof}

\restatablefinitemodular*

\begin{proof}
The preference relation~$\pref^{\RI}$ is a strict partial order, hence, since there cannot be cycles, for every finite set $\emptyset\neq X\subseteq\Dom^{\RI}$, $\min_{\pref^{\RI}}X\neq\emptyset$. We can define the function $h_{\RI}(\cdot)$ in the following way:
\begin{enumerate}
    \item $\Dom^{{\RI}^0}\defined \Dom^{\RI}$; 
    \item $i\defined 0$;
    \item If $\Dom^{{\RI}^i}\neq \emptyset$ proceed, else return the function $h_{\RI}$;
    \item $h_{\RI}(x)=i$ iff $x\in \min_{\pref^{\RI}}\Dom^{{\RI}^i}$; let $\Dom^{{\RI}^{i+1}}\defined \Dom^{{\RI}^i}\setminus\min_{\pref^{\RI}}\Dom^{{\RI}^i}$;
    \item $i\defined i+1$;
    \item Go back to step 3.
\end{enumerate}
It is easy to check that $h_{\RI}(\cdot)$ satisfies the convexity property and characterises $\pref_{\RI}$ (\ie, $x\pref_{\RI}y$ iff $h_{\RI}(x)< h_{\RI}(y)$).
\hfill\ \qed
\end{proof}



\restatableuniquerank*

\begin{proof}
Assume that for a ranked interpretation $\RI=\tuple{\Dom^{\RI},\cdot^{\RI},\pref^{\RI}}$ there are two distinct functions $h_{\RI}(\cdot)$ and $h'_{\RI}(\cdot)$ satisfying the convexity constraint and characterising $\pref^{\RI}$. Since the two functions are distinct, at a certain point they must diverge; that is, there must be an $i\in \mathbb{N}$ \st\ for every $k<i$ and every $x\in \Dom^{\RI}$, $h_{\RI}(x)=k$ iff $h'_{\RI}(x)=k$, but there is a $y\in \Dom^{\RI}$ \st\ $h_{\RI}(y)=i$ and $h'_{\RI}(y)=j$, with $j>i$. The convexity constraint imposes that there must be a $z\in \Dom^{\RI}$ \st\ $h'_{\RI}(z)=i$: then $h'_{\RI}(\cdot)$ enforces $z\pref^{\RI} y$, while according to $h_{\RI}(\cdot)$ that cannot be the case (it must be $h_{\RI}(y)\leq h_{\RI}(z)$). \hfill\ \qed
\end{proof}

\myskip

Some extra material needs to be introduced to prove  Theorem \ref{Theorem:FMP}, stating the Finite Model Property for Defeasible \ALC. First of all, we will refer to the following semantic construction.

\begin{definition}[Finite Model Construction) ($*$]\label{Def:FMC}
Let $\KB=\TB\cup\DB$ be a finite defeasible knowledge base, and let $\RI=\tuple{\Dom^{\RI},\cdot^{\RI},\pref^{\RI}}$ be a modular model of~$\KB$ (with $\Dom^{\RI}$ possibly infinite). Let  $\CN,\RN$ be the sets of names of our language, as from Section \ref{Preliminaries}, and $\Gamma$ be the set of concepts $\{C_1,\ldots,C_n\}\subseteq\Lang$  obtained by closing the set of all concepts appearing in the axioms in~$\KB$ under sub-concepts and negation. We define the equivalence relation $\approx_{\Gamma}$ as follows: for every $x,y\in\Dom^{\RI}$, $x\approx_{\Gamma}y$ if for every $C\in\Gamma$, $x\in C^{\RI}$ iff $y\in C^{\RI}$.

We indicate with $[x]_{\Gamma}$ the equivalence class of the objects that are related to an object~$x$ through $\approx_{\Gamma}$:
\[
[x]_{\Gamma}\defined\{y\in\Dom^{\RI} \mid x\approx_{\Gamma} y\}
\]

We introduce a new model $\RI'=\tuple{\Dom^{\RI'},\cdot^{\RI'},\prec^{\RI'}}$, defined as:

\begin{itemize}\setlength{\itemsep}{1.3ex}
\item $\Dom^{\RI'}=\{[x]_{\Gamma} \mid x\in\Dom^{\RI}\}$;
\item For every $A\in\CN\cap\Gamma$, $A^{\RI'}=\{[x]_{\Gamma} \mid x\in A^{\RI}\}$;
\item For every $A\notin\CN\cap\Gamma$, $A^{\RI'}=\emptyset$;
\item For every $r\in\RN$, $r^{\RI'}=\{([x]_{\Gamma},[y]_{\Gamma}) \mid (x,y)\in r^{\RI}\}$;
\item For every $[x]_{\Gamma},[y]_{\Gamma}\in\Dom^{\RI'}$, $[x]_{\Gamma}\prec^{\RI'} [y]_{\Gamma}$ if there is an object $z\in [x]_{\Gamma}$ \st~for all the objects $v\in [y]_{\Gamma}$, $z\pref^{\RI}v$;
\end{itemize}

Let $\sim^{\RI'}$ be the indifference relation, defined as usual:

\begin{itemize}\setlength{\itemsep}{1.3ex}
\item $[x]_{\Gamma}\sim^{\RI'}[y]_{\Gamma}$ if $[x]_{\Gamma}\not\prec^{\RI'}[y]_{\Gamma}$ and $[y]_{\Gamma}\not\prec^{\RI'}[x]_{\Gamma}$.
\end{itemize}

\end{definition}

Given that $\Gamma$ is finite, $\Dom^{\RI'}$ is clearly finite. The following results are easy to prove.

\begin{lemma}[$*$]\label{Lemma:FiniteModelConcept}
For every $C\in\Gamma$ and every $x\in\Dom^{\RI}$, $x\in C^{\RI}$ iff $[x]_{\Gamma}\in C^{\RI'}$.
\end{lemma}
\begin{proof}
The proof is analogous to that for the classical case and is by induction on the structure of concepts. \hfill\ \qed
\end{proof}

\begin{lemma}[$*$]\label{Lemma:FiniteRanked}
Let $\prec^{\RI'}$ and $\prec^{\RI}$ be as in Definition \ref{Def:FMC}. Then $\prec^{\RI'}$ is a strict partial order.
\end{lemma}
\begin{proof}
We show that $\prec^{\RI'}$ is irreflexive and transitive.
\myskip

\noindent Irreflexivity: Assume $[x]_{\Gamma}\prec^{\RI'}[x]_{\Gamma}$. By the definition of $\prec^{\RI'}$, it implies that there is a $z\in[x]_{\Gamma}$ s.t. $z\pref^{\RI}v$ for every $v\in[x]_{\Gamma}$. That is, we have that $z\pref^{\RI}z$ that, since $\pref^{\RI}$ is a strict partial order (Definitions \ref{Def:PrefInterpretation} and \ref{Def:ModInterpretation}), cannot be the case.
\myskip

\noindent Transitivity: Assume $[x]_{\Gamma}\prec^{\RI'}[y]_{\Gamma}$ and $[y]_{\Gamma}\prec^{\RI'}[u]_{\Gamma}$. 
This means that there is a $z\in[x]_{\Gamma}$ s.t. $z\pref^{\RI}v$ for every $v\in[y]_{\Gamma}$, and there is a $v'\in[y]_{\Gamma}$ s.t. $v'\pref^{\RI}w$ for every $w\in[u]_{\Gamma}$. Since $\pref^{\RI}$ is transitive, it follows that there is a $z\in[x]_{\Gamma}$ s.t. $z\pref^{\RI}w$ for every $w\in[u]_{\Gamma}$, that is, $[x]_{\Gamma}\prec^{\RI'}[u]_{\Gamma}$.
 \hfill\ \qed
\end{proof}

\begin{lemma}[$*$]\label{Lemma:TransIndiff}
Let $\sim^{\RI'}$ be as in Definition \ref{Def:FMC}. Then relation $\sim^{\RI'}$ is transitive.
\end{lemma}
\begin{proof}
Let $[x]_{\Gamma}\sim^{\RI'}[y]_{\Gamma}$, $[y]_{\Gamma}\sim^{\RI'}[u]_{\Gamma}$, but $[x]_{\Gamma}\not\sim^{\RI'}[u]_{\Gamma}$. The latter implies that either $[x]_{\Gamma}\prec^{\RI'}[u]_{\Gamma}$ or $[u]_{\Gamma}\prec^{\RI'}[x]_{\Gamma}$; w.l.o.g. let's assume $[x]_{\Gamma}\prec^{\RI'}[u]_{\Gamma}$. That is, there is a $z\in[x]_{\Gamma}$ s.t. $z\pref^{\RI}w$ for every $w\in[u]_{\Gamma}$. 

$[x]_{\Gamma}\sim^{\RI'}[y]_{\Gamma}$ implies that $z\sim^{\RI}v$ for some $v\in[y]_{\Gamma}$. Assume the latter does not hold, then for every $v\in[y]_{\Gamma}$ either $z\prec^{\RI}v$ or $v\prec^{\RI}z$. It cannot be that $z\prec^{\RI}v$ for every $v\in[y]_{\Gamma}$, since that would imply $[x]_{\Gamma}\prec^{\RI'}[y]_{\Gamma}$, so there must be some $v\in[y]_{\Gamma}$ s.t. $v\prec^{\RI}z$. However the latter would also imply, due to the transitivity of $\prec^{\RI}$, that there is a $v\in[y]_{\Gamma}$ s.t. $v\pref^{\RI}w$ for every $w\in[u]_{\Gamma}$, that is, $[y]_{\Gamma}\prec^{\RI'}[u]_{\Gamma}$, against the hypothesis that $[y]_{\Gamma}\sim^{\RI'}[u]_{\Gamma}$. Consequently, $z\sim^{\RI}v$ for some $v\in[y]_{\Gamma}$.

So there is a $z\in[x]_{\Gamma}$ s.t. $z\pref^{\RI}w$ for every $w\in[u]_{\Gamma}$ and there is a $v\in[y]_{\Gamma}$ s.t. $v\sim^{\RI}z$. That implies that $v\pref^{\RI}w$ for every $w\in[u]_{\Gamma}$. To see it, assume that it not the case, that is, we have that for some $w'\in[u]_{\Gamma}$ either $w'\pref^{\RI} v$ or $w'\sim^{\RI} v$: in the former case we would obtain $z\pref^{\RI} v$, in the latter $z\sim^{\RI} w'$, both taking us to contradiction. Hence $v\pref^{\RI}w$ for every $w\in[u]_{\Gamma}$, that is, $[y]_{\Gamma}\prec^{\RI'}[u]_{\Gamma}$, against the hypothesis. \hfill\ \qed
\end{proof}

\begin{lemma}[$*$]\label{Lemma:FMP}
Let $\KB=\TB\cup\DB$ be finite. If $\KB$ has a modular model, then it has a finite ranked model.
\end{lemma}
\begin{proof}
Let $\KB=\TB\cup\DB$ be a finite defeasible knowledge base, $\RI$ a model of~$\KB$ and~$\RI'$ a finite interpretation constructed from~$\RI$ as in Definition \ref{Def:FMC}. $\RI'$ is a finite interpretation, and it is modular, since Lemmas \ref{Lemma:FiniteRanked} and \ref{Lemma:TransIndiff} prove that $\prec^{\RI'}$  satisfies Definition \ref{Def:Modular}. Being $\RI'$  a finite modular interpretation, it is a finite ranked interpretation (Lemma \ref{finite_modular}).

It remains to prove that $\RI'$ is a model of $\KB$. The proof that~$\RI'$ satisfies~$\TB$ is straightforward by Lemma~\ref{Lemma:FiniteModelConcept}.  
With regard to~$\DB$, let $C\dsubs D\in\DB$. Since $\RI$ is a model of $\DB$, either $C^{\RI}=\emptyset$, or the height of $C\dland D$ in~$\RI$ is lower than the height of $C\dland\lnot D$, that is, there is at least an object~$y$ in $(C\dland D)^{\RI}$ \st~for every object~$x$ in $(C\dland\lnot D)^{\RI}$, $y\pref^{\RI}x$. Since~$C$, $D$ and $\lnot D$ are in $\Gamma$, the object $[y]_{\Gamma}\in\Dom^{\RI'}$ (obtained from $y\in (C\dland D)^{\RI}$) must be preferred to all the objects in $(C\dland\lnot D)^{\RI}$, that is, $[y]_{\Gamma}\pref^{\RI'}[x]_{\Gamma}$ for every object $[x]_{\Gamma}$ \st~$[x]_{\Gamma}\in(C\dland\lnot D)^{\RI'}$. Therefore $\RI'\sat C\dsubs D$. \hfill\ \qed

\end{proof}

\begin{lemma}[$*$]\label{Lemma:FCMP}
Let $\KB=\TB\cup\DB$ be finite and $C,D\in\Lang$. If $\KB$ has a modular model that is a counter-model to~$C\dsubs D$, then it has a finite ranked model that is a counter-model to~$C\dsubs D$.
\end{lemma}
\begin{proof}
It is sufficient to apply the same construction defined for the finite-model property above. We just need to add~$C$ and $D$ to the set~$\Gamma$ (and close $\Gamma$ also under the subconcepts of~$C$ and $D$ and their negations). If~$\RI\nsat C\dsubs D$, then there is an object~$x$ \st\ $x\in(C\dland\lnot D)^{\RI}$ and $x\pref^{\RI}y$ or $x\sim^{\RI}y$ for every object~$y$ \st\ $y\in(C\dland D)^{\RI}$. That implies that in~$\RI'$ $[y]_{\Gamma}\not\prec_{\RI'}[x]_{\Gamma}$, that is,  $[x]_{\Gamma}\prec_{\RI'}[y]_{\Gamma}$ or $[x]_{\Gamma}\sim_{\RI'}[y]_{\Gamma}$ for every~$y$ \st\ $y\in(C\dland D)^{\RI}$, and consequently $\RI'\nsat C\dsubs D$. \hfill\ \qed
\end{proof}

\begin{corollary}[$*$]\label{Cor:FMP_rank0}
Let $\KB=\TB\cup\DB$ be a finite defeasible knowledge base. If $\KB$ has a modular model~$\RI$, then for every $C\in\Lang$ \st\ $h_{\RI}(C)=0$ there is also a finite ranked model~$\RI'$ of~$\KB$ \st\ $h_{\RI'}(C)=0$.
\end{corollary}
\begin{proof}
Given $\KB=\TB\cup\DB$, a model~$\RI$ of~$\KB$ and a concept~$C$ \st\ $h_{\RI}(C)=0$, a finite model~$\RI'$ satisfying the constraint above can be defined in the same way as the model~$\RI'$ from Definition \ref{Def:FMC}. We just need to add~$C$ to the set~$\Gamma$ (and close $\Gamma$ also under the subconcepts of~$C$ and their negations). To see that~$\RI'$ is a model of~$\KB$, just go again through the proof of the finite-model property above, and check that the addition of~$C$ to~$\Gamma$ does not affect any of the above results.

Now, $h_{\RI}(C)=0$ implies that there is an object $x\in\Dom^{\RI}$ \st\ $x\in C^{\RI}$ and $h_{\RI}(x)=0$. Consider now $[x]_{\Gamma}$. By Lemma~\ref{Lemma:FiniteModelConcept}, $[x]_{\Gamma}\in C^{\RI'}$. Since $h_{\RI}(x)=0$, for every $[y]_{\Gamma}\in\Dom^{\RI'}$ it cannot be the case that there is an object $z\in[y]_{\Gamma}$ \st\ $z\pref^{\RI}v$ for every $v\in[x]_{\Gamma}$; hence, the definition of $\prec^{\RI'}$ implies that for every $[y]_{\Gamma}\in\Dom^{\RI'}$, $[y]_{\Gamma}\not\prec^{\RI'}[x]_{\Gamma}$, that is, $h_{\RI'}([x]_{\Gamma})=0$, that implies $h_{\RI'}(C)=0$. \hfill\ \qed
\end{proof}
\myskip

Now we can prove Theorem \ref{Theorem:FMP}.

\restatableFiniteModelProperty*

\begin{proof}
The result follows straightforwardly from Lemmas~\ref{Lemma:FMP} and~\ref{Lemma:FCMP}. \hfill\ \qed
\end{proof}

\myskip

\restatableClosureDisjointUnionMod*

\begin{proof}
Let $\mathfrak{R}$ be a set of ranked models of a defeasible knowledge base $\KB$, and let $\RI^{\mathfrak{R}}\defined\tuple{\Dom^{\mathfrak{R}},\cdot^{\mathfrak{R}},\pref^{\mathfrak{R}}}$ be its ranked union. We want to prove that also $\RI^{\mathfrak{R}}$ is  a ranked model of $\KB$, and to do that is sufficient to prove that for every DCI $C\dsubs D$, if $\RI\sat C\dsubs D$ for every $\RI\in\mathfrak{R}$, then $\RI^{\mathfrak{R}}\sat C\dsubs D$.

It is easy to prove by induction on the construction of the concepts that for every object $x_\RI\in\Dom^{\mathfrak{R}}$ and every concept $C$, $x_\RI\in C^{\mathfrak{R}}$ iff $x\in C^\RI$.

This, together with  the condition that, for every $x_{\RI}\in\Dom^{\mathfrak{R}}$, $h_{\mathfrak{R}}(x_{\RI})=h_{\RI}(x)$, implies that for every concept~$C$, $h_{\RI^{\mathfrak{R}}}(C)=\min\{h_\RI(C)\mid \RI\in\RI^{\mathfrak{R}}\}$.

Now, let $C\dsubs D$ be satisfied by every $\RI\in\mathfrak{R}$. Hence, for every $\RI\in\mathfrak{R}$, either $h_\RI(C\dland D)<h_\RI(C\dland \neg D)$ or $h_\RI(C)=\infty$. Since $h_{\RI^{\mathfrak{R}}}(C)=\min\{h_\RI(C)\mid \RI\in\RI^{\mathfrak{R}}\}$, $h_{\RI^{\mathfrak{R}}}(C\dland D)=\min\{h_\RI(C\dland D)\mid \RI\in\RI^{\mathfrak{R}}\}$, and  $h_{\RI^{\mathfrak{R}}}(C\dland \neg D)=\min\{h_\RI(C\dland \neg D)\mid \RI\in\RI^{\mathfrak{R}}\}$, it is easy to check that $\RI^{\mathfrak{R}}$ satisfies $C\dsubs D$ too: assume that is not the case, that is, $h_{\RI^{\mathfrak{R}}}(C\dland \neg D)\leq h_{\RI^{\mathfrak{R}}}(C\dland D)$ and $h_{\RI^{\mathfrak{R}}}(C)<\infty$; then we have that $\min\{h_\RI(C\dland\neg D)\mid \RI\in\RI^{\mathfrak{R}}\}\leq \min\{h_\RI(C\dland D)\mid \RI\in\RI^{\mathfrak{R}}\}$ and $\min\{h_\RI(C)\mid \RI\in\RI^{\mathfrak{R}}\}<\infty$, that, since for every $\RI\in\mathfrak{R}$, either $h_\RI(C\dland D)<h_\RI(C\dland \neg D)$ or $h_\RI(C)=\infty$,  cannot be the case. \hfill\ \qed

\end{proof}

\restatableCountablyInfiniteDomain*
\begin{proof}
Let $\Delta$ be a countably infinite domain. For the only-if part, if $\KB\entails_{\modular}C\dsubs D$, then obviously $\RI\sat C\dsubs D$ for every $\RI\in\Mod_{\Dom}(\KB)$. For the if part, assume $\KB\nentails_{\modular}C\dsubs D$. Then, thanks to the finite-model property (Theorem~\ref{Theorem:FMP}), there is a modular model~$\RI_{\text{fin}}$ with a finite domain that is a model of~$\KB$ and a counter-model of $C\dsubs D$; since the domain is finite, the modular model $\RI_{\text{fin}}$ is a  ranked model (Lemma \ref{finite_modular}). Given~$\RI_{\text{fin}}$, we can extend it to a model of~$\KB$ that is a counter-model of $C\dsubs D$ with a countably infinite domain in the following way: make a countably infinite number of copies of $\RI_{\text{fin}}$ and make the ranked union of them. Now, let $\RI'=\tuple{\Dom^{\RI'},\cdot^{\RI'},\pref^{\RI'}}$ be the result of such ranked union, that is, a ranked model of~$\KB$ and a counter-model of $C\dsubs D$ with $\Dom^{\RI'}$ being countably infinite (it is the disjoint union of a countably infinite number of finite domains). It is easy to build an isomorphic interpretation $\RI=\tuple{\Delta,\cdot^{\RI},\pref^{\RI}}$, once we have defined a bijection $b: \Dom^{\RI'}\longrightarrow\Delta$, which must exist, being both~$\Dom^{\RI'}$ and~$\Delta$ countably infinite sets. We can define~$\cdot^{\RI}$ and~$\pref^{\RI}$ in the following way:

\begin{itemize}
\item For every $A\in\CN$ and every $x\in\Dom^{\RI'}$, $b(x)\in A^{\RI}$ iff $x\in A^{\RI'}$;
\item For every $r\in\RN$ and every $x,y\in\Dom^{\RI'}$, $(b(x),b(y))\in r^{\RI}$ iff $(x,y)\in r^{\RI'}$;
\item For every $x\in\Dom^{\RI'}$, $h_{\RI}(b(x))=h_{\RI'}(x)$.
\end{itemize}

It is easy to prove by induction on the construction of the concepts that for every $C\in\Lang$ and every $x\in\Dom^{\RI'}$, $x\in C^{\RI'}$ iff $b(x)\in C^{\RI}$. Moreover, $x\in\min_{\pref^{\RI'}}(C^{\RI'})$ iff $b(x)\in\min_{\pref^{\RI}}(C^{\RI})$. Hence, there is a ranked $\KB$-model which is a counter model for $C\dsubs D$ with~$\Delta$ as its domain. \hfill\ \qed
\end{proof}

\section{Proofs of results in Section~\ref{RationalClosure}}\label{ProofsComputingRC}

\textbf{NB}: The results marked $(*)$ are introduced here in the Appendix, while they are omitted in the main text.

\restatableInfiniteRank*

\begin{proof}
If $\KB$ does not have a modular model or  $C$ is never satisfiable, then the result is straightforward. Let $\KB=\TB\cup\DB$ have a modular model, and let $C$ be satisfiable. Also, let $\DB$ be ranked into $\tuple{\DB^{\rank}_0,\ldots,\DB^{\rank}_n,\DB^{\rank}_\infty}$.

From left to right, let $\rank_{\KB}(C)=\infty$ but $\KB\not\entails_{\modular}C\subs\bot$. Together they imply that $\TB\cup\DB^{\rank}_\infty\entails_{\modular}\top\dsubs \neg C$ but $\TB\cup\DB^{\rank}_\infty\not\entails_{\modular}C\subs\bot$. Hence, due to the FMP (Theorem \ref{Theorem:FMP}), there is a finite ranked model $\RI$ of $\TB\cup\DB^{\rank}_\infty$ with the domain $\Dom^{\RI}$ layered into $(L^{\RI}_0,\ldots,L^{\RI}_n)$, and s.t. $\RI\sat \top\dsubs \neg C$ but $\RI\not\sat C\subs\bot$, that is, in $\Dom^{\RI}$ there is an object $o$ s.t. $o\in L^{\RI}_i$, with $0<i\leq n$, and $o\in C^{\RI}$. 

Now let's define a new model $\RI'$ simply taking the lower layer and putting it at the ``top'' of our model, that is, we rearrange the interpretation in the following way:
\begin{itemize}
    \item $\Dom^{\RI'}=\Dom^{\RI}$;
    \item $\cdot^{\RI'}=\cdot^{\RI}$;
    \item $L^{\RI'}_n=L^{\RI}_{0}$;
    \item for every $i<n$, $L^{\RI'}_i=L^{\RI}_{i+1}$.
\end{itemize}

Clearly for every concept $D$, $D^{\RI'}=D^{\RI}$ (it is easy to prove by induction on the construction of the concepts), and consequently $\RI'$ is still a model of $\TB$. We can prove that is still also a model of $\DB^{\rank}_\infty$. Assume that is not the case, that is, there is a some $D\dsubs E\in \DB^{\rank}_\infty$ s.t. $\RI\sat D\dsubs E$ and $\RI'\not\sat D\dsubs E$. $\RI\sat D\dsubs E$ if either $h_{\RI}(D\dland  E) <h_{\RI}(D\dland \neg E)$ or $h_{\RI}(D)=\infty$. It cannot be the latter, since $h_{\RI}(D)=\infty$ corresponds to $D^{\RI}=\emptyset$, and we would have also $D^{\RI'}=\emptyset$ and $h_{\RI'}(D)=\infty$. Hence it must be $h_{\RI}(D\dland  E) <h_{\RI}(D\dland \neg E)$, while $h_{\RI}(D\dland  E) \not<h_{\RI}(D\dland \neg E)$. Let $h_{\RI}(D\dland  E)=i$ and $h_{\RI}(D\dland \neg E)=j$ with $i<j$. If $i>0$, then $h_{\RI'}(D\dland  E)=i-1$ and $h_{\RI'}(D\dland \neg E)=j-1$, and $h_{\RI'}(D\dland  E)< h_{\RI'}(D\dland \neg E)$ again; hence it must be $h_{\RI}(D\dland  E)=0$, that is, $h_{\RI}(D)=0$, but that is incompatible with $D\dsubs E\in \DB^{\rank}_\infty$, since  $\TB\cup\DB^{\rank}_\infty\entails_{\modular}\top\dsubs \neg D$, that is, $h_{\RI}(D)>0$. Consequently, $\RI'$ too must be a model of $\TB\cup\DB^{\rank}_\infty$. 

We have assumed that in $\RI$ there is an object $o$ s.t. $o\in C^{\RI}$ and $o\in L^{\RI}_i$ for some $0<i\leq n$. Repeating the procedure used to define $\RI'$ for $i$ times, we obtain a model $\RI^*$ of $\TB\cup\DB^{\rank}_\infty$ s.t. $o\in C^{\RI^*}$ and $o\in L^{\RI^*}_0$. However, since $\rank_{\KB}(C)=\infty$ implies $\TB\cup\DB^{\rank}_\infty\entails_{\modular}\top\dsubs \neg C$, this cannot be the case. We conclude that if $\rank_{\KB}(C)=\infty$ then $\KB\entails_{\modular}C\subs\bot$.

From right to left, let $\KB\entails_{\modular}C\subs\bot$ but $\rank_{\KB}(C)\neq\infty$. The latter implies that there is a model of $\TB\cup\DB^{\rank}_\infty$ that does not satisfy $\top\dsubs\neg C$, that is, does not satisfy $C\subs\bot$. Referring again to the FMP (Theorem \ref{Theorem:FMP}), we can say that there is a finite ranked model $\RI$ of $\TB\cup\DB^{\rank}_\infty$ that does not satisfy $C\subs\bot$. Let $k$ be the number of layers in $\RI$.

Now consider $\TB\cup(\DB^{\rank}_n \cup\DB^{\rank}_\infty)$. For each $D\dsubs E\in \DB^{\rank}_n$ there must be a model in which $D\dland E$ is not exceptional, that is, it is satisfied in the layer $0$. As a consequence, still using the FMP (Corollary \ref{Cor:FMP_rank0}), for each $D\dsubs E\in \DB^{\rank}_n$ there must a finite ranked model $\RI_{D{\tiny \dsubs} E}$ of $\TB\cup(\DB^{\rank}_n \cup\DB^{\rank}_\infty)$ s.t. $h_{\RI_{D{\tiny \dsubs} E}}(D\dsubs E)=0$.

Build a ranked interpretation $\RI_n$ as follows:
\begin{itemize}
    \item for every $D\dsubs E\in \DB^{\rank}_n$, let $\RI_{D{\tiny \dsubs} E}$ be a finite ranked model of $\TB\cup(\DB^{\rank}_n \cup\DB^{\rank}_\infty)$ in which
    $h_{\RI_{D{\tiny \dsubs} E}}(D\dsubs E)=0$.
    \item Let $\RI'=\tuple{\Dom^{\RI'},\cdot^{\RI'},\prec^{\RI'}}$ be the ranked union of such sets. $\RI'$ is a model of $\DB^{\rank}_n$ (Lemma \ref{Lemma:ClosureDisjointUnionMod}) s.t. for every $D\dsubs E\in \DB^{\rank}_n$, $h_{\RI_{D{\tiny \dsubs} E}}(D\dsubs E)=0$. Since $\DB^{\rank}_n$ is finite, it has been obtained from a finite set of finite models and so it is a finite ranked model. Let $m$ be the number of layers in $\RI'$.
    \item From $\RI'$ and $\RI$ define a finite ranked interpretation $\RI_n=\tuple{\Dom^{\RI_n},\cdot{\RI_n},\prec^{\RI_n}}$ as:
    \begin{itemize}
        \item $\Dom^{\RI_n}=\Dom^{\RI}\cup\Dom^{\RI'}$;
        \item $A^{\RI_n}=A^{\RI}\cup A^{\RI'}$ for every $A\in\CN$;
        \item $r^{\RI_n}=r^{\RI}\cup r^{\RI'}$ for every $r\in\RN$;
        \item for every $i\leq m$, $L^{\RI_n}_i=L^{\RI'}_i$;
        \item for every $m< i\leq (m+k)$, $L^{\RI_n}_i=L^{\RI}_{(i-(m+1))}$.
    \end{itemize}
\end{itemize}

Informally, we build the model $\RI_n$ by adding $\RI$ on top of $\RI'$. It is easy to prove by induction on the construction of the concepts that every object in $\RI_n$ satisfies a concept $D$ iff it satisfies $D$ also in the original model, $\RI$ or $\RI'$. As a consequence, $\RI_n\not\sat C\subs\bot$. Also, it is easy to prove that $\RI_n$ is a model of $\TB\cup(\DB^{\rank}_n \cup\DB^{\rank}_\infty)$: $\RI'$ is a model of $\DB^{\rank}_n$ with at layer $0$ an object satisfying $D\dland E$ for each $D\dsubs E\in \DB^{\rank}_n$, and both $\RI$ and $\RI'$ are models of $\TB\cup\DB^{\rank}_\infty$.

Now consider $\TB\cup(\DB^{\rank}_{(n-1)}\cup\DB^{\rank}_n \cup\DB^{\rank}_\infty)$. Using the same procedure defined for $\RI_n$ we can build a model $\RI_{n-1}$, obtained doing the ranked union of a finite set of finite models of $\TB\cup(\DB^{\rank}_{(n-1)}\cup\DB^{\rank}_n \cup\DB^{\rank}_\infty)$ and adding on top $\RI_n$. $\RI_{n-1}$ will be a finite ranked model of $\TB\cup(\DB^{\rank}_{(n-1)}\cup\DB^{\rank}_n \cup\DB^{\rank}_\infty)$ s.t. $\RI_{n-1}\not\sat C\subs\bot$.

We can go on with this procedure until we define a finite ranked model $\RI_0$ of $\TB\cup(\DB^{\rank}_0\cup\ldots\cup\DB^{\rank}_n \cup\DB^{\rank}_\infty)$. That is, $\RI_0$ is a model of $\TB\cup\DB$ s.t. $\RI_0\not\sat C\subs\bot$, against the hypothesis that $\KB\entails_{\modular}C\subs\bot$.  \hfill\ \qed
\end{proof}

\myskip
In order to prove Theorem \ref{Theorem:Characterisation}, we will use the following lemma.

\begin{lemma}[*]\label{Lemma:Characterisation}
Let $\KB=\TB\cup\DB$ be a defeasible knowledge base having a modular model, $\OI$ its big ranked model, and $\Dom$ the countably infinite domain used to define $\OI$. For every $C\dsubs D\in \DB$, $\rank_{\KB}(C\dland D)=i$ iff there is a model $\RI_{\Dom}\in \Mod_{\Dom}(\KB)$ s.t. $h_{\RI_{\Dom}}(C\dland D)=i$.
\end{lemma}

\begin{proof}
First of all, we observe that the exceptionality function in Definition~\ref{Def:Exceptionality} is correctly captured in the model~$\OI$, that is, for every $C\in\Lang$, $\KB\entails_{\modular}\top\dsubs\lnot C$ iff $\OI\sat\top\dsubs\lnot C$. Indeed, by Lemma~\ref{Lemma:CountablyInfiniteDomain}, a concept~$C$ is exceptional \wrt\ $\KB$ iff $\RI_{\Dom}\sat\top\dsubs\lnot C$, for every $\RI_{\Dom}\in\Mod_{\Dom}(\KB)$, which immediately corresponds to $\OI\sat\top\dsubs\lnot C$.

Since $\RI_{\Dom}\sat\KB$ for every $\RI_{\Dom}\in \Mod_{\Dom}(\KB)$, if $h_{\RI_{\Dom}}(C)=i$, it is immediate that $\rank_{\KB}(C)\leq i$, otherwise it would be $h_{\RI_{\Dom}}(C)>i$ for every $\RI_{\Dom}\in \Mod_{\Dom}(\KB)$. We have to prove that for every $C\dsubs D\in\DB$, if $\rank_{\KB}(C\dland D)=i$, then there is a $\RI_{\Dom}\in \Mod_{\Dom}(\KB)$ s.t. $h_{\RI_{\Dom}}(C\dland D)=i$. In case $i=\infty$, Lemma \ref{Lemma:InfiniteRank} guarantees that if $\rank_{\KB}(C\dland D)=\infty$, then for all the $\RI_{\Dom}\in \Mod_{\Dom}(\KB)$, $h_{\RI_{\Dom}}(C\dland D)=\infty$. In case $i< \infty$, we can prove it by induction on the ranking value~$i$.

Let $C\dsubs D\in\DB$, and let $\rank_{\KB}(C\dland D)=i$. For $i=0$,  we already have  all that is needed to prove that there is a $\RI_{\Dom}\in \Mod_{\Dom}(\KB)$ s.t. $\RI_{\Dom}\not\sat\top\dsubs\lnot (C\dland D)$: 

\begin{itemize}
    \item $\rank_{\KB}(C\dland D)=0$ iff $\KB\not\entails_{\modular}\top\dsubs\lnot (C\dland D)$ (Definition~\ref{Def:Exceptionality});
    \item $\KB\not\entails_{\modular}\top\dsubs\lnot (C\dland D)$ iff there is a $\RI_{\Dom}\in \Mod_{\Dom}(\KB)$ s.t. $\RI_{\Dom}\not\sat\top\dsubs\lnot (C\dland D)$ (by Lemma~\ref{Lemma:CountablyInfiniteDomain});
    \item $\RI_{\Dom}\not\sat\top\dsubs\lnot (C\dland D)$ iff $h_{\RI_{\Dom}}(C\dland D)=0$.
\end{itemize}

For $i>0$, we can define a modular model $\RI$ of $\KB$ as follows:

Let $C\dsubs D\in\DB$ with $\rank_{\KB}(C\dland D)=i$, and let $\DB^{\rank}_{\geq i}$ be the subset of~$\DB$ containing the DCIs with a ranking value of at least~$i$, and $\DB^{\rank}_{< i}=\DB\setminus \DB^{\rank}_{\geq i}$. Let $\RI'$ be a modular model of $\TB\cup\DB^{\rank}_{\geq i}$ such that $h_{\RI'}(C\dland D)=0$. Such a model must exist, since $\rank_{\KB}(C\dland D)=i$, that is, $C\dland D$ is not exceptional in $\TB\cup\DB^{\rank}_{\geq i}$. We can assume that~$\RI'$ has a finite domain, given the finite-model property (Corollary~\ref{Cor:FMP_rank0}), and hence it is a ranked model (Lemma \ref{finite_modular}). 

For each DCI $D\dsubs E\in\DB^{\rank}_{< i}$, that is, such that $\rank_{\KB}(D\dland E)=j$ for some $j<i$, let $\RI_{D\dland E}\in\Mod_{\Dom}(\KB)$ be a model of~$\KB$ satisfiying $D\dland E$ s.t. $h_{\RI_{D\dland E}}(D\dland E)=j$. The induction hypothesis guarantees that such a model exists for each such DCI. 

Now we define a new interpretation $\RI''=\tuple{\Dom^{\RI''},\cdot^{\RI''},\pref^{\RI''}}$ in the following way:

\begin{itemize}
\item $\Dom^{\RI''}=\Dom^{\RI'}\cup \bigcup_{(C{\tiny \dsubs} D\in\DB^{\rank}_{< i})}\Dom^{\RI_{C\dland D}}$;
\item For every concept name $A\in\CN$ and every $x\in\Dom^{\RI''}$, $x\in A^{\RI''}$ iff one of the two following cases holds: either $x\in\Dom^{\RI_{D\dland E}}$ for some $D\dsubs E\in\DB^{\rank}_{< i}$ and $x\in A^{\RI_{C\dland D}}$, or $x\in\Dom^{\RI'}$ and $x\in A^{\RI'}$;
\item For every role name $r\in\RN$ and every $x,y\in\Dom^{\RI''}$, $(x,y)\in r^{\RI''}$ iff one of the two following cases holds: either $x,y\in\Dom^{\RI_{C\dland D}}$ for some $C\dsubs D\in\DB^{\rank}_{< i}$ and $(x,y)\in r^{\RI''}$, or $x,y\in\Dom^{\RI'}$ and $(x,y)\in r^{\RI'}$;
\item For every $x\in\Dom^{\RI''}$, $h_{\RI''}(x)=j$ iff one of the two following cases holds: either $x\in\Dom^{\RI_{D\dland E}}$ for some $D\dsubs E\in\DB^{\rank}_{< i}$ and and $h_{\RI_{D\dland E}}(x)=j$, or $x\in\Dom^{\RI'}$ and $h_{\RI'}(x)=j-i$.
\end{itemize}

 The idea is to create a model of $\KB$ that guarantees for a specific inclusion $C\dsubs D\in\DB$ that the height of $C$ in the model corresponds exactly to the rank of $C$. That is, given an inclusion $C\dsubs D$ that has rank $i$, we have built a ranked interpretation $\RI''$ in which $C$ has height $i$. Now we need to:
\begin{itemize}
    \item Prove that $\RI''$ is a model of $\KB$;
    \item Show that an isomorphic model to $\RI''$ is in $\Mod_{\Dom}(\KB)$. 
\end{itemize}

It can easily be proven that~$\RI''$ is a model of~$\KB$: first we prove by induction on the construction of concepts that, for every $x\in\Dom^{\RI''}$, $x\in D^{\RI''}$ iff the corresponding object falls under~$D$ in the original model; this grants us that $\RI''$ satisfies $\TB$. About the satisfaction of $\DB$, referring to the height values that have been assigned to  each object in $\RI''$, we can prove that for every  $D\dsubs E\in\DB$, $h_{\RI''}(D\dland E)<h_{\RI''}(D\dland \lnot E)$ (or $h_{\RI''}(D)=\infty$). Hence, $\RI''$ is a model of~$\KB$. Also, notice that in $\RI'$ we must have an object $o$ s.t. $h_{\RI'}(o)=0$ and $o\in (C\dland D)^{\RI'}$. The construction of of $\RI''$ implies that $h_{\RI'}(o)=i$ and $o\in (C\dland D)^{\RI''}$. That is, $\RI''$ is a model of~$\KB$ in which  $h_{\RI''}(C\dland D)=i$.

$\Dom^{\RI''}$ has been created unifying a finite number of model with the countably infinite domain $\Dom$ plus the finite domain $\Dom^{\RI'}$, hence $\Dom^{\RI''}$ has a countably infinite domain, and there is a model $\RI^{''}_{\Dom}$ that is isomorphic to $\RI''$ and has $\Dom$ as domain. \hfill\ \qed

\end{proof}

\myskip
Using Lemma \ref{Lemma:Characterisation}, we can prove Theorem \ref{Theorem:Characterisation}.

\restatableCharacterisation*
\begin{proof}
Let $\KB$ be a defeasible knowledge base with a modular model, $\OI$  the big ranked model of $\KB$, and $\Dom$ the countably infinite domain used to define $\OI$. $\KB\entails_{\rat} \alpha$ iff $\OI\sat \alpha$ (Definition \ref{Def:RationalEntailment}), so we need to prove that $\OI\sat \alpha$ iff $\alpha $ is in the rational closure of $\KB$. We first prove the result where $\alpha$ is a DCI (of the form $C\dsubs D$), that is, we need to prove that $\OI\sat C\dsubs D$ iff  $\rank_{\KB}(C\dland D)<\rank_{\KB}(C\dland \lnot D)$ or $\rank_{\KB}(C)=\infty$. In turn, that means that we need to prove that for every DCI $C\dsubs D$, $h_{\OI}(C\dland D)<h_{\OI}(C\dland \lnot D)$ or $h_{\OI}(C)=\infty$ iff $\rank_{\KB}(C\dland D)<\rank_{\KB}(C\dland \lnot D)$ or $\rank_{\KB}(C)=\infty$.
Such a result follows immediately if we can prove that, for every concept $C$, $h_{\OI}(C)=\rank_{\KB}(C)$, and that's what we are going to do.

An immediate consequence of Lemma \ref{Lemma:Characterisation} is  that, for every $C\dsubs D\in \DB$, $h_{\OI}(C\dland D)=\rank_{\KB}(C\dland D)$. Being $\OI$ a model of $\DB$,  if $C\dsubs D\in \DB$ then  $h_{\OI}(C\dland D)=h_{\OI}(C)$ and $\rank_{\KB}(C\dland D)=\rank_{\KB}(C)$. So, for every $C\dsubs D\in \DB$, $h_{\OI}(C)=h_{\OI}(C\dland D)=\rank_{\KB}(C\dland D)=\rank_{\KB}(C)$. Now we extend such a result to any concept $C$, using a construction that is in line with the one used to prove Lemma \ref{Lemma:Characterisation}.

Since $\OI\sat\KB$, if $h_{\OI}(C)=i$, it is immediate that $\rank_{\KB}(C)\leq i$, otherwise it would be $h_{\OI}(C)>i$. We have to prove that for every concept $C$, if $\rank_{\KB}(C)=i$, then $h_{\OI_{\Dom}}(C)=i$, that is, there is a $\RI_{\Dom}\in \Mod_{\Dom}(\KB)$ s.t. $h_{\RI_{\Dom}}(C)=i$. In case $i=\infty$, Lemma \ref{Lemma:InfiniteRank} guarantees that if $\rank_{\KB}(C)=\infty$, then for all the $\RI_{\Dom}\in \Mod_{\Dom}(\KB)$, $h_{\RI_{\Dom}}(C)=\infty$. In case $i< \infty$, we can prove it by induction on the ranking value~$i$.

Let $\rank_{\KB}(C)=i$. For $i=0$,  we already have  all that is needed to prove that there is a $\RI_{\Dom}\in \Mod_{\Dom}(\KB)$ s.t. $\RI_{\Dom}\not\sat\top\dsubs\lnot (C)$: 

\begin{itemize}
    \item $\rank_{\KB}(C)=0$ iff $\KB\not\entails_{\modular}\top\dsubs\lnot (C)$ (Definition~\ref{Def:Exceptionality});
    \item $\KB\not\entails_{\modular}\top\dsubs\lnot (C)$ iff there is a $\RI_{\Dom}\in \Mod_{\Dom}(\KB)$ s.t. $\RI_{\Dom}\not\sat\top\dsubs\lnot (C)$ (by Lemma~\ref{Lemma:CountablyInfiniteDomain});
    \item $\RI_{\Dom}\not\sat\top\dsubs\lnot (C)$ iff $h_{\RI_{\Dom}}(C)=0$.
    \item $h_{\OI}(C)=0$ iff there is a $\RI_{\Dom}\in \Mod_{\Dom}(\KB)$ s.t. $h_{\RI_{\Dom}}(C)=0$.
\end{itemize}

For $i>0$, we can define a modular model $\RI$ of $\KB$ as follows:

Let  $\rank_{\KB}(C)=i$, and, as in Lemma \ref{Lemma:Characterisation}, let $\DB^{\rank}_{\geq i}$ be the subset of~$\DB$ containing the DCIs with a ranking value of at least~$i$, and $\DB^{\rank}_{< i}=\DB\setminus \DB^{\rank}_{\geq i}$. Let $\RI'$ be a modular model of $\TB\cup\DB^{\rank}_{\geq i}$ such that $h_{\RI'}(C)=0$. Such a model must exist, since $\rank_{\KB}(C)=i$, that is, $C$ is not exceptional in $\TB\cup\DB^{\rank}_{\geq i}$. We can assume that~$\RI'$ has a finite domain, given the finite-model property (Corollary~\ref{Cor:FMP_rank0}), and hence it is a ranked model (Lemma \ref{finite_modular}). 

Then we define an interpretation  $\RI''$ exactly as done in the proof of Lemma \ref{Lemma:Characterisation}, and, exactly as in Lemma \ref{Lemma:Characterisation}, we can prove that $\RI''$ is a model of $\KB$ with a countably infinite domain and s.t. $h_{\RI''}(C)=i$. 

That implies that there is a model $\RI^{''}_{\Dom}\in \Mod_{\Dom}(\KB)$ that is isomorphic to $\RI''$, with $h_{\RI^{''}_{\Dom}}(C)=i$. 

Since $\RI^{''}_{\Dom}$ is used in the construction of $\OI$, $h_{\OI}(C)\leq i$; since $\rank_{\KB}(C)=i$, $h_{\OI}(C)\geq i$. Hence, $h_{\OI}(C)= i$. \hfill\ \qed

\end{proof}


\restatableExceptionality*
\begin{proof} 
It suffices to prove that if $\TB\cup\DB\nentails_{\modular}\top\dsubs\lnot C$ then $\TB\nentails\bigsqcap\mat{\DB}\subs\lnot C$. So, suppose that $\TB\cup\DB\nentails_{\modular}\top\dsubs\lnot C$. This means there is a modular model $\RI$ of $\TB\cup\DB$ for which we have an $x\in\Delta^{\RI}$ such that $x\in C^{\RI}$. Let $\I$ be the classical interpretation associated with $\RI$. It follows immediately that $\I$ is a model of $\TB$ and that $x\in(\bigsqcap\mat{\DB})^{\I}$, but that $x\notin(\lnot C)^{\I}$.  \hfill\ \qed
\end{proof}

\begin{lemma}[$*$]
\label{Prop:RationalInducedSubsumption}
Let $\KB=\TB\cup\DB$. Then (\emph{i})~$\KB\subseteq\Cn{\rat}(\KB)$ and (\emph{ii})~$\Cn{\rat}(\KB)$ induces a defeasible subsumption relation $\dsubs^{\KB}_{\rat}\defined\{(C,D) \mid \KB\proves_{\rat}C\dsubs D\}$ that is rational.
\end{lemma}
\begin{proof}
Let $\KB=\TB\cup\DB$.
\myskip

\noindent Proving~(\emph{i}): Assume $C\subs D\in\TB$. $\KB\proves_{\rat}C\subs D$ iff $\TB^{*}\entails C\subs D$; since $\TB\subseteq\TB^{*}$, $\TB\subseteq\Cn{\rat}(\KB)$. Assume that $C\dsubs D\in\DB$. Either $C\dsubs D$ ends up in $\DB^{*}_\infty$, or there will be an~$i$ ($0\leq i\leq n$) \st\ $\rk(C)=\rk(C\dsubs D)=i$. In the former case, $C\subs D$ is in $\TB^{*}$, and so $\TB^{*}\entails C\subs D$, \ie, $\KB\proves_{\rat}C\dsubs D$. In the latter case, $\entails\mat{\E_{i}}\subs\lnot C\dlor D$, and so $\TB^{*}\entails\mat{\E_{i}}\dland C\subs D$, \ie, $C\dsubs D\in\Cn{\rat}(\KB)$. Hence $\TB\cup\DB\subseteq\Cn{\rat}(\KB)$.
\myskip

\noindent Proving~(\emph{ii}): Let $\dsubs^{\KB}_{\rat}\defined\{(C,D) \mid \KB\proves_{\rat}C\dsubs D\}$. We show $\dsubs^{\KB}_{\rat}$ satisfies all rationality properties.

\begin{itemize}\setlength{\itemsep}{1.3ex}
\item (Ref). Since $\entails C\subs C$ is valid for any $C\in\Lang$, we have that $\TB^{*}\entails\mat{\E_{i}}\dland C\subs C$ for any $\TB^{*}$ and $\mat{\E_{i}}$.
\item (LLE). $C\dsubs E\in\Cn{\rat}(\KB)$ implies that $\TB^{*}\entails\mat{\E_{i}}\dland C\subs E$ for some~$i$ (or $\TB^{*}\entails C\subs E$, if $\rk(C)=\infty$). Since $\entails C\equiv D$, $\mat{\E_{i}}$ is the lowest~$i$ \st\ $\TB^{*}\nentails\bigsqcap\mat{\E_{i}}\sqsubseteq\lnot D$, and $\TB^{*}\entails\mat{\E_{i}}\dland D\subs E$, too.
\item (And). $\TB^{*}\entails\bigsqcap\mat{\E_{i}}\dland C\subs D$ and $\TB^{*}\entails \bigsqcap\mat{\E_{i}}\dland C\subs E$ (possibly without $\bigsqcap\mat{\E_{i}}$, if~$C$ has an infinite rank), hence $\TB^{*}\entails\bigsqcap\mat{\E_{i}}\dland C\subs D\dland E$, that is, $C\dsubs D\dland E\in\Cn{\rat}(\KB)$.
\item (Or). $\TB^{*}\entails\bigsqcap\mat{\E_{i}}\dland C\subs E$ for some~$i$ and $\TB^{*}\entails\bigsqcap\mat{\E_{j}}\dland D\subs E$ for some~$j$. Assume that $i\leq j$ and $i<\infty$, that is, $\entails\bigsqcap\mat{\E_{i}}\subs\bigsqcap\mat{\E_{j}}$. Then, since $\TB^{*}\nentails\bigsqcap\mat{\E_{i}}\subs\lnot C$, we have that $\TB^{*}\nentails\bigsqcap\mat{\E_{i}}\subs\lnot(C\dlor D)$. Moreover $\TB^{*}\entails\bigsqcap\mat{\E_{j}}\dland D\subs E$ and $\entails\bigsqcap\mat{\E_{i}}\subs\bigsqcap\mat{\E_{j}}$ imply that $\TB^{*}\entails \bigsqcap\mat{\E_{i}}\dland D\subs E$. So, $\TB^{*}\entails\bigsqcap\mat{\E_{i}}\dland(C\dlor D)\subs E$. The proof is analogous for $j\leq i$ with $j<\infty$, or if~$i$ and~$j$ correspond to~$\infty$.
\item (RW). $C\dsubs D\in\Cn{\rat}(\KB)$ if $\TB^{*}\entails\bigsqcap\mat{\E_{i}}\dland C\subs D$ for some $\bigsqcap\mat{\E_{i}}$ (or $\TB^{*}\entails C\subs D$, if $\rk(C)=\infty$). Since $\entails D\subs E$, $\TB^{*}\entails\bigsqcap\mat{\E_{i}}\dland C\subs E$.
\item (CM). If $\rk(C)=i<\infty$, $\TB^{*}\entails\bigsqcap\mat{\E_{i}}\dland C\subs D$ and $\TB^{*}\entails\bigsqcap\mat{\E_{i}}\dland C\subs E$ for some $\bigsqcap\mat{\E_{i}}$. Since $\TB^{*}\entails\bigsqcap\mat{\E_{i}}\dland C\subs D$ and $\TB^{*}\nentails\bigsqcap\mat{\E_{i}}\subs\lnot C$, $\TB^{*}\nentails\bigsqcap\mat{\E_{i}}\subs\lnot(C\dland D)$, otherwise we would have $\TB^{*}\entails\bigsqcap\mat{\E_{i}}\dland C\subs D\dland\lnot D$, \ie, $\TB^{*}\entails\bigsqcap\mat{\E_{i}}\subs\lnot C$. Hence we have $C\dland D\dsubs E\in\Cn{\rat}(\KB)$ since $\TB^{*}\entails\bigsqcap\mat{\E_{i}}\dland C\dland D\subs E$. If $\rk(C)=\infty$, we have $\TB^{*}\entails C\subs\bot$, and the proof is trivial.
\item (RM). If $\rk(C)=i<\infty$, $\TB^{*}\entails\bigsqcap\mat{\E_{i}}\dland C\subs E$ and $\TB^{*}\nentails\bigsqcap\mat{\E_{i}}\dland C\subs\lnot D$ for some $\bigsqcap\mat{\E_{i}}$. Since $\TB^{*}\nentails\bigsqcap\mat{\E_{i}}\dland C\subs\lnot D$, $\TB^{*}\nentails\bigsqcap\mat{\E_{i}}\subs\lnot(C\dland D)$, otherwise we would have $\TB^{*}\entails\bigsqcap\mat{\E_{i}}\dland C\subs\lnot D$. Hence we have $C\dland D\dsubs E\in\Cn{\rat}(\KB)$ since $\TB^{*}\entails\bigsqcap\mat{\E_{i}}\dland C\dland D\subs E$. If $\rk(C)=\infty$, then we have $\TB^{*}\entails C\subs\bot$, and the proof is trivial. \hfill\ \qed
\end{itemize}
\end{proof}



The following lemma states that, as in the propositional case~\cite{LehmannMagidor1992}, our procedure correctly manages the classical information, that is, an axiom $C\dsubs\bot$ is in the rational closure of~$\KB$ if and only if it is also a modular consequence of~$\KB$.

\begin{lemma}[$*$]
\label{Lemma:ClassicalInformation}
Let $\KB=\TB\cup\DB$ and assume $C\dsubs D\in\DB$. Then $\KB\entails_{\modular}C\dsubs\bot$ iff $\rk(C)=\infty$ iff $\TB^{*}\entails C\subs\bot$.
\end{lemma}
\begin{proof}
\noindent Let $\KB=\TB\cup\DB$.
\myskip

\noindent For the only-if part, $\KB\entails_{\modular}C\dsubs\bot$ implies that every rational subsumption relation containing~$\KB$ must satisfy also $C\dsubs\bot$. Hence we have that $\KB\proves_{\rat}C\dsubs\bot$, since $\Cn{\rat}(\KB)$ induces a rational subsumption relation extending~$\KB$ (Lemma~\ref{Prop:RationalInducedSubsumption}). From Definition~\ref{Def:RationalDeduction}, we know that $\KB\proves_{\rat}C\dsubs\bot$ is possible only if~$C$ is always negated in the ranking procedure, \ie, $\TB^{*}\entails C\subs\bot$.
\myskip

\noindent For the if part, we define from~$\KB$ a new knowledge base $\KB^{*}\defined\TB^{*}\cup\DB^{*}$, with $\TB^{*}$ obtained from~$\TB$ by adding all the sets $\{C\subs D\mid C\dsubs D\in\DB^{*}_{\infty}\}$ that we obtain at each iteration of function~$\compRank(\cdot)$. Let us denote with $\DB^{1}_{\subs},\ldots,\DB^{n}_{\subs}$ such sets. Assume that $\TB^{*}\entails C\subs\bot$, but $\KB\nentails_{\modular}C\dsubs\bot$, \ie, there is a modular model of~$\KB$ in which~$C$ is non-empty. Let~$\RI$ be such a model, with an object~$x$ falling under~$C^{\RI}$. Since $\TB^{*}\entails C\subs\bot$, there must be a GCI $E\subs F$ in some $\DB^{i}_{\subs}$ that is not satisfied, that is, given the nature of the GCIs in every $\DB^{n}_{\subs}$ ($\TB^{*}\entails E\subs\bot$ for every $E\subs F$ contained in some $\DB^{n}_{\subs}$), this means that there is a subsumption $E\subs\bot$ that is not satisfied in~$\RI$. Therefore, there must be an object~$y$ falling under~$E^{\RI}$. Hence, assuming $E\subs F\in\DB^{i}_{\subs}$, since $\TB\cup\DB^{1}_{\subs}\cup\ldots\cup\DB^{i-1}_{\subs}\entails\bigsqcap\{\lnot G\dlor H\mid G\subs H\in\DB^{i}_{\subs}\}\subs\lnot E$, either $\RI\sat\TB\cup\DB^{1}_{\subs}\cup\ldots\cup\DB^{i-1}_{\subs}$ and $y\in(G\dland\lnot H)^{\RI}$ for some $G\subs H\in\DB^{i}_{\subs}$ (Case~1 below), or $\RI\nsat\TB\cup\DB^{1}_{\subs}\cup\ldots\cup\DB^{i-1}_{\subs}$ (Case~2 below).

\begin{description}\setlength{\itemsep}{1.3ex}
\item[Case 1.] Since~$\RI\sat\KB$, $\RI$ is also a model of $G\dsubs H$, which is an element of~$\DB$. Hence there must be an object~$y$ such that $y\pref^{\RI}x$ (remember that~$x\in C^{\RI}$) and $y\in(G\dland H)^{\RI}$. Again, since $G\subs H\in\DB^{i}_{\subs}$ (which implies $\TB\cup\DB^{1}_{\subs}\cup\ldots\cup\DB^{i-1}_{\subs}\entails\bigsqcap\{\lnot G\dlor H\mid G\subs H\in\DB^{i}_\subs\}\subs\lnot G$) and~$\RI\sat\TB\cup\DB^{1}_{\subs}\cup\ldots\cup\DB^{i-1}_{\subs}$, there must be a GCI $I\subs L\in\DB^{i}_{\subs}$ such that $y\in(I\dland\lnot L)^{\RI}$, and we need an object~$z$ such that $z\pref^{\RI}y$ and $z\in(H\dland I)^{\RI}$, and so on\ldots This procedure creates an infinitely descending chain of objects, and, since the number of the antecedents of the axioms in~$\DB^{i}_{\infty}$ is finite, it cannot be the case since the model would not satisfy the smoothness condition for the concept $\bigsqcup\{C\mid C\dsubs D\in\DB^{i}_{\infty}\}$ (see Definition~\ref{Def:PrefInterpretation}).
\item[Case 2.] If $\RI\nsat\TB\cup\DB^{1}_{\subs}\cup\ldots\cup\DB^{i-1}_{\subs}$, then~$\RI$ does not satisfy some $E\subs F\in\DB^{j}_{\subs}$ for some $j<i$, and therefore there must be an object falling under~$E^{\RI}$. Again, it is either Case~1 or Case~2. Nevertheless, since at every iteration of Case~2 we pick a lower value~$j$ for $\DB^{j}_{\subs}$ and we have a finite sequence of $\DB^{j}_{\subs}$, we know that after some steps (in the worst case when we reach $\DB^{0}_{\subs}$) we necessarily fall into Case~1, which cannot be the case. \hfill\ \qed
\end{description}
\end{proof}

An immediate consequence of Lemma~\ref{Lemma:ClassicalInformation} binds preferential consistency  (existence of a preferential model -- \cf\ Definition~\ref{Def:Satisfaction}) to classical consistency.

\begin{corollary}[$*$]
\label{Cor:Consistency}
Let $\KB=\TB\cup\DB$. Then $\KB\entails_{\modular}\top\dsubs\bot$ iff $\TB^{*}\entails \top\subs\bot$.
\end{corollary}

We can now prove that the knowledge bases $\KB=\TB\cup\DB$ and $\KB^{*}=\TB^{*}\cup\DB^{*}$ (in rank normal form) are modularly equivalent.

\restatablePrefEquivalence*


\begin{proof}
Given $\KB=\TB\cup\DB$, the function $\compRank(\KB)$ outputs a knowledge base $\KB^{*}=\TB^{*}\cup\DB^{*}$, in which the iteration of lines 5--13 identifies a (possibly empty) set $\{C_1\dsubs D_1,\ldots,C_n\dsubs D_n\}$ of always exceptional defeasible subsumptions, that is moved from $\DB$ to $\TB$. That is, we have $\TB^{*}=\TB\cup\{C_1\subs D_1,\ldots,C_n\subs D_n\}$ and $\DB^{*}=\DB\setminus\{C_1\dsubs D_1,\ldots,C_n\dsubs D_n\}$. It is sufficient to prove that $\KB\entails_{\modular} C_{i}\subs\bot$ and $\KB^{*}\entails_{\modular}C_i\dsubs D_{i}$ for every $C_{i}\dsubs D_{i}$ ($1\leq i\leq n$).

Let $C_{i}\dsubs D_{i}\in\DB\setminus\DB^{*}$. It means that, at some iteration through Lines~4--14  of function~$\compRank(\cdot)$, we have $\TB^{*}\entails\bigsqcap\mat{\DB^{*}_{\infty}}\subs\lnot C_{i}$, which implies that $\TB^{*}\cup\DB^{*_{\subs}}_{\infty}\entails\top\subs\lnot C_{i}$, where $\DB^{*_{\subs}}_{\infty}\defined\{C\subs D\mid C\dsubs D\in \DB^{*}_{\infty}\}$). Since every $\DB^{*_{\subs}}_{\infty}$ created at every iteration is contained in the final $\TB^{*}$, using such final $\TB^{*}$ we have that $\TB^{*}\entails C_{i}\subs\bot$. Hence, by Lemma~\ref{Lemma:ClassicalInformation} we have that $\KB\entails_{\modular}C_{i}\dsubs\bot$, \ie, $\KB\entails_{\modular}C_{i}\subs\bot$.

On the other hand, if $C_{i}\dsubs D_{i}\in\DB\setminus\DB^{*}$, then $C_{i}\subs D_{i}\in\TB^{*}$, and hence $\KB^{*}\entails_{\modular}C_{i}\dsubs D_{i}$ by supra-classicality (\cf\ proof of Lemma~\ref{Lemma:TopC} below).
\hfill\ \qed
\end{proof}
\myskip

Now we are justified in using the rank normal form $\KB^{*}=\TB^{*}\cup\DB^{*}$ in order to analyse the rational closure of the knowledge base $\KB=\TB\cup\DB$. Hence, in what follows we shall assume that the knowledge bases we are working with are already in rank normal form (and therefore $\DB_{\infty}=\emptyset$).

In the next lemma, we observe that the inference relation $\proves_{\rat}$ respects the preferential conclusions of~$\KB$ \wrt\ assertions of the form $\top\dsubs C$, another desideratum proven for the propositional case by Lehmann and Magidor~\cite{LehmannMagidor1992}.

\begin{lemma}[$*$]
\label{Lemma:TopC}
For every $C\in\Lang$, $\KB\entails_{\modular}\top\dsubs C$ iff $\KB\proves_{\rat}\top\dsubs C$.
\end{lemma}
\begin{proof}
First of all, recall that $\KB\proves_{\rat}\top\dsubs C$ if $\TB^{*}\entails\bigsqcap\mat{\DB^{*}}\subs C$ (\cf\ Definition~\ref{Def:RationalDeduction}).
\myskip

\noindent For the if part, first we need to prove two properties of~$\entails_{\modular}$, namely supra-classicality (Sup) and one half of the deduction theorem (S):
\[
(\text{\small Sup})\ \frac{C\subs D}{C\dsubs D}
\]
The derivation of Sup is straightforward: remember that $C\dsubs C$ holds (Ref), assume $C\subs D$ and then apply (RW).
\[
(\text{\small S})\ \frac{C\dsubs D}{\top\dsubs\lnot C\dlor D}
\]
To see that (S) holds, assume $C\dsubs D$ and note that $\entails D\subs\lnot C\dlor D$; we derive by (RW) $C\dsubs\lnot C\dlor D$. Since $\entails\lnot C\subs\lnot C\dlor D$, we obtain $\lnot C\dsubs\lnot C\dlor D$ by (Sup). Then apply (Or) to $C\dsubs\lnot C\dlor D$ and $\lnot C\dsubs\lnot C\dlor D$, obtaining $\top\dsubs\lnot C\dlor D$.

Now we have to prove that if $\TB^{*}\entails\bigsqcap\mat{\DB^{*}}\subs C$, then $\KB\entails_{\modular}\top\dsubs C$.

From Lemma~\ref{Prop:PrefEquivalence} we know that $\TB^{*}\cup\DB^{*}$ is in the modular consequences of~$\KB$. Applying~(S) to all DCIs $C\dsubs D$ in $\DB^{*}$, we have $\KB\entails_{\modular}\top\dsubs\lnot C\dlor D$ from each of them. Applying (And) to all these DCIs, we have $\top\dsubs\bigsqcap\mat{\DB'}$  and, by (RW'), we obtain $\top\dsubs C$.
\myskip

The only-if part is an immediate consequence of Lemma~\ref{Prop:RationalInducedSubsumption}. \hfill\ \qed
\myskip
\end{proof}

\begin{lemma}[$*$]\label{Lemma:SameInfiniteRank}
For every $\KB=\TB\cup\DB$ and every $C\in\Lang$, $\rank_{\KB}(C)=\infty$ iff $\rk(C)=\infty$.
\end{lemma}
\begin{proof}
Let $\KB=\TB\cup\DB$ and transform it into a modularly equivalent knowledge base~$\DB'$ composed of only DCIs (see Lemma~\ref{Lemma:ClassicalStatements}). Since the model~$\OI$ of the rational closure of~$\KB$ must also be a model of~$\DB'$, we can easily derive from Lemma~\ref{Lemma:CountablyInfiniteDomain} that $\KB\entails_{\rat}C\dsubs\bot$ (that is, $\rank_{\KB}(C)=\infty$) iff $\KB\entails_{\modular}C\dsubs\bot$. From Lemma \ref{Lemma:ClassicalInformation} we have that $\KB\entails_{\modular}C\dsubs\bot$ iff $\rk(C)=\infty$, hence the result. \hfill\ \qed
\end{proof}

\restatableSameRankings*

\begin{proof} From Lemmas~\ref{Lemma:ClassicalInformation} and~\ref{Lemma:SameInfiniteRank} and Lemma~\ref{Prop:PrefEquivalence} we can see that, given a knowledge base $\KB=\TB\cup\DB$ (possibly with an empty~$\TB$), we can define a modularly equivalent knowledge base $\KB^{*}=\TB^{*}\cup\DB^{*}$ such that all the classical information implicit in~$\DB$ is moved into $\TB^{*}$. $\KB^{*}$ can be defined identifying the elements of~$\DB$ that have~$\infty$ as ranking value, and Lemma~\ref{Lemma:SameInfiniteRank} shows that \wrt\ the value~$\infty$, $\rank_{\KB}(\cdot)$ and $\rk(\cdot)$ are equivalent, while Lemma~\ref{Prop:PrefEquivalence} tells us that~$\KB$ and~$\KB^{*}$ are modularly equivalent. Once we have defined $\KB^{*}$, Lemma~\ref{Lemma:TopC} implies that a concept~$C\in\Lang$ is exceptional \wrt~$\entails_{\rat}$ ($\KB\entails_{\modular}\top\dsubs\lnot C$) iff $\KB\proves_{\rat}\top\dsubs\lnot C$. Hence the two ranking functions $\rank_{\KB}(\cdot)$ and~$\rk(\cdot)$ give back exactly the same results. \hfill\ \qed
\end{proof}

\restatableSoundnessCompletenessRC*

\begin{proof}
 Since we have already proven Lemma~\ref{Lemma:SameRankings}, here we can use $\rk(\cdot)$ to indicate indifferently the equivalent ranking functions $\rank_{\KB}(\cdot)$ and $\rk(\cdot)$.
\myskip

For the only-if part, assume $\KB\entails_{\rat}C\dsubs D$. That means that either $\rk(C\dland\lnot D)>\rk(C)$ or $\rk(C)=\infty$. In the first case, it means that there is some~$i$, $0\leq i\leq n$, such that $\TB^{*}\nentails\bigsqcap\mat{\E_{i}}\subs\lnot C$ and $\TB^{*}\entails\bigsqcap\mat{\E_{i}}\subs\lnot(C\dland\lnot D)$, hence $\TB^{*}\entails\bigsqcap\mat{\E_{i}}\dland C\subs D$, \ie, $\KB\proves_{\rat}C\dsubs D$. In the second case, we have $\TB^{*}\entails C\subs\bot$, which implies $\KB\proves_{\rat}C\dsubs D$.
\myskip

For the if part, assume $\KB\proves_{\rat}C\dsubs D$. Then either there is some~$i$ which is the lowest number such that $\TB^{*}\nentails\bigsqcap\mat{\E_{i}}\subs\lnot C$ (hence $\rk(C)=i$), or $\TB^{*}\entails C\subs\bot$. In the first case, we have also that $\TB^{*}\entails \bigsqcap\mat{\E_{i}}\dland C\subs D$, which implies $\TB^{*}\entails\bigsqcap\mat{\E_{i}}\subs\lnot(C\dland\lnot D)$, \ie, $\rk(C\dland\lnot D)>i$. In the second case, $\rk(C)=\infty$, which implies $\KB\entails_{\rat}C\dsubs D$. \hfill\ \qed
\end{proof}

\restatableComplexity*
\begin{proof}
Observe that function~$\compClosure(\cdot)$ performs at most $n+2$ (classical) subsumption checks, where~$n$ is the number of ranks assigned to elements of $\DB$. So the number of subsumption checks performed by function~$\compClosure(\cdot)$ is $O(|\DB|)$. Furthermore, we need to call function~$\compRank(\cdot)$ to obtain the knowledge base $\KB^{*}=\TB^{*}\cup\DB^{*}$ and the sequence $\E_{0},\ldots,\E_{n}$, which are needed as input to function~$\compClosure(\cdot)$. First bear in mind that function~$\excep(\cdot)$, with~$\E$ as input, performs at most $|\E|$ classical subsumption checks. From this, and an analysis of function~$\compRank(\cdot)$, it follows that the number of subsumption checks performed by function~$\compRank(\cdot)$ is $O(|\DB|^3)$. Since we know that subsumption checking \wrt\ general TBoxes in~$\ALC$ is \ExpTime-complete~\cite[Chapter~3]{BaaderEtAl2007}, the result follows. \hfill\ \qed
\end{proof}

\end{document}